\documentclass{article}

\usepackage[T1]{fontenc}
\usepackage{tikz}
\usetikzlibrary{arrows}
\usetikzlibrary{calc}
\usetikzlibrary{snakes}
\usepackage{cite}
\usepackage{graphicx}
\usepackage{array}
\usepackage{mdwmath}
\usepackage{mdwtab}
\usepackage{eqparbox}
\usepackage{fixltx2e}
\usepackage{url}
\usepackage{pgf}
\usepackage{lipsum}
\usepackage{multirow}
\usepackage{subcaption}
\usepackage[cmex10]{amsmath}
\usepackage{amsfonts,amssymb,amsthm}
\usepackage{xspace}
\usepackage{xcolor}
\usepackage{algorithmic}
\usepackage{algorithm}
\usepackage{bm}
\usepackage{bbm}
\usepackage{booktabs}
\usepackage{etoolbox}
\usepackage{makecell}
\usepackage{nicefrac}
\usepackage{textcomp}
\usepackage{stmaryrd}
\usepackage{mathrsfs}
\usepackage{paralist}
\usepackage{comment}

\usepackage[font=footnotesize,labelsep=space,labelfont=bf]{caption}

\usepackage{url}

\usepackage[pdftex,bookmarks=true,pdfstartview=FitH,colorlinks,linkcolor=blue,filecolor=blue,citecolor=blue,urlcolor=blue,pagebackref=false]{hyperref}
\makeatletter

\newcommand{\Rmnum}[1]{\expandafter\@slowromancap\romannumeral #1@}
\makeatother

\newtheorem{theorem}{Theorem}%[chapter]
\newtheorem*{theorem*}{Theorem}
\newtheorem{definition}{Definition}%[chapter]
\newtheorem{lemma}{Lemma}%[chapter]
\newtheorem*{lemma*}{Lemma}
\newtheorem{claim}{Claim}%[chapter]
\newtheorem*{claim*}{Claim}
\newtheorem*{cor*}{Corollary}
\newtheorem{remark}{Remark}%[chapter]
\newtheorem{fact}{Fact}%[chapter]
\newtheorem{assumption}{Assumption}

\newcommand{\namedref}[2]{\hyperref[#2]{#1~\ref*{#2}}}

\newcommand{\GD}{\text{GD}}

\newcommand{\Algorithmref}[1]{\namedref{Algorithm}{algo:#1}}
\newcommand{\Assumptionref}[1]{\namedref{Assumption}{assump:#1}}

\newcommand{\Sectionref}[1]{\namedref{Section}{sec:#1}}
\newcommand{\Subsectionref}[1]{\namedref{Section}{subsec:#1}}

\newcommand{\Appendixref}[1]{\namedref{Appendix}{app:#1}}
\newcommand{\Theoremref}[1]{\namedref{Theorem}{thm:#1}}

\newcommand{\Lemmaref}[1]{\namedref{Lemma}{lem:#1}}
\newcommand{\Remarkref}[1]{\namedref{Remark}{remark:#1}}
\newcommand{\Claimref}[1]{\namedref{Claim}{claim:#1}}
\newcommand{\Factref}[1]{\namedref{Fact}{fact:#1}}

\newcommand{\Figureref}[1]{\namedref{Figure}{fig:#1}}

\newcommand{\Pageref}[1]{\hyperref[#1]{page~\pageref*{#1}}}

\newcommand{\eps}{\ensuremath{\epsilon}\xspace}

\definecolor{darkred}{rgb}{0.5, 0, 0} 
\definecolor{darkblue}{rgb}{0,0,0.5} 
\hypersetup{
    colorlinks=true,
    linkcolor=darkred,
    citecolor=darkblue,
    urlcolor=darkblue   
}

\newcommand{\topk}{\ensuremath{\mathrm{top}_k}\xspace}
\newcommand{\randk}{\ensuremath{\mathrm{rand}_k}\xspace}
\newcommand{\select}{\ensuremath{\mathrm{select}}\xspace}
\renewcommand{\(}{\ensuremath{\left(}\xspace}
\renewcommand{\)}{\ensuremath{\right)}\xspace}

\newcommand{\vct}[1]{\boldsymbol{#1}}

\newcommand{\bmu}{\ensuremath{\boldsymbol{\mu}}\xspace}
\newcommand{\sigmag}{\ensuremath{\sigma_{\text{g}}}\xspace}
\newcommand{\bzero}{\ensuremath{{\bf 0}}\xspace}

\newcommand{\bg}{\ensuremath{\boldsymbol{g}}\xspace}
\newcommand{\btg}{\ensuremath{\widehat{\bg}}\xspace}
\newcommand{\btlg}{\ensuremath{\widetilde{\bg}}\xspace}

\newcommand{\bx}{\ensuremath{\boldsymbol{x}}\xspace}
\newcommand{\bp}{\ensuremath{\boldsymbol{p}}\xspace}

\newcommand{\bhx}{\ensuremath{\widehat{\bx}}\xspace}
\newcommand{\by}{\ensuremath{\boldsymbol{y}}\xspace}
\newcommand{\bty}{\ensuremath{\widetilde{\by}}\xspace}
\newcommand{\bz}{\ensuremath{\boldsymbol{z}}\xspace}

\newcommand{\bw}{\ensuremath{\boldsymbol{w}}\xspace}
\newcommand{\btw}{\ensuremath{\widetilde{\bw}}\xspace}
\newcommand{\bu}{\ensuremath{\boldsymbol{u}}\xspace}
\newcommand{\bv}{\ensuremath{\boldsymbol{v}}\xspace}

\newcommand{\bA}{\ensuremath{{\bf A}}\xspace}
\newcommand{\bB}{\ensuremath{{\bf B}}\xspace}

\newcommand{\btF}{\ensuremath{\widehat{F}}\xspace}
\newcommand{\btlF}{\ensuremath{\widetilde{F}}\xspace}
\newcommand{\bG}{\ensuremath{{\bf G}}\xspace}
\newcommand{\bI}{\ensuremath{{\bf I}}\xspace}
\newcommand{\bM}{\ensuremath{{\bf M}}\xspace}
\newcommand{\bP}{\ensuremath{{\bf P}}\xspace}

\newcommand{\calA}{\ensuremath{\mathcal{A}}\xspace}
\newcommand{\calB}{\ensuremath{\mathcal{B}}\xspace}
\newcommand{\calC}{\ensuremath{\mathcal{C}}\xspace}
\newcommand{\calD}{\ensuremath{\mathcal{D}}\xspace}
\newcommand{\calE}{\ensuremath{\mathcal{E}}\xspace}
\newcommand{\calS}{\ensuremath{\mathcal{S}}\xspace}
\newcommand{\calT}{\ensuremath{\mathcal{T}}\xspace}

\newcommand{\bW}{\ensuremath{{\bf W}}\xspace}

\newcommand{\bY}{\ensuremath{{\bf Y}}\xspace}
\newcommand{\bZ}{\ensuremath{{\bf Z}}\xspace}

\newcommand{\bbE}{\ensuremath{\mathbb{E}}\xspace}
\newcommand{\calF}{\ensuremath{\mathcal{F}}\xspace}
\newcommand{\calH}{\ensuremath{\mathcal{H}}\xspace}
\newcommand{\calK}{\ensuremath{\mathcal{K}}\xspace}
\newcommand{\calO}{\ensuremath{\mathcal{O}}\xspace}

\newcommand{\calQ}{\ensuremath{\mathcal{Q}}\xspace}

\newcommand{\Y}{\ensuremath{\mathcal{Y}}\xspace}

\newcommand{\calV}{\ensuremath{\mathcal{V}}\xspace}
\newcommand{\W}{\ensuremath{\mathcal{W}}\xspace}

\newcommand{\N}{\ensuremath{\mathcal{N}}\xspace}

\newcommand{\R}{\ensuremath{\mathbb{R}}\xspace}

\renewcommand{\paragraph}[1]{\smallskip\noindent{\bf #1}~}

\usepackage{microtype}

\usepackage{fullpage}
\usepackage{hyperref}
\usepackage{sidecap}

\begin{document}

\title{Byzantine-Resilient SGD in High Dimensions on \\ Heterogeneous Data\thanks{This work was supported by the NSF grants \#1740047, \#1514731, and by the UC-NL grant LFR-18-548554.}}

\author{\medskip Deepesh Data and Suhas Diggavi \\
University of California, Los Angeles, USA \\ Email: \texttt{deepesh.data@gmail.com, suhas@ee.ucla.edu}}

\date{}
\maketitle

\begin{abstract}
We study distributed stochastic gradient descent (SGD) in the master-worker architecture under Byzantine attacks. %, where, in each SGD iteration, the corrupt workers may send arbitrary gradients to master to disrupt the gradient aggregation. 
We consider heterogeneous data model, where different workers may have different local datasets, and we do not make any probabilistic assumption on data generation. 
%We propose two Byzantine-resilient SGD algorithms, depending on whether workers send full gradients or compressed gradients.  
At the core of our algorithm, we use the polynomial-time outlier-filtering procedure for robust mean estimation proposed by Steinhardt et al.\ (ITCS 2018) %(which was developed for robust mean estimation in high dimensions) 
to filter-out corrupt gradients.
% and compute an approximation of the average of uncorrupt gradients.
In order to be able to apply their filtering procedure in our {\em heterogeneous} data setting where workers compute {\em stochastic} gradients, we derive a new matrix concentration result, which may be of independent interest.
%bound by generalizing a recent result from Charikar et al.\ (STOC 2017), which was developed to tackle several problems in robust learning.

We provide convergence analyses for smooth strongly-convex and non-convex objectives. 
We derive our results under the bounded variance assumption on local stochastic gradients and a {\em deterministic} condition on datasets, namely, gradient dissimilarity; and for both these quantities, we provide concrete bounds in the statistical heterogeneous data model.
%, and give {\em dimension-independent} approximation error guarantees, provided that workers use a sufficiently large mini-batch size. 
%We assume these functional properties (e.g., smoothness, convexity, etc.) only for the global objective function, and the local objective functions (defined by the local datasets) at different workers may be arbitrary. 
We give a trade-off between the mini-batch size for stochastic gradients and the approximation error. 
Our algorithm can tolerate up to $\frac{1}{4}$ fraction Byzantine workers. It can find approximate optimal parameters in the strongly-convex setting exponentially fast, and reach to an approximate stationary point in the non-convex setting with a linear speed, thus, matching the convergence rates of vanilla SGD in the Byzantine-free setting. 

We also propose and analyze a Byzantine-resilient SGD algorithm with gradient compression, 
where workers send $k$ random coordinates of their gradients. 
Under mild conditions, we show a $\frac{d}{k}$-factor saving in communication bits as well as decoding complexity over our compression-free algorithm without affecting its convergence rate (order-wise) and the approximation error.
%, in comparison with our compression-free algorithm. 
%We do not incorporate error-feedback, and an implication of our result is that, under a mild condition, {\em incorporating error-feedback cannot improve the convergence} of our algorithm. This is in sharp contrast with the Byzantine-free SGD with compression, where error-feedback provably improves the convergence.

%(without incorporating error-feedback). Under a mild condition -- that $k$ is bigger than the number of honest workers and a constant fraction of workers are corrupt -- the convergence rate of this algorithm matches with its compression-free counterpart, which implies a $\frac{d}{k}$ factor saving in communication bits as well as in the decoding complexity. A further implication of this result is that {\em incorporating error-feedback cannot improve the convergence} of this algorithm, which is in sharp contrast with the Byzantine-free SGD with compression, where error-feedback provably improves the convergence.
\end{abstract}

{\allowdisplaybreaks

\section{Introduction}\label{sec:intro}
Stochastic gradient descent (SGD) \cite{RobbinsMonro51} is the main workhorse behind the optimization procedure in several modern large-scale learning algorithms \cite{Bottou10}. In this paper, we consider a master-worker architecture, where the training data is distributed across several machines (workers) and a central node (master) wants to learn a machine learning model using SGD~\cite{mapreduce}.
This setting naturally arises in the case of {\em federated learning} \cite{KonecnyThesis,jakub2016fed}, where user devices are recruited to help build machine learning models.
This can also arise in a distributed setup, where data is partitioned and stored in many servers to speed up the computation. 
In such scenarios, the recruited worker nodes may not be trusted with their computation, 
either because of non-Byzantine failures, such as software bugs, noisy training data, etc.,
or because of Byzantine attacks, where corrupt nodes may manipulate the information to their advantage \cite{LamportShPe82}.
These Byzantine adversaries may collaborate and arbitrarily deviate from their pre-specified programs.
Training machine learning models in the presence of Byzantine attacks has received attention lately
\cite{Bartlett-Byz18,Bartlett-Byz_nonconvex19,Alistarh_Byz-SGD18,SuX_Byz19,ChenSX_Byz17,Krum_Byz17,SignSGD_Byz19,Draco_ByzGD18,Detox_ByzSGD19,DataSoDi_arxiv-Byz19,DataSoDi_ByzGD19,DataDi_ByzCD19,NirupamVa_Byz-SGD19} and also in the context of the Internet of Battlefield Things (IoBT) \cite{IoBTproject}.
See also \cite[Section 5]{Kairouz-FL-survey19} for a detailed survey on Byzantine-robustness in federated learning. 
The importance of this problem motivates us to study Byzantine-resilient optimization algorithms that are suitable for large-scale learning problems.
See \Subsectionref{related-work} where we put our work in context with these works.

In this paper, we study distributed SGD in the presence of Byzantine adversaries for empirical risk minimization.
The training data is distributed across $R$ different workers, and master wants to iteratively build a machine learning model using the gradients computed at the workers. All workers have potentially different local datasets, and we do not make any probabilisitic assumption on data generation. 
In our setup, up to $\eps R$ workers (where $\eps>0$ is a constant) may be under Byzantine attacks, and corrupt workers may collaborate and report adversarially chosen gradients to the master; see \Figureref{dist-byz-sgd}.
See also \Subsectionref{adversary-model} for more details on our adversary model.

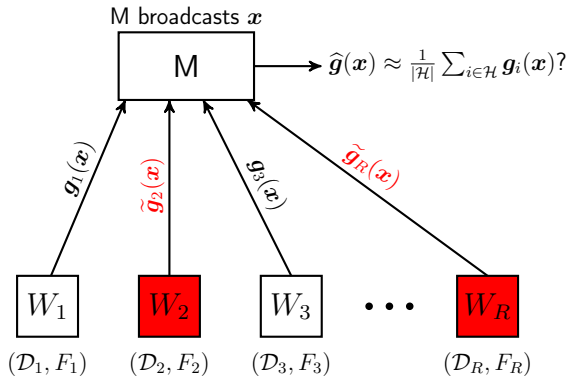
\begin{SCfigure}
\centering
\begin{tikzpicture}[scale=0.9, >=stealth', font=\sffamily\large, thick]
\node [scale=0.9, right] at (-0.2, 5.7) {Distributed SGD with Byzantine adversaries};
\node [rotate=0, scale=0.75] at (2.5,4.75) {M \text{broadcasts} $\bx$};
\draw (1.5,3.5) rectangle (3.5,4.5) node [pos=0.5] {M};
\draw [->] (3.5,4.0) -- (4.5,4.0);
\node [scale=0.75, right] at (4.5,4.0) {$\btg(\bx)\approx\frac{1}{|\calH|}\sum_{i\in\calH}\bg_i(\bx)$?};

\draw (0,0) rectangle (0.9,0.9) node [pos=0.5] {$W_1$}; \node [scale=0.75] at (0.45, -0.4) {$(\calD_1,F_1)$};
\draw [fill=red] (1.8,0) rectangle (2.7,0.9) node [pos=0.5] {$W_2$}; \node [scale=0.75] at (2.25, -0.4) {$(\calD_2,F_2)$};
\draw (3.6,0) rectangle (4.5,0.9) node [pos=0.5] {$W_3$}; \node [scale=0.75] at (4.05, -0.4) {$(\calD_3,F_3)$};
\draw [fill=black] (5.2,0.45) circle [radius=0.04];
\draw [fill=black] (5.5,0.45) circle [radius=0.04];
\draw [fill=black] (5.8,0.45) circle [radius=0.04];
\draw [fill=red] (6.5,0) rectangle (7.4,0.9) node [pos=0.5] {$W_R$}; \node [scale=0.75] at (6.95, -0.4) {$(\calD_R,F_R)$};

\draw [->] (0.5,0.9) -- (1.6,3.5); \node [scale=0.75, rotate=68] at (0.89,2.5) {$\bg_1(\bx)$};
\draw [->] (2.25,0.9) -- (2.25,3.5); \node [scale=0.75, rotate=90, red] at (2.0,2.25) {$\btlg_2(\bx)$};
\draw [->] (4.05,0.9) -- (2.75,3.5); \node [scale=0.75, rotate=-65] at (3.7,2.25) {$\bg_3(\bx)$};
\draw [->] (6.95,0.9) -- (3.4,3.5); \node [scale=0.75, rotate=-32, red] at (5.25,2.5) {$\btlg_R(\bx)$};
\end{tikzpicture}
\centering
\caption{In the master-worker architecture for distributed optimization, each of the $R$ workers (denoted by $W_i$) stores local datasets -- worker $r$ stores $\calD_r$ with an associated local loss function $F_r$. We are in a heterogeneous data setting, where the local datasets $\calD_r$'s are arbitrary and are not necessarily generated from the same distribution.
Master (denoted by M) wants to learn a machine learning model through SGD which minimizes the average of local loss functions; see \eqref{eq:problem-expression}.
The adversarial nodes are denoted in red color. Let $\calH$ denote the set of honest workers.
In any SGD iteration, master broadcasts the current model parameter vector $\bx$ to all workers.
Each honest worker $i$ computes the stochastic gradient $\bg_i(\bx)$ and sends it back to the master; 
corrupt nodes may send arbitrary vectors.
Master wants to compute $\btg(\bx)\approx\frac{1}{|\calH|}\sum_{i\in\calH}\bg_i(\bx)$ in order to update the model parameter vector.
Computing $\btg(\bx)$ and providing convergence analyses for strongly-convex and non-convex objectives is the subject of this paper.}
\label{fig:dist-byz-sgd}
\end{SCfigure}

\subsection{Our contributions}
In our Byzantine-resilient SGD algorithm (see \Algorithmref{Byz-SGD}), we use a non-trivial decoding at master to filter-out corrupt gradients, which is inspired by the recent advances in high-dimensional robust mean estimation \cite{Robust_mean_LRV16,Robust_mean_DiakonikolasKK016,Resilience_SCV18}.
In particular, we use the outlier-filtering method developed by \cite{Resilience_SCV18}.
Our algorithm can tolerate $\eps\leq\frac{1}{4}$ fraction of corrupt worker nodes.
Our results can be summarized as follows:
\begin{itemize}
\item We give convergence analyses for smooth strongly-convex and non-convex objectives.
We assume these functional properties (e.g., smoothness, convexity, etc.) only for the global objective function, and the local objective functions (defined by the local datasets) at different workers may be arbitrary. 
We identify a {\em deterministic} condition on datasets, namely, that the gradient dissimilarity at different workers is bounded, and we derive our convergence results under this condition. 
For SGD, we make a standard assumption that the local variances (due to sampling stochastic gradients) at all workers are bounded.
We provide concrete upper bounds on the gradient dissimilarity as well as the local variances in the statistical heterogeneous model under different distributional assumptions (sub-exponential and sub-Gaussian) on local gradients; see \Theoremref{concrete-kappa-bound_stat} and \Theoremref{variance-bound_stat} in \Sectionref{statistical-model}.

\item In the strongly-convex case, our algorithm can find optimal parameters within an approximation error of $\calO(\kappa^2+\frac{\sigma^2}{bR}+\frac{\sigma^2d\eps'}{bR})$ (where $\eps'>\eps$ is any constant and $b$ is the mini-batch size for stochastic gradients) ``exponentially fast''; and in the non-convex case, it can find an approximate stationary point within the same error with ``linear speed'', i.e., with a rate of $\nicefrac{1}{T}$.
The $\kappa^2$ term in the approximation error quantifies the gradient dissimilarity, and is equal to zero if all workers are assumed to have access to the same dataset, which (i.e., $\kappa=0$) had been the underlying assumption in most results in Byzantine-resilient SGD; see, for example, \cite{Krum_Byz17,Alistarh_Byz-SGD18}.
The second term $\frac{\sigma^2}{bR}$ is the standard SGD variance term and
the third term $\frac{\sigma^2d\eps'}{bR}$ is due to Byzantine attacks.
Note that both these terms can be made small by taking a sufficiently large mini-batch size of stochastic gradients. 
Note that when workers compute full-batch gradients (i.e., $\sigma=0$), the approximation error becomes $\calO(\kappa^2)$.\footnote{It is not surprising that when $\kappa=0$, we reach to an exact optimum in full-batch GD -- when $\kappa=0$, all workers have the same data, and master can decode the correct gradient by simply taking the majority vote of the received gradients.}
See \Theoremref{SGD_convergence} for our mini-batch SGD results and \Theoremref{full-batch-GD} for our full-batch GD results.

\item In order to be able to apply the robust mean estimation result of \cite{Resilience_SCV18} to our {\em heterogeneous} data setting where workers sample {\em stochastic} gradients from their local datasets, we derive a new matrix concentration bound (stated in \Theoremref{gradient-estimator}), by building upon a recent result from \cite{Untrusted-data_Charikar17}, which was developed to tackle several robust learning tasks.
See \Subsectionref{bounded-variance_subset} for more details.

\item We extend the above results to a setting, where, instead of sending full gradients, workers send {\em compressed} gradients to master. In particular, workers send random $k$ coordinates of their $d$-dimensional gradients to master; see \Algorithmref{algo_compression}. Under a mild condition that $\eps$ is constant and $k\geq(1-\eps)R$, i.e., the adversary corrupts a constant fraction of workers and $k$ is more than the number of honest workers, we show that the convergence rate and the approximation error achieved by this algorithm match (order-wise) with its compression-free counterpart. This has two direct implications: first, it gives a saving of $\frac{d}{k}$ factor in communication bits as well as in the decoding complexity, which could be as large as $\Omega(\frac{d}{R})$; and second, it shows that incorporating error-feedback cannot improve the convergence of our algorithm, which is in sharp contrast with the Byzantine-free SGD with compression, where error-feedback provably improves the convergence rate \cite{memSGD}. See \Theoremref{convergence_randk} for our results with gradient compression.
\end{itemize}
As far as we know, this is the first paper that studies Byzantine-resilient distributed SGD under standard SGD assumptions in the heterogeneous data model, where different workers may have different local datasets with no probabilistic assumptions on data generation.\footnote{Note that \Theoremref{concrete-kappa-bound_stat} and \Theoremref{variance-bound_stat} are just to demonstrate the behavior of the parameters if the data was generated by a heterogeneous statistical model.}
We can control the approximation error (in our convergence results) with the batch size of the mini-batch stochastic gradients, and can reduce it a point where the approximation error only depends on the heterogeneity in the local datasets. Note that the approximation error consists of two types of terms: one is a fixed term that captures the gradient dissimilarity among local datasets; and the other is a statistical error term, which comes from two sources, one from stochastic sampling of gradients and the other from robustly estimating the true gradient. Reducing the first type of statistical error by increased mini-batch is standard, but an implication of our result is that we can even control the second type of statistical error (which is due to Byzantine attacks) in the same way.

\subsection{Related work}\label{subsec:related-work}
Byzantine-resilient distributed computing has a long history \cite{LamportShPe82} and is a very well studied topic, which has received recent attention in the context of distributed learning. 
The approach taken to tackle Byzantine attacks in literature can be broadly divided into two categories: One using statistical assumptions on the data (e.g., all workers have data generated from the same distribution, or in the extreme case, all workers have the same data) -- detailed references are given below; and the other using coding-theoretic or redundancy-based techniques \cite{Draco_ByzGD18,DataSoDi_arxiv-Byz19,Detox_ByzSGD19}.

Byzantine SGD has been studied under the assumption that all workers have access to the entire dataset:
\cite{Krum_Byz17} employed a heuristic to compute the smallest ball containing $(1-\eps)R$ gradients, where an $\eps$ fraction of the $R$ workers are corrupt. They consider non-convex objectives and showed an almost sure convergence of gradients under stringent assumptions. \cite{Alistarh_Byz-SGD18} filter out corrupt workers based on gradients received in different iterations. Using martingale-based methods, they showed convergence under an assumption that the stochastic error in gradients is bounded with probability 1 (instead of assuming bounded variance). 
\cite{NirupamVa_Byz-SGD19} studies linear regression only -- it removes corrupt nodes based on norm-filtering, and achieves an error that scales with the number of data points.
In contrast, our results in this paper are for general smooth strongly-convex and non-convex objectives under a standard SGD assumption of bounded variance, where different workers have different local datasets and we assume no data distribution (thus, deviating from the standard assumption in Byzantine-resilient SGD of all workers having access to the same dataset \cite{Krum_Byz17,Alistarh_Byz-SGD18}). Our results are qualitatively different from these papers, where we can control the approximation error by the mini-batch size (larger the mini-batch size, better the accuracy), and with sufficiently large batch size, we can even get an error that depends only on the heterogeneity bound (which is captured through the gradient dissimilarity).

There have been works in {\em full-batch} gradient descent against Byzantine attacks, where data at workers is assumed to be drawn i.i.d.~from a probability distribution, and the goal is to minimize the {\em population risk} \cite{Bartlett-Byz18,Bartlett-Byz_nonconvex19,SuX_Byz19,ChenSX_Byz17}. 
In contrast, we consider a more realistic {\em heterogeneous} data setting, where different workers may have local datasets satisfying bounded gradient dissimilarity (the bound is $\kappa^2$, as stated in \eqref{bounded_local-global}), and our goal is to minimize the {\em empirical risk}. 
We give concrete bounds on $\kappa^2$ for a statistical heterogeneous model, 
where data points at different workers are sampled from different distributions, but i.i.d.\ from the same distribution at any given worker. In this setting, we prove that $\kappa^2 = \kappa_{\text{mean}}^2 + \calO\(\frac{d\log(nd)}{n}\)$, where $\kappa_{\text{mean}}$ captures the difference between local and global population means and $n$ is the number of data samples at each worker. Note that $\kappa_{\text{mean}}^2=0$ in the above-mentioned works, as all workers sample data from the same distribution. 
Though our main focus in this paper is on SGD, we also give our specialized results for full-batch GD, with which we will compare the results in above works.
With the above bound on $\kappa^2$ (derived in the statistical heterogeneous data model), the approximation error in our convergence results with full-batch GD for both strongly-convex and non-convex objectives is bounded by $\widetilde{\calO}\(\kappa_{\text{mean}}^2 + \frac{d}{n}\)$.\footnote{The $\widetilde{\calO}$ and $\widetilde{\Omega}$ notations hide logarithmic factors.}
\cite{Bartlett-Byz18} employed coordinate-wise median and trimmed median, and got an approximation error of $\widetilde{\calO}\(\frac{d^2}{nR}\)$ for both convex and non-convex objectives, which could be prohibitive in high-dimensional problems; 
\cite{ChenSX_Byz17} and \cite{SuX_Byz19} considered only strongly-convex objectives, where \cite{ChenSX_Byz17} used decoding based on median-of-means and gave an error of $\widetilde{\calO}\(\frac{\eps d}{n}\)$, and \cite{SuX_Byz19} improved it to $\widetilde{\calO}\(\frac{d}{nR}\)$ for constant $\eps$  -- observe that these papers only study strongly-convex objectives, whereas, in addition, we study non-convex objectives too. The decoding algorithm in \cite{SuX_Byz19} is taken from \cite{Resilience_SCV18}, which is based on robust mean estimation, and we also use that algorithm in our decoding. Note that \cite{SuX_Byz19} requires $\widetilde{\Omega}(d^2)$ data points to succeed, whereas, we require $\widetilde{\Omega}(dR)$ data points, which could be significantly smaller when $d\gg R$. 
\cite{Bartlett-Byz_nonconvex19} proposed and analyzed an algorithm to avoid saddle-point attacks in non-convex problems and provided {\em second-order} convergence guarantees. 
In the high-dimensional setting, they also used the decoding algorithm of \cite{Resilience_SCV18} and gave an approximation error of $\widetilde{\calO}\(\frac{d}{nR}\)$ in their i.i.d.\ homogeneous data setting.
Note that our bound $\widetilde{\calO}\(\kappa_{\text{mean}}^2 + \frac{d}{n}\)$ on the approximation error is in a more general {\em heterogeneous} data setting, and we believe that our results can also be extended to combat the saddle-point attacks in non-convex problems, which we leave as part of the future work.

Apart from the heterogeneity in data, there are other technical differences between \cite{SuX_Byz19,Bartlett-Byz_nonconvex19} and our work, and we would like to point out one of them here. In order to use the decoding algorithm of \cite{Resilience_SCV18}, both these works derive a matrix concentration bound, the need of which arises because they minimize the population risk.
In this paper, since we minimize the empirical risk, we do not need such a result. 
However, we do need to prove a matrix concentration bound (which is of a very different nature than theirs, and we use entirely different tools to prove that), the need of which arises because the gradients are {\em stochastic} due to SGD -- if we work with {\em full-batch} deterministic gradients, we would not need any of such concentration bounds. See also the discussion in \Subsectionref{bounded-variance_subset} for more details and \Theoremref{gradient-estimator} for our new matrix concentration result.
Note that \cite{SuX_Byz19} analyzed {\em full-batch} gradient descent only for strongly-convex objectives in the i.i.d.\ homogeneous data setting and left a few problems open, including analyzing the {\em stochastic} gradient descent, convergence for non-convex objectives, and an algorithm with gradient compression.
In this paper, we resolve all these open problems (while minimizing the empirical risk) in a more general {\em heterogeneous} data setting, and provide comprehensive analyses of Byzantine SGD and prove its convergence for both strongly-convex and non-convex objectives.
Note that \cite{Bartlett-Byz_nonconvex19} also provided a convergence analysis for non-convex objectives in the same setting as that of \cite{SuX_Byz19}, i.e., the i.i.d.\ homogeneous data setting while minimizing the population risk with full-batch GD; in contrast, our results are in a {\em heterogeneous} data setting, and we minimize the empirical risk with SGD.

As far as we know, not much has been studied for Byzantine learning with gradient compression, except for a few notable exceptions of \cite{SignSGD_Byz19,Compressed_ByzGD19}. 
Under the assumption that all workers have access to the same data,
\cite{SignSGD_Byz19} achieves compression using a 1-bit quantizer -- sign of the gradient vector -- and performs a simple decoding at master node using the majority vote.
They assume that each component of the stochastic gradients has symmetric and unimodal distribution around its mean. Their algorithm can only tolerate ``blind multiplicative adversaries'', which restricts the adversary to multiply the gradient by any vector of its choice, but it has to be decided before observing the gradient.\footnote{It is not hard to come up with a slightly more powerful adversarial attack that breaks their system.}
Their convergence results only hold under large mini-batch stochastic gradients, where the mini-batch size is equal to the total number of iterations.
\cite{Compressed_ByzGD19} studies {\em full batch} gradient descent under the i.i.d.~homogeneous data assumption. Their setting and distributional assumptions are similar to \cite{Bartlett-Byz18}, and they get an approximation error of $\widetilde{\calO}\(\frac{d^2}{nR}\)$, which could be prohibitive in high-dimensional settings. Their results are for an arbitrary compressor, but their decoding algorithm employs only norm filtering; see the discussion in \Sectionref{robust-grad-est} on why norm-based filtering is not sufficient for getting good approximation guarantees in high-dimensional learning.
In contrast to the settings in both these papers, we study mini-batch distributed SGD on {\em heterogeneous} data under standard SGD assumptions (with different workers having different datasets and no data distribution), where we can control the approximation error by the mini-batch size, and thus obtain qualitatively different results.

\subsection{Paper organization}
We describe our problem setup in \Sectionref{problem-setup}.
We state our main convergence results in \Sectionref{our-results} and extend them with gradient compression in \Sectionref{compression}.
We describe the core part of our algorithm, the robust gradient estimation, and our new matrix concentration result in \Sectionref{robust-grad-est}. 
We instantiate our assumptions in the statistical heterogeneous data model in \Sectionref{statistical-model}.
We conclude with a few open problems in \Sectionref{conclusion}.

\subsection{Notation}
For any $n,n_1,n_2\in\mathbb{N}$ such that $n_1\leq n_2$, we denote the set $\{1,2,\hdots,n\}$ by $[n]$, 
and the set $\{n_1,n_1+1,\hdots,n_2\}$ by $[n_1:n_2]$.
For any finite set $\calK\subset\mathbb{N}$, we write $k\in_U\calK$ to denote that $k$ is sampled uniformly at random from $\calK$.
We denote matrices with bold capital letters $\bA,\bB,$ etc., and 
vectors with bold small letters $\bx,\by$, etc.
All vector norms in this paper are $\ell_2$-norm, and, for simplicity, without explicitly writing $\|\cdot\|_2$, we will just denote them by $\|\cdot\|$. For a matrix $\bA$, we denote the matrix norm of $\bA$ (induced by the $\ell_2$-norm on the vector space) by $\|\bA\|$ (instead of explicitly writing $\|\bA\|_2$), which is equal to the largest singular value of $\bA$.
For a matrix $\bA$, we write $\bA\succeq\bzero$ and $\bA\succ\bzero$ to denote that $\bA$ is positive semi-definite and positive definite, respectively. For two matrices $\bA,\bB$, we write $\bA\preceq\bB$ to denote that $(\bB-\bA)$ is positive semi-definite.
For a square matrix $\bA$, we denote its largest eigenvalue by $\lambda_{\max}(\bA)$.

\section{Problem Setup}\label{sec:problem-setup}
In the master-worker architecture that we consider in this paper, each of the $R$ workers may have different datasets; see \Figureref{dist-byz-sgd}.
Let the dataset stored at the $r$'th worker be denoted by $\calD_{r}=\{\bz_{r,1},\bz_{r,2},\hdots,\bz_{r,n_r}\}$, which is a collection of $n_r$ data points for some $n_r\in\mathbb{N}$. 
We allow different workers to have different number of data points.
Let $\calC\subseteq\R^d$ denote the parameter space.
We can take $\calC$ to be equal to $\R^d$ in the absence of any constraints on parameters or a compact and convex set otherwise.
Note that the dimension of the data samples may be much smaller than $d$. 
For example, in the case of neural networks, the dimension of the data samples is equal to the number of inputs in the first layer, which may be much smaller than the dimension $d$ of the model learned.

Our goal is to learn a model $\bx\in\calC$ that minimizes the average loss $F(\bx):=\frac{1}{R}\sum_{r=1}^R F_r(\bx)$, where $F:\R^d\to\R$ denotes the global loss function, and for each $r\in[R]$, $F_r:\R^d\to\R$ denotes the local loss function at worker $r$. For $r\in[R]$, $F_r$ is defined as $F_r(\bx):= \frac{1}{n_r}\sum_{i=1}^{n_r}F_{r,i}(\bx)$, where $F_{r,i}(\bx)$ denotes the loss associated with $\bz_{r,i}$ (the $i$'th data-point at worker $r$) with respect to (w.r.t.) the model $\bx$.
Note that $F_r(\bx)$ denotes the average loss associated with the data-points in $\calD_r$ w.r.t.\ the model $\bx$, 
and we want to find an $\bx\in\calC$ that minimizes the average loss $\frac{1}{R}\sum_{r=1}^R F_r(\bx)$.
Formally, we want to solve the following minimization problem:
\begin{align}\label{eq:problem-expression}
\arg \min_{\bx\in \calC}\left(F(\bx) := \frac{1}{R}\sum_{r=1}^R F_r(\bx)\right).
\end{align} 
All convergence results in this paper only require properties of the global loss function $F$; the local loss functions $F_r,r\in[R]$ may be arbitrary.
For example, in the smooth strongly-convex case, we only require $F$ to be smooth and strongly-convex, 
and we do not impose any condition on $F_r$'s. Similarly for the smooth non-convex case.

When $F$ is strongly-convex, let the minimization in \eqref{eq:problem-expression} be attained at $\bx^*$ and we assume that $\bx^*\in\calC$. In the case of non-convex $F$, as standard in literature, 
we find a stationary point where the gradient becomes zero.

We can minimize \eqref{eq:problem-expression} using {\em distributed stochastic gradient descent} (SGD), which is an iterative algorithm that proceeds as follows: 
Initialize the model $\bx^{0}:=\bzero$.
At the $t$'th iteration, for $t\geq0$, master broadcasts $\bx^{t}$; each worker $r\in[R]$ sends $\bg_r(\bx^{t}):=\nabla F_{r,r_t}(\bx^{t})$ to the master for a randomly chosen $r_t\in_U[n_r]$, independent of the choice of other workers; 
master updates the parameter vector according to the following update rule:
\begin{align}\label{eq:sgd-update-rule}
\bx^{0}:=\bzero; \quad \bx^{t+1} = \bx^{t} - \eta\frac{1}{R}\sum_{r=1}^R \bg_r(\bx^{t}), \ \ t = 0,1,2,\hdots
\end{align}
Here, $\eta$ denotes the learning rate. 
We make the following assumptions about distributed SGD.

Note that, for any $r\in[R]$, $\bbE_{i\in_U[n_r]}[\nabla F_{r,i}(\bx)]=\nabla F_r(\bx)$ holds for every $\bx\in\R^d$.
\begin{assumption}[Bounded local variances]\label{assump:bounded_local-variance}
The stochastic gradient sampled from any local dataset is uniformly bounded over $\calC$ for all workers, i.e., there exists a finite $\sigma$, such that
\begin{align}\label{bounded_local-variance}
\bbE_{i\in_U[n_r]}\|\nabla F_{r,i}(\bx) - \nabla F_r(\bx)\|^2 \leq \sigma^2, \quad \forall \bx\in\calC, r\in[R].  
\end{align}
\end{assumption}
\begin{assumption}[Bounded gradient dissimilarity]\label{assump:gradient-dissimilarity}
The difference between the local gradients $\nabla F_r(\bx), r\in[R]$ and the global gradient $\nabla F(\bx)=\frac{1}{R}\sum_{r=1}^R\nabla F_r(\bx)$ is uniformly bounded over $\calC$ for all workers, i.e., there exists a finite $\kappa$, such that
\begin{align}\label{bounded_local-global}
\|\nabla F_r(\bx) - \nabla F(\bx)\|^2 \leq \kappa^2, \quad \forall \bx\in\calC, r\in[R].  
\end{align}
\end{assumption}
Though we assume that all local datasets have the same $\sigma,\kappa$, this is without loss of generality -- in case different datasets have different $\sigma_r,\kappa_r,r\in[R]$, since $\sigma,\kappa$ are upper bounds in \eqref{bounded_local-variance}, \eqref{bounded_local-global}, respectively, we can take $\sigma=\max_{r\in[R]}\sigma_r$ and $\kappa=\max_{r\in[R]}\kappa_r$ to be the maximum of the corresponding local parameters.

\Assumptionref{bounded_local-variance} is standard in the SGD literature.
In \Assumptionref{gradient-dissimilarity}, $\kappa$ quantifies the deviation between the local loss functions $F_r, r\in[R]$ and the global loss function $F$, and this assumption states that this deviation is bounded. Note that when all workers have access to the same dataset, we have $\kappa=0$, which has been the standard assumption for Byzantine SGD in literature \cite{Alistarh_Byz-SGD18,Krum_Byz17}. \Assumptionref{gradient-dissimilarity} has been used earlier; see, for example, \cite{Momentum_linear-speedup19}, which studies decentralized SGD with momentum (without Byzantine workers).

\begin{remark}\label{remark:bounded_local-global}
Observe that, in distributed algorithms in the presence of Byzantine adversaries, since we do not know which subset of $\eps R$ workers are corrupt, we have to make some assumption on the data to provide relationships among gradients sampled at different nodes for reliable decoding. 
Otherwise, since the adversary can corrupt any subset of $\eps R$ workers and we do not assume any relationship among gradients, there is no way we can do reliable decoding at the master.
Previous works created such relationships by assuming that either the workers' data is drawn i.i.d.~from a probability distribution \cite{Bartlett-Byz18,Bartlett-Byz_nonconvex19,SuX_Byz19,ChenSX_Byz17} or all workers can sample gradients from the same data \cite{Alistarh_Byz-SGD18,Krum_Byz17}.  
Other works have used redundancy-based or coding-theoretic techniques to provide such relationships among gradients \cite{Draco_ByzGD18,DataSoDi_arxiv-Byz19,Detox_ByzSGD19}.
All these approaches fall short of in a distributed setup such as federated learning \cite{jakub2016fed},
where workers have non-i.i.d.\ data and therefore cannot sample gradients from the same data, 
and it is infeasible to perform data encoding across different nodes.
In this paper, we neither use coding/redundancy-based techniques, nor do we make any probabilistic assumptions on the data generation; and we allow different workers to have different arbitrary datasets.
As argued above, since we have to make some assumption that correlates local data, we assume uniform boundedness \eqref{bounded_local-global} of the deviation of local gradients from the global gradient, which is a much weaker assumption that the above-mentioned ones.
\end{remark}

The gradient dissimilarity bound in \eqref{bounded_local-global} can be seen as a {\em deterministic} condition on local datasets, under which we derive our results. 
All results (matrix concentration and convergence) in this paper are given in terms of the variance bound $\sigma^2$ and the gradient dissimilarity bound $\kappa^2$.
In \Sectionref{statistical-model}, we provide concrete bounds on $\sigma,\kappa$ in the statistical {\em heterogeneous} model under different distributional assumptions (sub-exponential and sub-Gaussian) on local gradients. 
In the statistical heterogeneous model, local datasets at different workers are generated from potentially different distributions. This model is more suitable for federated learning \cite{jakub2016fed} than the statistical {\em homogeneous} model considered in literature \cite{ChenSX_Byz17,Bartlett-Byz18,SuX_Byz19,Bartlett-Byz_nonconvex19}, where all local datasets are generated from the {\em same} distribution.
Note that we make distributional assumptions on data generation {\em only} to derive bounds on $\sigma,\kappa$. Other than that, we do not make any distributional assumption on the data and all results in this paper hold for arbitrary datasets satisfying \eqref{bounded_local-variance}, \eqref{bounded_local-global}.

In the parameter update rule \eqref{eq:sgd-update-rule}, once the workers send the local stochastic gradients to master,
it aggregates them by taking their average and updates the parameter vector according to \eqref{eq:sgd-update-rule}.
Observe that, this simple aggregation rule (i.e., averaging) at master is vulnerable to Byzantine attacks, because, instead of sending the true stochastic gradients, the corrupt workers may send adversarially chosen vectors to disrupt the computation -- it is known that even a single Byzantine worker can prevent the algorithm to converge, even worse, it can cause the algorithm to converge to an adversarially chosen point \cite{Krum_Byz17}. Our adversary model is described  next.

\subsection{Adversary model}\label{subsec:adversary-model}
We assume that an $\eps$ fraction of $R$ workers are corrupt; as we see later, we can tolerate $\eps\leq\frac{1}{4}$. The corrupt workers can collaborate and arbitrarily deviate from their pre-specified programs: In any SGD iteration, instead of sending the true stochastic gradients, corrupt workers can send adversarially chosen vectors (they may not even send anything if they wish, in which case, the master can treat them as {\em erasures} and replace them with a fixed value). Note that, in the erasure case, master knows which workers are corrupt; whereas, in the Byzantine problem, master does not have this information.

Our algorithms are also resilient against a more powerful {\em adaptive} and {\em mobile} adversary (which can corrupt a different set of $\eps R$ workers in different SGD iterations based on the knowledge it has gathered in the past),\footnote{We do not allow a mobile adversary to contaminate local datasets of the compromised nodes; otherwise, after a certain number of iterations, it can end up contaminating the entire data stored at all the workers, which renders solving the optimization problem in \eqref{eq:problem-expression} meaningless.} as long as it does not change the set of corrupt workers {\em after} observing the gradients in any iteration; otherwise, the gradients of honest workers may not remain independent, a property we need in order to derive our matrix concentration result stated in \Theoremref{gradient-estimator}.
Note that since we allow a mobile adversary, we cannot consider optimizing the expression in \eqref{eq:problem-expression} with respect to the data stored only at the honest workers, as there is no fixed set of honest workers during the entire optimization procedure.

\section{Our Results}\label{sec:our-results}
We tackle the Byzantine behavior of corrupt workers by applying a non-trivial decoding algorithm at the master in each SGD iteration. Our decoding algorithm is inspired by the recent breakthrough results in theoretical computer science for  robust mean estimation \cite{Robust_mean_LRV16,Robust_mean_DiakonikolasKK016,Resilience_SCV18};
see \Sectionref{robust-grad-est} for more details. 

Before stating our results, we need to formally define {\em mini-batch} SGD. Note that we can speed up the convergence of distributed SGD by having each worker sample many data points (without replacement), say, $b\geq1$ data points, and send the average gradients on these data points to the master. This is called mini-batch SGD. 
To formalize this, for any $\bx\in\R^d, r\in[R], b\in[n_r]$, consider the following set 
\begin{align}\label{bigger_set}
\calF_r^{\otimes b}(\bx):=\left\{ \frac{1}{b}\sum_{i\in\calH_b}\nabla F_{r,i}(\bx) :  \calH_b\in \binom{[n_r]}{b} \right\}.
\end{align}
Let $H_b$ denote a random variable taking values in $\binom{[n_r]}{b}$ with uniform distribution.
We denote a mini-batch stochastic gradient with batch size $b$ by $\nabla F_{r,H_b}(\bx)$, which is a uniformly random element of $\calF_r^{\otimes b}(\bx)$.
It is not hard to show that the mean of $\nabla F_{r,H_b}(\bx)$ remains unchanged and is equal to $\nabla F_r(\bx)=\frac{1}{n_r}\sum_{i=1}^{n_r}\nabla F_{r,i}(\bx)$, and the variance reduces by a factor of $b$:
\begin{align}
\bbE_{H_b} \left[\nabla F_{r,H_b}(\bx)\right] &= \nabla F_r(\bx), \label{same_mean} \\
\bbE_{H_b} \left\|\nabla F_{r,H_b}(\bx) - \nabla F_r(\bx)\right\|^2 &\leq \frac{\sigma^2}{b}.\label{reduced_variance}
\end{align}
Note that the variance bound in \eqref{reduced_variance} trivially follows if we assume that the workers sample stochastic gradients {\em with} replacement.
In this paper, since workers sample stochastic gradients {\em without} replacement, we can in fact show a stronger bound of $\bbE_{H_b} \left\|\nabla F_{r,H_b}(\bx) - \nabla F_r(\bx)\right\|^2 \leq \frac{(n_r-b)}{b(n_r-1)}\sigma^2$.
However, for simplicity of exposition, we only use the weaker variance bound of \eqref{reduced_variance} in this paper.
\begin{algorithm}[t]
   \caption{Byzantine-Resilient SGD}\label{algo:Byz-SGD}
\begin{algorithmic}[1]
   \STATE {\bf Initialize.} Set $\bx^0 := \bzero$. Fix a constant learning rate $\eta$ and a mini-batch size $b$.
   \FOR{$t=0$ {\bfseries to} $T-1$}
  \STATE \textbf{On Workers:}
   \FOR{$r=1$ {\bfseries to} $R$}
   \STATE Receive $\bx^t$ from master. Take a mini-batch stochastic gradient $\bg_r(\bx^{t}) \in_U \calF_r^{\otimes b}(\bx^{t})$.
   \STATE 
   $\btlg_r(\bx^{t}) = 
   \begin{cases}
   \bg_r(\bx^{t}) & \text{ if worker $r$ is honest}, \\
   \divideontimes & \text{if worker $r$ is corrupt},
   \end{cases}$
   
   where $\divideontimes$ is an arbitrary vector in $\R^d$.
   \STATE Send $\btlg_r(\bx^{t})$ to master. 
   \ENDFOR
   
   \STATE \textbf{At Master:}
   \STATE Receive $\{\btlg_r(\bx^{t})\}_{r=1}^R$ from the $R$ workers.
   \STATE Apply the decoding algorithm {\sc RGE} (described in \Algorithmref{robust-grad-est} in \Appendixref{robust-mean-estimation}) on $\{\btlg_r(\bx^{t})\}_{r=1}^R$. Let $$\btg(\bx^{t})=\textsc{RGE}(\btlg_1(\bx^{t}),\hdots,\btlg_R(\bx^{t})).$$
   \STATE Update the parameter vector: 
   \begin{align*}
   \bhx^{t+1} = \bx^{t} - \eta\btg(\bx^{t}); \qquad 
   \bx^{t+1} = \Pi_{\calC}\left(\bhx^{t+1}\right).
   \end{align*}
   \STATE Broadcast $\bx^{t+1}$ to all workers.
   \ENDFOR
\end{algorithmic}
\end{algorithm}

We present our Byzantine-resilient SGD algorithm in \Algorithmref{Byz-SGD}.
Our convergence results are for both strongly-convex and non-convex smooth functions.
Before stating them, we need some definitions first. 
\begin{itemize}
\item {\bf $L$-smoothness:} A function $F:\calC\to\R$ is called $L$-smooth over $\calC\subseteq\R^d$, if for every $\bx,\by\in\calC$, we have $\|\nabla F(\bx) - \nabla F(\by)\| \leq L\|\bx-\by\|$ (this property is also known as $L$-Lipschitz gradients). This is also equivalent to $F(\by) \leq F(\bx) + \langle\nabla F(\bx), \by-\bx\rangle + \frac{L}{2}\|\bx-\by\|^2$.
\item {\bf $\mu$-strong convexity:} A function $F:\calC\to\R$ is called $\mu$-strongly convex over $\calC\subseteq\R^d$, if for every $\bx,\by\in\calC$, we have $F(\by) \geq F(\bx) + \langle\nabla F(\bx), \by-\bx\rangle + \frac{\mu}{2}\|\bx-\by\|^2$.
\end{itemize}

\subsection{Convergence results}\label{subsec:SGD_main-results}
\begin{theorem}[Strongly-convex and Non-convex]\label{thm:SGD_convergence}
Suppose an $\epsilon>0$ fraction of $R$ workers are adversarially corrupt.
For an $L$-smooth 
global objective function $F:\calC\to\R$, let \Algorithmref{Byz-SGD} generate a sequence of iterates $\{\bx^t\}_{t=0}^T$ when run with a fixed learning rate $\eta$, where in the $t$'th iteration, every honest worker $r\in[R]$ samples a mini-batch stochastic gradient from $\calF_r^{\otimes b}(\bx^t)$, satisfying \eqref{same_mean} and \eqref{reduced_variance} (corrupt workers may send arbitrary vectors).
Fix an arbitrary constant $\eps'>0$. If $\eps\leq\frac{1}{4} - \eps'$, then with probability at least $1-T\exp(-\frac{\eps'^2(1-\eps)R}{16})$, we have the following convergence guarantees:
\begin{itemize}
\item {\bf Strongly-convex:} If $F$ is also $\mu$-strongly convex and we take $\eta=\frac{\mu}{L^2}$, then we have 
\begin{align}\label{convex_convergence_rate}
\bbE\|\bx^{T} - \bx^* \|^2 &\leq \left(1-\frac{\mu^2}{2L^2}\right)^T\|\bx^0-\bx^*\|^2 + \frac{2L^2}{\mu^4}\varGamma.
\end{align}
If we take $T= {\log\left(\frac{\mu^4}{L^2\varGamma}\|\bx^0-\bx^*\|^2\right)}/{\log(\frac{1}{1-\nicefrac{\mu^2}{2L^2}})}$, we get $\bbE\|\bx^{T} - \bx^* \|^2 \leq \frac{3L^2}{\mu^4}\varGamma$.
\item {\bf Non-convex:} If we take $\eta=\frac{1}{4L}$, then we have 
\begin{align}\label{nonconvex_convergence_rate}
\frac{1}{T}\sum_{t=0}^T\bbE\|\nabla F(\bx^t)\|^2 &\leq \frac{8L^2}{T}\|\bx^0-\bx^*\|^2 + \varGamma,
\end{align}
If we take $T = \frac{8L^2\|\bx^0-\bx^*\|^2}{\varGamma}$, we get $\frac{1}{T}\sum_{t=0}^T\bbE\|\nabla F(\bx^t)\|^2 \leq 2\varGamma$.
\end{itemize}
In both \eqref{convex_convergence_rate} and \eqref{nonconvex_convergence_rate}, expectation is taken over the sampling of mini-batch stochastic gradients. Here, $\varGamma=\frac{9\sigma^2}{(1-(\eps+\eps'))bR} + 9\kappa^2 + 9\varUpsilon^2$ with $\varUpsilon = \calO\left(\sigma_0\sqrt{\eps+\eps'}\right)$, where $\sigma_0^2 = \frac{24\sigma^2}{b\eps'}\left(1 + \frac{d}{(1-(\eps+\eps'))R}\right) + 16\kappa^2$.
\end{theorem}
We prove the strongly-convex part of \Theoremref{SGD_convergence} in \Appendixref{convex_convergence} and the non-convex part in \Appendixref{nonconvex_convergence}.

\paragraph{Projection.}
If the parameter space $\calC$ is not equal to $\R^d$, then our convergence analysis for non-convex objectives requires a mild technical assumption on the size of $\calC$. This assumption is only required to ensure that the iterates $\bx^t$ always stay inside $\calC$ without projection. Similar assumption has also been made in \cite{Bartlett-Byz18} for the same purpose. This assumption streamlines our convergence analysis, as our focus in this paper is on Byzantine-resilience.
\begin{assumption}[Size of $\calC$]\label{assump:non-convex_size-params-space}
Suppose $\|\nabla F(\bx)\|\leq M$ for all $\bx\in\calC$.
We assume that $\calC$ contains the $\ell_2$ ball $\{\bx\in\R^d:\|\bx-\bx^0\|\leq\frac{2L}{\varGamma}(M+\varGamma_1)\|\bx^0-\bx^*\|^2\}$, where $\varGamma=\frac{9\sigma^2}{(1-(\eps+\eps'))bR} + 9\kappa^2 + 9\varUpsilon^2$ and $\varGamma_1=\frac{n_{\max}\sigma}{b} + \kappa + \varUpsilon$, 
where $n_{\max}=\max_{r\in[R]}n_r$ and other parameters are as defined in \Theoremref{SGD_convergence} above.
\end{assumption}
Note the dependence of the size of $\calC$ on $\frac{n_{\max}}{b}$, which is the maximum number of data samples at any worker. 
This happens because we want a {\em deterministic} bound on the size of $\calC$ (not in expectation) even though we are doing {\em stochastic} sampling of data points for gradient computation.
See the proof of \Lemmaref{non-convex_projection-not-needed} (and also \Claimref{robust-grad-est-under-deterministic-variance}) in \Appendixref{nonconvex_convergence} for more details.
\subsection{Remarks about \Theoremref{SGD_convergence}}\label{subsec:remarks_main-result}
In this section, we discuss some important aspects of our convergence results.

\paragraph{Analysis of the approximation error.}
In both parts of \Theoremref{SGD_convergence}, the approximation error $\varGamma$ consists of three error terms: first is $\varGamma_1=\calO(\nicefrac{\sigma^2}{(1-(\eps+\eps'))bR})$, which is the standard error arising due to the sampling of stochastic gradients; second is $\varGamma_2=\calO(\kappa^2)$, which is due to dissimilarity in the local datasets; and third is $\varGamma_3=\calO\left(\left(\frac{\sigma^2}{b\eps'}\left(1 + \frac{d}{(1-(\eps+\eps'))R}\right) + \kappa^2\right)(\eps+\eps')\right)$, which is due to Byzantine attacks. Observe that $\varGamma_1$ decreases with the mini-batch size $b$ and the number of workers $R$, as desired. Note that $\varGamma_3$ consists of two terms $\varGamma_{3,1}=\frac{\sigma^2}{b\eps'}\left(1 + \frac{d}{(1-(\eps+\eps'))R}\right)(\eps+\eps')$ and $\varGamma_{3,2}=\kappa^2(\eps+\eps')$, where we can make $\varGamma_{3,1}$ small by taking a large mini-batch size $b$. Note that the presence of $\varGamma_{3,2}$ is inevitable, since $\kappa$ captures the dissimilarity in different datasets, and that will always show up when bounding the deviation of the true ``global'' gradient from the decoded one in the presence of Byzantine workers.
See \Figureref{robust-grad-est} to get a pictorial intuition on the above analysis of the approximation error.

Hence, by taking a sufficiently large mini-batch size, we can reduce the error term to $\calO(\kappa^2)$, which, in the statistical heterogeneous model described in \Sectionref{statistical-model}, is equal to $\calO\(\kappa_{\text{mean}}^2+\frac{d\log(nd)}{n}\)$; see \Theoremref{concrete-kappa-bound_stat}. Here, $\kappa_{\text{mean}}$ captures the difference between local and global population means (see \Assumptionref{uniform-bound-mean-heterogeneous}) and $n$ is the number of data samples at each worker. In particular, if each worker has $n=\Omega\(d\log(nd)\)$ data points, and they take a sufficiently large mini-batch size in each iteration of \Algorithmref{Byz-SGD}, we can reduce the approximation error to $\calO(\kappa_{\text{mean}}^2)$.
Note that our {\em heterogeneous} data setting generalizes (as far as we know) the only data settings studied in literature for Byzantine-resilient distributed optimization (see also \Remarkref{bounded_local-global}), where workers either have i.i.d.\ {\em homogeneous} data (i.e., $\kappa_{\text{mean}}=0$) \cite{Bartlett-Byz18,Bartlett-Byz_nonconvex19,SuX_Byz19,ChenSX_Byz17}, or are assumed to have access to the {\em same} data \cite{Alistarh_Byz-SGD18,Krum_Byz17} (i.e., $\kappa=0$).
 
\paragraph{Convergence rates.} Note that, in the strongly-convex case, \Algorithmref{Byz-SGD} approximately finds optimal parameters $\bx^*$ (within $\varGamma$ error, which could be a constant)  ``exponentially fast''; and in the non-convex case, \Algorithmref{Byz-SGD} approximately finds a stationary point up to the same error with ``linear speed'', i.e., with a rate of $\nicefrac{1}{T}$.
Thus, we recover the convergence rate of vanilla SGD (running in the Byzantine-free setting) for both the objectives.

\paragraph{Corruption threshold.}
Our proposed algorithm can tolerate up to $\frac{1}{4}$ fraction Byzantine workers, which is away from the  information-theoretically optimal $\frac{1}{2}$ fraction. The $\frac{1}{4}$ bound comes from the subroutine of robust mean estimation (RME) that we use for robust gradient estimation (RGE), as explained in \Sectionref{robust-grad-est}. So, improved algorithms for RME that can be adapted to our setting will directly give an improved corruption threshold for our algorithm.

\paragraph{Failure probability.}
The failure probability of our algorithm is at most $T\exp(-\frac{\eps'^2(1-\eps)R}{16})$, which is at most $\delta$, for any $\delta>0$, provided we run our algorithm for $T \leq \delta\exp(\frac{\eps'^2(1-\eps)R}{16})$ iterations. 
Though the error probability scales linearly with $T$, it also goes down exponentially with the number of workers $R$.
As a result, in settings such as federated learning, where number of workers $R$ could be very large (in tens of thousands, or in millions), we can get a very small probability of error, say, $\nicefrac{1}{100}$, 
even if run our algorithm for a very long time. 
Note that the probability of error is due to the {\em stochastic} sampling of gradients,
and if we want a ``zero'' probability of error, we can run full-batch gradient descent, which is described in the next section.

\subsection{Convergence results for full-batch gradient descent}\label{subsec:GD_main-result}
In this section, we provide our results for the setting where workers compute {\em full-batch} gradients, instead of mini-batch stochastic gradients. This setting will simplify the approximation error in the solution produced by \Algorithmref{Byz-SGD} on both strongly-convex and non-convex objectives as well as their convergence analyses.
\begin{theorem}\label{thm:full-batch-GD}
Suppose an $\epsilon>0$ fraction of $R$ workers are adversarially corrupt.
For an $L$-smooth global objective function $F:\calC\to\R$, let \Algorithmref{Byz-SGD} generate a sequence of iterates $\{\bx^t\}_{t=0}^T$ when run with a fixed learning rate $\eta$, where in the $t$'th iteration, every honest worker $r\in[R]$ sends $\nabla F_r(\bx^t)$ to the master (corrupt workers may send arbitrary vectors).
If $\eps\leq\frac{1}{4}$, then with probability 1, we have the following convergence guarantees (where $\varGamma_{\GD} = 6\kappa^2 + 6\varUpsilon_{\GD}^2$ with $\varUpsilon_{\GD} = \calO\left(\kappa\sqrt{\eps}\right)$):
\begin{itemize}
\item {\bf Strongly-convex:} If $F$ is also $\mu$-strongly convex and we take $\eta=\frac{\mu}{L^2}$, then we have 
\begin{align}\label{convex-GD_convergence_rate}
\|\bx^{T} - \bx^* \|^2 &\leq \left(1-\frac{\mu^2}{2L^2}\right)^T\|\bx^0-\bx^*\|^2 + \frac{2L^2}{\mu^4}\varGamma_{\GD}.
\end{align}
\item {\bf Non-convex:} If we take $\eta=\frac{1}{4L}$, then we have 
\begin{align}\label{nonconvex-GD_convergence_rate}
\frac{1}{T}\sum_{t=0}^T\|\nabla F(\bx^t)\|^2 &\leq \frac{8L^2}{T}\|\bx^0-\bx^*\|^2 + \varGamma_{\GD}.
\end{align}
\end{itemize}
\end{theorem}
\Theoremref{full-batch-GD} is proved in \Appendixref{convergence_full-batch-GD}.

As mentioned in \Subsectionref{SGD_main-results}, when $\calC$ is a bounded set,
our convergence analysis for non-convex objectives requires a mild technical assumption on the size of $\calC$:
\begin{assumption}\label{assump:non-convex_size-params-space_GD}
Suppose $\|\nabla F(\bx)\|\leq M$ for every $\bx\in\calC$.
We assume that $\calC$ contains the $\ell_2$ ball $\{\bx\in\R^d:\|\bx-\bx^0\|\leq\frac{2L}{\varGamma_{\GD}}(M+\varGamma_2)\|\bx^0-\bx^*\|^2\}$, where $\varGamma_{\GD}=6\kappa^2 + 6\varUpsilon_{\GD}^2$ and $\varGamma_2 = \kappa + \varUpsilon$, 
where the parameters are as defined in \Theoremref{full-batch-GD} above.
\end{assumption}
Note that, unlike in \Assumptionref{non-convex_size-params-space}, $\varGamma_2$ in \Assumptionref{non-convex_size-params-space_GD} does not depend on the number of local samples at workers. This is because the variance is  zero when workers take full-batch gradients.

\paragraph{A note about the approximation error.}
Note that the approximation error term in both strongly-convex and non-convex objectives is $\varGamma_{\GD} = 6\kappa^2 + \calO(\kappa^2\eps) = \calO(\kappa^2)$, which only depends on the heterogeneity in the data,
and as argued in the error analysis paragraph of \Subsectionref{remarks_main-result}, is inevitable in the heterogeneous data setting.
In the statistical heterogeneous data model described in \Sectionref{statistical-model}, this is equal to $\calO\(\kappa_{\text{mean}}^2+\frac{d\log(nd)}{n}\)$; see the discussion in \Subsectionref{remarks_main-result}. A special case is where all workers have i.i.d.\ data (i.e., $\kappa_{\text{mean}}=0$), which is the setting considered in \cite{Bartlett-Byz18,Bartlett-Byz_nonconvex19,SuX_Byz19,ChenSX_Byz17}. These papers also studied full-batch gradient descent, but to minimize the {\em population risk}, as opposed to minimizing the {\em empirical risk}, which is the focus of this paper.
See \Subsectionref{related-work} for a detailed comparison of our approximation error with that in these works.

\section{Robust Gradient Estimation (RGE)}\label{sec:robust-grad-est}
We are given $R$ gradient vectors $\btlg_1(\bx),\hdots,\btlg_R(\bx)\in\R^d$ for an arbitrary $\bx\in\R^d$, 
where, $\btlg_r(\bx)=\bg_r(\bx)$ is a uniform sample from $\calF_r^{\otimes b}(\bx)$ if the $r$'th worker is honest, otherwise, $\btlg_r(\bx)$ can be arbitrary.
We want to compute $\btg(\bx)$, an estimate of $\nabla F(\bx)=\frac{1}{R}\sum_{r=1}^R\nabla F_r(\bx)$, where $\nabla F_r(\bx)$ is the mean of $\calF_r^{\otimes b}(\bx)$, such that $\|\btg(\bx) - \nabla F(\bx)\|$ is small for all $\bx\in\R^d$. 
In this section, building upon the recent advances in high-dimensional robust mean estimation problem, we provide a polynomial-time decoding algorithm (in particular, we use the outlier-filtering algorithm proposed by \cite{Resilience_SCV18}) and derive a matrix concentration result in order to use that algorithm in our setting. 

First we describe the problem of robust mean estimation (RME). 
In RME, we are given $R$ samples in $\R^d$ (out of which an $\eps$-fraction is corrupted) from an unknown distribution with unknown mean, and the goal is to estimate its mean. 
Though our problem is more general than RME, it would be helpful to first get some perspective on what makes RME, and hence our problem, so difficult.
RME is a classic problem in robust statistics \cite{Tukey-contaminated60,Huber64}. Until recently, all the solutions to this problem were either computationally intractable or were very poor in terms of the quality of the estimator produced. The method of {\em Tukey median} \cite{TukeyMedian75} solves this problem with dimension-independent error guarantees, but it is NP-hard to compute in general \cite{JohnsonP78}. On the other hand, solutions based on {\em geometric-median}, {\em coordinate-wise median} are computationally tractable, but can only give dimension-dependent error guarantees, which scales with $\sqrt{d}$ \cite{Robust_mean_LRV16}. Below we give an intuition on the fundamental difficulty of this problem.

\paragraph{Why is robust mean estimation in high dimensions such a difficult problem?}
To understand this, assume that gradients are distributed according to a high-dimensional Gaussian distribution $\N(\bzero,I)$. It is a well known fact that samples from such a distribution lie around the annulus at a distance $\sqrt{d}$ from the origin, w.h.p. 
So, it would not be in the adversary's best interest to put the corrupt samples far from the annulus, as they can be trivially filtered out just based on the norm. However, the adversary can put the corrupted samples in a concentrated form around the annulus, 
which cannot be detected just based on the norm, but can shift the sample mean away from the true mean in an adversarially chosen direction. This implies that filtering based on individual sample-by-sample basis is not enough, and we have to filter the outliers collectively, i.e., using all the samples at once.
This makes devising computationally-efficient decoding algorithm that provide good approximation guarantees highly non-trivial.
Recently, \cite{Robust_mean_LRV16} and \cite{Robust_mean_DiakonikolasKK016} in their breakthrough papers independently provided {\em computationally efficient} algorithms for RME that give {\em dimension-independent} error guarantees.
Following these papers, there had been a flurry of research improving upon their results in various directions; see \cite{RecentAdvances_RobustStatistics19} and references therein.

\paragraph{Difficulty of our problem.}
When all local datasets $\calD_r,r\in[R]$, are the same, we have that for every $\bx\in\R^d$, all $\calF_r^{\otimes b}(\bx), r\in[R]$ are the same and so are $\nabla F_r(\bx), r\in[R]$. By letting $\nabla F_r(\bx):=\nabla F(\bx)$ and $\calF_r^{\otimes b}(\bx):=\calF^{\otimes b}(\bx)$, we can map our problem to RME as follows:
Let the $R$ gradient samples come i.i.d.\ from $\calF^{\otimes b}(\bx)$ with a uniform distribution, out of which an $\eps$ fraction may be adversarially corrupt. Note that each gradient sample is unbiased and has mean equal to $\nabla F(\bx)$ and has variance bounded by $\frac{\sigma^2}{b}$ (see \eqref{same_mean}, \eqref{reduced_variance}), and our goal is to estimate the mean $\nabla F(\bx)$.
Note that not all results on RME are applicable to our setting, as most of these results have been derived assuming particular distributions, e.g., Gaussian, from which the samples are drawn; whereas, in this paper we only assume the variance bound on the gradients.

Note that our problem is more general than the one described above. In our setting, different workers have different datasets, which adds further complications.
In RME, all samples come from the {\em same} distribution, whereas, in our problem, different samples come from {\em different} local distributions (which are all uniform but over distinct supports, with potentially different support sizes) -- for any worker $r\in[R]$, the $r$'th gradient sample $\bg_r(\bx)$ comes uniformly at random from $\calF_r^{\otimes b}(\bx)$, 
where $\calF_r^{\otimes b}(\bx)$'s are different for different $r\in[R]$.
Note that we get one sample from each distribution, and we want to estimate $\nabla F(\bx)=\frac{1}{R}\sum_{r=1}^R\nabla F_r(\bx)$, the average of the local means.
Observe that, if we do not have any correlation among the local datasets (e.g., a probabilistic model for the data or assuming that all local datasets are the same), it would be impossible to solve this problem using RME-type algorithms and get a meaningful result; see also \Remarkref{bounded_local-global}.
To make this tractable, we assume that the local datasets are not arbitrarily far from each other, in the sense that the true local gradients (evaluated at an arbitrary point in the domain) at any worker are at most $\kappa$ away from the global gradient; see \eqref{bounded_local-global}. As mentioned in \Remarkref{bounded_local-global}, all other papers in distributed optimization literature that provide Byzantine-resilience for optimizing a generic objective function \eqref{eq:problem-expression} under standard assumptions either assume that the local data at different workers is generated i.i.d.\ from the same distribution, or that all workers can access to the same data.

\paragraph{Reducing the sample complexity by increased mini-batch size.} \label{sample-complexity_vs_mini-batch}
It is known that, in the problem of RME, in order to get a good estimator, the sample complexity (i.e., the number of samples required) for robustly estimating the mean grows at least linearly with the dimension $d$ \cite{Robust_mean_LRV16}.
In a setting where all local datasets are the same, this implies that to robustly estimating $\nabla F(\bx)$, the number of workers $R$ should grow linearly with the dimension $d$.
In a distributed setup, since it is not practical to increase the number of workers with the dimension of the problem, we address this issue by increasing the mini-batch size $b$.
As noted in \eqref{reduced_variance}, by increasing the mini-batch size $b$, the variance of the resulting gradients (which are samples from $\calF^{\otimes b}(\bx)$) reduces by a factor of $b$, which implies that as we increase $b$, the resulting gradients become closer to the mean $\nabla F(\bx)$; and as we see later, this will cut down the requirement of $R$ growing linearly with $d$.
Observe that it is crucial that increasing the mini-batch size does not change the mean, as we want to estimate $\nabla F(\bx)$ using $\calF^{\otimes b}(\bx)$.
As we show later, this argument, in fact, holds true in a more general setting that we consider in this paper, where different workers have different datasets, and we get one gradient sample from each $\calF_r^{\otimes b}(\bx)$. In this case, the approximation error will inevitably be affected by $\kappa$, which captures the dissimilarity between the local datasets. Increasing the mini-batch size $b$ will result in making local stochastic gradients close to their corresponding local true gradients, and which, on average (after removing the effect of corrupt gradients), will be about $\kappa$ distance away from the global gradient. See also \Figureref{robust-grad-est}.

\begin{SCfigure}
\begin{tikzpicture}[scale=1.25]
\draw[blue,thick,dashed] (-2.5,1) circle (0.75cm); \draw [fill, black] (-2.5,1) circle (0.05cm); \node [scale=0.75] at (-2.5,0.75) {$\nabla F_1(\bx)$}; \draw [fill, blue] (-2.25,2.0) circle (0.05cm); \node [scale=0.75] at (-2.45,2.0) {$\bg_1$};

\draw[blue,thick,dashed] (-1.25,2.5) circle (0.75cm); \draw [fill, black] (-1.25,2.5) circle (0.05cm); \node [scale=0.75] at (-1.25,2.7) {$\nabla F_2(\bx)$}; \draw [fill, blue] (-0.2,1.5) circle (0.05cm); \node [scale=0.75] at (-0.4,1.5) {$\bg_2$}; \draw [cyan,thick] (-0.35,1.5) ellipse (0.25cm and 0.13cm); 

\draw[blue,thick,dashed] (0.5,2.65) circle (0.75cm); \draw [fill, black] (0.5,2.65) circle (0.05cm); \node [scale=0.75] at (0.5,2.45) {$\nabla F_3(\bx)$}; \draw [fill, blue] (0.1,3) circle (0.05cm); \node [scale=0.75] at (0.35,3) {$\bg_3$};  \draw [cyan,thick] (0.27,3) ellipse (0.28cm and 0.13cm); \draw [->] (0.5,2.65) -- (1,3.2); \node [scale=0.65] at (0.9,2.85) {$\frac{\sigma}{\sqrt{b}}$};

\draw[blue,thick,dashed] (2.35,1.85) circle (0.75cm); \draw [fill, black] (2.35,1.85) circle (0.05cm); \node [scale=0.75] at (2.6,2) {$\nabla F_4(\bx)$}; \draw [fill, blue] (3.25,1.5) circle (0.05cm); \node [scale=0.75] at (3.4,1.65) {$\bg_4$};  \draw [rotate around={45:(3.35,1.6)}, cyan, thick] (3.35,1.6) ellipse (0.25cm and 0.15cm);

\draw[blue,thick,dashed] (2.9,0.1) circle (0.75cm); \draw [fill, black] (2.9,0.1) circle (0.05cm); \node [scale=0.75] at (3.1,0.25) {$\nabla F_5(\bx)$}; \draw [fill, blue] (3.1,-0.45) circle (0.05cm); \node [scale=0.75] at (3.25,-0.3) {$\bg_5$}; \draw [rotate around={48:(3.2,-0.35)}, cyan, thick] (3.2,-0.35) ellipse (0.23cm and 0.17cm);

\draw[blue,thick,dashed] (2.2,-1.35) circle (0.75cm); \draw [fill, black] (2.2,-1.35) circle (0.05cm); \node [scale=0.75] at (2.45,-1.55) {$\nabla F_6(\bx)$}; \draw [fill, blue] (2.4,-1.15) circle (0.05cm); \node [scale=0.75] at (2.6,-1.1) {$\bg_6$}; \draw [rotate around={15:(2.53,-1.12)}, cyan, thick] (2.53,-1.12) ellipse (0.23cm and 0.17cm);

\draw[blue,thick,dashed] (1.25,-2) circle (0.75cm); \draw [fill, black] (1.25,-2) circle (0.05cm); \node [scale=0.75] at (1.45,-2.2) {$\nabla F_7(\bx)$}; \draw [fill, blue] (1.35,-1.25) circle (0.05cm); \node [scale=0.75] at (1.25,-1.45) {$\bg_7$}; \draw [rotate around={35:(1.28,-1.35)}, cyan, thick] (1.28,-1.35) ellipse (0.23cm and 0.17cm);

\draw[red,thick,dashed] (-0.85,-2.2) circle (0.75cm); \draw [fill, black] (-0.85,-2.2) circle (0.05cm); \node [scale=0.75] at (-0.85,-2.45) {$\nabla F_8(\bx)$};
\draw[red,thick,dashed] (-2.1,-1.2) circle (0.75cm); \draw [fill, black] (-2.1,-1.2) circle (0.05cm); \node [scale=0.75] at (-2.1,-1.45) {$\nabla F_9(\bx)$};

\draw [fill, black] (0.0,0.0) circle (0.1cm); \node [scale=0.75] at (-0.1,-0.3) {$\nabla F(\bx)$}; \draw [->] (0.0,0.0) -- (-1.22,2.47); \node [scale=0.8] at (-0.7,1) {$\kappa$};

\draw [fill, blue] (0.5,0.5) circle (0.1cm); \node [scale=0.75] at (0.6,0.25) {$\nabla F_{[2:7]}(\bx)$};
\draw [fill, cyan] (1.05,0.8) circle (0.1cm); \node [scale=0.75] at (1.4,0.6) {$\bg_{[2:7]}$}; \draw [rotate around={-22:(1.29,0.69)}, cyan, thick] (1.29,0.69) ellipse (0.45cm and 0.23cm);
\draw [fill, black] (0.55,1) circle (0.1cm); \node [scale=0.75] at (0.5,1.25) {$\btg(\bx)$};

\end{tikzpicture}
\centering
\caption{We have total 9 workers, out of which 2 workers (numbered 8, 9) are Byzantine. Since different workers have different datasets, their true local gradients (denoted by $\nabla F_i(\bx)$) are placed in different locations. The blue dashed circles (numbered 1 to 7) are centered at the true local gradients of honest workers, and have their radius equal to the standard deviation $\nicefrac{\sigma}{\sqrt{b}}$, which implies that their stochastic gradient samples $\bg_i$ may not lie inside the blue circles.
The red dashed circles correspond to the Byzantine workers, and we do not have any control over them.
Let $\{\bg_2,\hdots,\bg_7\}$ be the subset $\calS$ of uncorrupted gradients ensured by the first part of \Theoremref{gradient-estimator}.
Let the robust gradient estimator in the second part of \Theoremref{gradient-estimator} outputs $\btg(\bx)$ as an estimate of $\bg_{[2:7]}:=\frac{1}{6}\sum_{i=2}^7\bg_i$. 
To bound the approximation error $\bbE\|\btg(\bx)-\nabla F(\bx)\|$, note that $\bbE\|\btg(\bx)-\nabla F(\bx)\| \leq \bbE\|\btg(\bx)-\bg_{[2:7]}(\bx)\|+\bbE\|\bg_{[2:7]}-\nabla F_{[2:7]}(\bx)\| + \|\nabla F_{[2:7]}(\bx)-\nabla F(\bx)\|$, where the first term can be bounded by $\calO(\sigma_0\sqrt{\eps+\eps'})$, the second term can be bounded by the square root of $\nicefrac{\sigma^2}{6b}$, which comes from the variance bound for sampling, and the third term can be bounded by $\kappa$, which is the gradient dissimilarity bound from \eqref{bounded_local-global}. Note that the $\kappa$ term is inevitable because, in the presence of a constant number of Byzantine workers, intuitively, $\nabla F_{[2:7]}(\bx)$ will shift away from $\nabla F(\bx)$ by a constant fraction of $\kappa$.}
\label{fig:robust-grad-est}
\end{SCfigure}
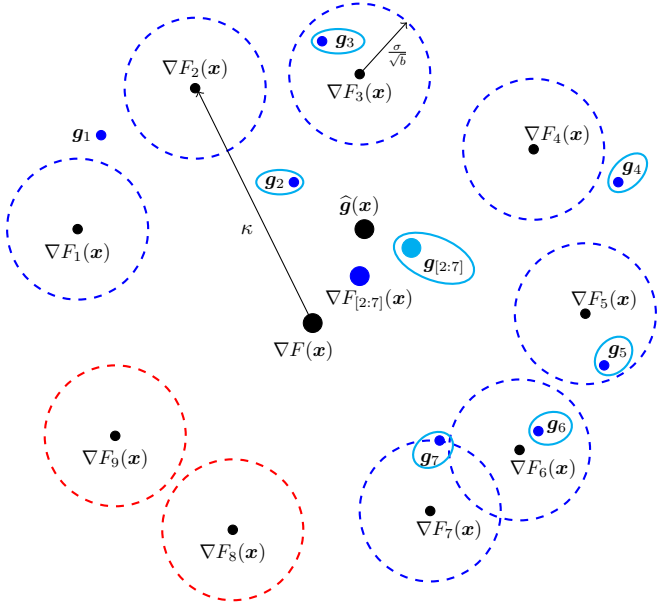
Our main result for robust gradient estimation is as follows:
\begin{theorem}[Robust Gradient Estimation]\label{thm:gradient-estimator}
Fix an arbitrary $\bx\in\R^d$.
Suppose an $\eps$ fraction of workers are corrupt and we are given $R$ gradients $\btlg_1(\bx),\hdots,\btlg_R(\bx)\in\R^d$, where $\btlg_r(\bx)=\bg_r(\bx)$ is a uniform sample from $\calF_r^{\otimes b}(\bx)$ satisfying \eqref{same_mean}, \eqref{reduced_variance} if the $r$'th worker is honest, otherwise can be arbitrary. Let $\btlg_i:=\btlg_i(\bx)$ for $i\in[R]$. Then, for any constant $\eps'>0$, we have the following:
\begin{enumerate}
\item {\bf Matrix concentration:} With probability $1-\exp(-\frac{\eps'^2(1-\eps)R}{16})$, there exists a subset $\calS\subset[R]$ of uncorrupted gradients of size $(1-(\eps+\eps'))R$
such that 
\begin{align}\label{mat-concen_gradient-estimator}
\lambda_{\max}\(\frac{1}{|\calS|}\sum_{i\in\calS} \(\bg_i - \bg_{\calS}\)\(\bg_i - \bg_{\calS}\)^T\) \leq \frac{24\sigma^2}{b\eps'}\left(1 + \frac{d}{(1-(\eps+\eps'))R}\right)+16\kappa^2, 
\end{align}
where $\bg_{\calS}:=\frac{1}{|\calS|}\sum_{i\in\calS}\bg_i$, $\kappa$ is from \eqref{bounded_local-global}, and $\lambda_{\max}$ denotes the largest eigenvalue.
\item {\bf Outlier-filtering algorithm:} If $\eps\leq\frac{1}{4} - \eps'$, then 
we can find an estimate $\btg$ of $\bg_{\calS}$ in polynomial-time with probability 1, such that $\left\| \btg - \bg_{\calS} \right\| \leq \calO\left(\sigma_0\sqrt{\eps+\eps'}\right)$, where $\sigma_0^2 = \frac{24\sigma^2}{b\eps'}\left(1 + \frac{d}{(1-(\eps+\eps'))R}\right)+16\kappa^2$.
\end{enumerate}
\end{theorem}
\begin{remark}\label{remark:grad-estimator_error-terms}
The squared approximation error has two terms $T_1=\calO\left(\frac{\sigma^2}{b\eps'}\left(1 + \frac{d}{(1-(\eps+\eps'))R}\right)(\eps+\eps')\right)$ and $T_2=\calO(\kappa^2(\eps+\eps'))$.
Here, the first error term $T_1$ appears due to the stochastic sampling of gradients, and is equal to zero when there is no gradient sampling error (i.e., $\sigma=0$), which would be the case, for instance, when workers take {\em full batch} gradients w.r.t.\ their local datasets; see also the proof of \Theoremref{full-batch-GD} for a detailed discussion for this case with an improved bound on the approximation error.
The second error term $T_2$ accounts for the dissimilarity $\kappa$ among the local datasets at different workers, and will be equal to zero when all workers are assumed to have access to the same dataset.\footnote{Note that when $\sigma=\kappa=0$, then the problem of robust gradient estimation becomes trivial, as this setting would correspond to running {\em full batch} gradient descent in a distributed manner, where all honest workers have the same data, and, therefore, send the same gradient to master, who can perform the majority vote to compute the correct gradient -- in that case, less than 1/2 fraction of Byzantine workers can be tolerated.}
\end{remark}

The statement of \Theoremref{gradient-estimator} consists of two parts: 
First, it shows an existence of a large subset $\calS$ of uncorrupted gradients having bounded concentration around their sample mean, which is a matrix concentration result; 
and second, it efficiently estimates the average of the gradients in $\calS$.
We prove the first part in \Subsectionref{bounded-variance_subset}; 
and for the second part, we use the polynomial-time outlier-filtering procedure of \cite{Resilience_SCV18}, 
which we describe in detail in \Algorithmref{robust-grad-est} in \Appendixref{robust-mean-estimation},
and prove the second part of \Theoremref{gradient-estimator} in \Appendixref{poly-time-lemma_proof} by providing a comprehensive analysis of the outlier-filtering procedure.

We also provide an intuition behind the outlier-filtering procedure from \cite{Resilience_SCV18} in \Appendixref{intuition_algo}
and its running time analysis in \Appendixref{running-time_robust-grad-est}, where we show that 
 \Algorithmref{robust-grad-est} performs at most $\calO(R)$ SVD computations of $d\times R$ matrices, 
 which can be performed in $\calO(dR^2\min\{d,R\})$ time \label{running-time_robust-grad-est-page} in total; hence, \Algorithmref{robust-grad-est} runs in polynomial-time.

Note that the same filtering procedure has also been used by \cite{SuX_Byz19,Bartlett-Byz_nonconvex19} in the context of Byzantine-robust {\em full batch} gradient descent, where data comes i.i.d.\ from a probability distribution, as opposed to the {\em stochastic} gradient descent considered in this paper. In our setting, different workers may have different datasets, and we do not make any probabilistic assumption on the data generation. Our results are derived under standard SGD assumptions in the distributed setting.

\subsection{Matrix concentration}\label{subsec:bounded-variance_subset}
Now we prove the first part of \Theoremref{gradient-estimator}. For that, we need to show an existence of a subset $\calS$ of the $R$ gradients (out of which an $\eps$ fraction is corrupted) that has good concentration, as quantified in \eqref{mat-concen_gradient-estimator}. 
We want to point out that if workers send full-batch {\em deterministic} gradients (as in \cite{SuX_Byz19,Bartlett-Byz_nonconvex19}), then we can take $\calS$ to be the set of all honest workers and get a deterministic bound; see \Theoremref{gradient-estimator_GD} in \Appendixref{convergence_full-batch-GD} for more details.
However, when workers compute mini-batch {\em stochastic} gradients, showing an existence of such a set is non-trivial.
Note that, though \cite{SuX_Byz19,Bartlett-Byz_nonconvex19} studied full-batch gradient descent, they also proved a matrix concentration result, which they needed because they minimize the population risk, whereas, we do not need such a result because, instead, we minimize the empirical risk in this paper. On the other hand, we also prove a matrix concentration bound as stated in the first part of \Theoremref{gradient-estimator}, whose need arises because of the stochasticity of gradients (due to SGD). This bound is of a very different nature than theirs and requires only the bounded variance assumption \eqref{bounded_local-variance} of local gradients to prove, whereas, their bound requires distributional assumptions (sub-exponential/sub-Gaussian) on local gradients.

In order to prove \eqref{mat-concen_gradient-estimator} in the first part of \Theoremref{gradient-estimator}, first we show a separate matrix concentration bound in the following lemma, and then we show how we can use that to prove our desired bound \eqref{mat-concen_gradient-estimator}.

\begin{lemma}\label{lem:subset_variance}
Suppose there are $m$ independent distributions $p_1,p_2,\hdots,p_m$ in $\R^d$ such that $\bbE_{\by\sim p_i}[\by]=\vct{\mu}_i, i\in[m]$ and each $p_i$ has bounded variance in all directions, i.e., $\bbE_{\by\sim p_i}[\langle \by - \vct{\mu}_i, \bv\rangle^2] \leq \sigma_{p_i}^2$ holds for all unit vectors $\bv\in\R^d$. Take an arbitrary $\eps'\in(0,1]$. Then, given $m$ independent samples $\by_1,\by_2,\hdots,\by_m$, where $\by_i\sim p_i$, with probability $1- \exp(-\eps'^2m/16)$, there is a subset $\calS$ of $(1-\eps')m$ points such that
\begin{align*}
\lambda_{\max}\(\frac{1}{|\calS|}\sum_{i\in\calS} \(\by_i-\vct{\mu}_i\)\(\by_i-\vct{\mu}_i\)^T \) \leq \frac{4\sigma_{p_{\max}}^2}{\eps'}\left(1 + \frac{d}{(1-\eps')m}\right), \quad \text{ where } \sigma_{p_{\max}}^2=\max_{i\in[m]}\sigma_{p_{i}}^2.
\end{align*}
\end{lemma}
\Lemmaref{subset_variance} is a generalization of \cite[Proposition B.1]{Untrusted-data_Charikar17}, where the $m$ samples $\by_1,\hdots,\by_m$ are drawn independently from a {\em single} distribution $p$ with mean $\bmu$ and variance bound of $\sigma_p^2$.
Note that, in our setting, different $\by_i$'s may come from different distributions, which may have different means and variances. 
We provide a proof of \Lemmaref{subset_variance} in \Appendixref{subset_bounded-variance_proof}.

\begin{proof}[Proof of the first part of \Theoremref{gradient-estimator}]
In order to use \Lemmaref{subset_variance} in our robust gradient estimation problem, 
for every $i\in[R]$ corresponding to the honest worker, take $p_i$ to be a uniform distribution over $\calF_i^{\otimes b}(\bx)$, which implies, using \eqref{same_mean} and \eqref{reduced_variance}, that its associated mean and variance are $\bmu_i=\nabla F_i(\bx)$ and $\sigma_{p_i}^2 = \frac{\sigma^2}{b}$, respectively.
It is easy to see that the hypothesis of \Lemmaref{subset_variance} is satisfied with $\by_i=\bg_i(\bx), \bmu_i=\nabla F_i(\bx), \sigma_{p_i}^2 = \frac{\sigma^2}{b}$:
\begin{align*}
\bbE[\langle \bg_i(\bx) - \nabla F_i(\bx), \bv\rangle^2] \ \stackrel{\text{(a)}}{\leq} \ 
 \bbE[\| \bg_i(\bx) - \nabla F_i(\bx)\|^2]\cdot \|\bv\|^2 \ \stackrel{\text{(b)}}{\leq} \ \frac{\sigma^2}{b}, 
\end{align*}
where (a) follows from the Cauchy-Schwarz inequality and (b) uses \eqref{reduced_variance} and $ \l\| \bv\| \leq 1$.

We are given $R$ gradients, out of which at least $(1-\eps)R$ are according to the correct distribution. 
By considering only the uncorrupted gradients (i.e., taking $m=(1-\eps)R$), we have from 
\Lemmaref{subset_variance} that there exists a subset $\calS$ of $R$ gradients of size $(1-\eps')(1-\eps)R \geq (1-(\eps+\eps'))R$ that satisfies 
\begin{align}\label{bounded_subset_variance}
\lambda_{\max}\(\frac{1}{|\calS|}\sum_{i\in\calS} \(\bg_i(\bx) - \nabla F_i(\bx)\)\(\bg_i(\bx) - \nabla F_i(\bx)\)^T \) \leq \frac{4\sigma^2}{b\eps'}\left(1 + \frac{d}{(1-(\eps+\eps'))R}\right).
\end{align}

Note that \eqref{bounded_subset_variance} is bounding the deviation of the points in $\calS$ from their respective means $\nabla F_i(\bx)$. However, in \eqref{mat-concen_gradient-estimator}, we need to bound the deviation of the points in $\calS$ from their sample mean $\bg_{\calS}(\bx)=\frac{1}{|\calS|}\sum_{i\in\calS}\bg_i(\bx)$.
Using the gradient dissimilarity bound \eqref{bounded_local-global} together with some algebraic manipulations provided in \Appendixref{remaining_part1-robust-grad}, we can show that \eqref{bounded_subset_variance} implies $\lambda_{\max}\(\frac{1}{|\calS|}\sum_{i\in\calS} \(\bg_i(\bx) - \bg_{\calS}(\bx)\)\(\bg_i(\bx) - \bg_{\calS}(\bx)\)^T\) \leq \frac{24\sigma^2}{b\eps'}\left(1 + \frac{d}{(1-(\eps+\eps'))R}\right)+16\kappa^2$, which completes the proof of the first part of \Theoremref{gradient-estimator}.
\end{proof}

\section{Gradient Compression}\label{sec:compression}
In this section, we make \Algorithmref{Byz-SGD} from \Sectionref{our-results} communication-efficient by having the workers send {\em compressed} gradients instead of full gradients.
There are many ways to compress gradients: {\sf (i)} {\em sparsify} them by taking their top $k$ entries, magnitude wise (denoted by $\topk$) or random $k$ entries (denoted by $\randk$) \cite{memSGD,alistarh-sparsified}; {\sf (ii)} {\em quantize} them to a small number of bits, either by using a deterministic quantizer (e.g., taking the sign vector of gradients \cite{signsgd1,stich-signsgd-19}) or a randomized quantizer \cite{QSGD,teertha}; or {\sf (iii)} using a combination of both \cite{Qsparse-local-sgd19}. In \Subsectionref{ruling-out-quantizers} below, first we rule out the use of quantizers in our setting, and then in \Subsectionref{compression-randk}, we show that the $\randk$ sparsifier works well for our purpose. We state our main results with gradient compression using the $\randk$ sparsifier in \Subsectionref{main-results_randk}.

\subsection{Quantization does not suffice}\label{subsec:ruling-out-quantizers}
First we argue that quantizers do not serve our purpose well. The reason being, since the variance bound for any quantizer ensures that quantization does not blow up the length of the input vector by much, and when applied to stochastic gradients, it turns out (and as shown below) that the variance of the quantized mini-batch stochastic gradients does not decrease with increased mini-batch size in general. Note that the reduction in variance with increased mini-batch size was crucial to our robust gradient estimation procedure described in \Sectionref{robust-grad-est}, and it played an instrumental role in our convergence results (see the discussion in \Subsectionref{remarks_main-result}), where we could control the approximation error of our solutions by increasing the mini-batch size of stochastic gradients. This important property of our solution will no longer hold if we use quantizers for gradient compression, as explained below in more detail.

Take any unbiased quantizer, say, $Q_s$ from \cite{QSGD}, which probabilistically maps each component of the vector to one of $s$ levels. It was shown in \cite{QSGD} that $Q_s$ is unbiased and has bounded variance, i.e., for every vector $\bv\in\R^d$, we have $\bbE_Q[Q_s(\bv)]=\bv$, and $\bbE_Q\|Q_s(\bv)-\bv\|^2 \leq \beta_{d,s}\|\bv\|^2$, where $\beta_{d,s}=\min\{\nicefrac{\sqrt{d}}{s},\nicefrac{d}{s^2}\}$. Fix any worker $r\in[R]$. By taking $\bv$ to be a stochastic gradient $\nabla F_{r,i}(\bx)$ where $i\in_U[n_r]$, and applying $Q_s$ on that, we get $\bbE_{Q,i}[Q_s(\nabla F_{r,i}(\bx))]=\nabla F_r(\bx)$, where $\nabla F_r(\bx)$ is the true local gradient at worker $r$ w.r.t.\ the model $\bx$, and $\bbE_{Q,i}\|Q_s(\nabla F_{r,i}(\bx))-\nabla F_r(\bx)\|^2 \leq (1+\beta_{d,s})\bbE_i\|\nabla F_{r,i}(\bx)\|^2$. 
Under the assumption that the local stochastic gradients have bounded second moment, i.e., $\bbE_i\|\nabla F_{r,i}(\bx)\|^2\leq G^2$ for some finite $G$, we have $\bbE_{Q,i}\|Q_s(\nabla F_{r,i}(\bx))-\nabla F_r(\bx)\|^2 \leq (1+\beta_{d,s})G^2$.
Recall from \eqref{same_mean}, \eqref{reduced_variance} that $\nabla F_{r,H_b}(\bx)$ denote a uniform sample from $\calF_r^{\otimes b}(\bx)$ and is a mini-batch stochastic gradient with batch size $b$. 
Even though $\mathrm{Var}(\nabla F_{r,H_b}(\bx))\leq \nicefrac{\sigma^2}{b}$ (see \eqref{reduced_variance}), i.e., by taking the mini-batch SGD with batch size $b$, the variance reduces by a factor of $b$; unfortunately, $\mathrm{Var}(Q_s(\nabla F_{r,H_b}(\bx)))$ remains the same and does not reduce by any factor, i.e., $\mathrm{Var}(Q_s(\nabla F_{r,H_b}(\bx)))=\bbE_{Q,H_b}\|Q_s(\nabla F_{r,H_b}(\bx))-\nabla F_r(\bx)\|^2 \leq (1+\beta_{d,s})G^2$. This is because when we take an average of different vectors, which are not far from each other (implied by the variance bound \eqref{bounded_local-variance}) and all have approximately the same length (implied by the bounded second moment assumption), the resulting vector will also have similar length, which does not decrease with the number of vectors. Formally, we have $\bbE_{r,H_b}\|\nabla F_{r,H_b}(\bx)\|^2 \leq \bbE_i\|\nabla F_{r,i}(\bx)\|^2 \leq G^2$,\footnote{The first inequality $\bbE_{r,H_b}\|\nabla F_{r,H_b}(\bx)\|^2 \leq \bbE_i\|\nabla F_{r,i}(\bx)\|^2$ follows from the Jensen's inequality in case when $H_b$ is a collection of $b$ elements drawn uniformly at random from $[n_r]$ {\em with} replacement. It is not crucial here, but we can show a similar bound when $H_b$ is a collection of $b$ elements drawn uniformly at random from $[n_r]$ {\em without} replacement.} i.e., unlike variance, the second moment bound is not affected by taking a larger mini-batch.

Observe that, for $\calH_b\in_U\binom{[n_r]}{b}$, if we compute the variance of $\frac{1}{b}\sum_{i\in \calH_b}Q_s\left(\nabla F_{r,i}(\bx)\right)$ (instead of $Q_s\left(\frac{1}{b}\sum_{i\in \calH_b}\nabla F_{r,i}(\bx)\right)$), we would get the desired reduction in variance by a factor of $b$. However, computing $b$ quantized gradients locally at each worker and then taking their average defeats the purpose of quantization, as (in addition to being computationally expensive, because we would be quantizing $b$ vectors separately and then taking an average) the entries in the resulting vector may not have low precision.  
Note that $Q_s\left(\frac{1}{b}\sum_{i\in \calH_b}\nabla F_{r,i}(\bx)\right)$ may be far from $\frac{1}{b}\sum_{i\in \calH_b}Q_s\left(\nabla F_{r,i}(\bx)\right)$ in general, as quantizers are data-dependent, and consequently, $\mathrm{Var}\left(\frac{1}{b}\sum_{i\in \calH_b}Q_s\left(\nabla F_{r,i}(\bx)\right)\right)$ could also be far from $\mathrm{Var}\left(Q_s\left(\frac{1}{b}\sum_{i\in \calH_b}\nabla F_{r,i}(\bx)\right)\right)$, implying that the results obtained for $\frac{1}{b}\sum_{i\in \calH_b}Q_s\left(\nabla F_{r,i}(\bx)\right)$ may not hold for $Q_s\left(\frac{1}{b}\sum_{i\in \calH_b}\nabla F_{r,i}(\bx)\right)$. 
Because of these reasons, using quantization for gradient compression is undesirable for the purpose of this paper.

\subsection{Using the $\randk$ sparsifier}\label{subsec:compression-randk}
Now we show that the (properly scaled) $\randk$ sparsifier has all the properties we want, i.e., unbiasedness, bounded variance, and variance reduction in proportion to the mini-batch size. 
Recall that $\randk:\R^d\to\R^d$ is a randomized map, that takes a vector in $\R^d$, selects its $k$ entries uniformly at random, 
and set the remaining $(d-k)$ entires to zero. 
Formally, we define $\randk$ as follows:
For any subset $\calK\in\binom{[d]}{k}$, define an operator $\select_{\calK}:\R^d\to\R^d$ such that for any vector $\bv\in\R^d$, $\select_{\calK}(\bv)_i = \bv_i$ when $i\in\calK$ and $\select_{\calK}(\bv)_i = 0$ when $i\in[d]\setminus\calK$. 
Note that $\randk(\bv)$ is equivalent to first selecting $\calK\in_U\binom{[d]}{k}$ and then outputting $\select_{\calK}(\bv)$. 
Let $K$ and $H_b$ be random variables respectively taking values in $\binom{[d]}{k}$ and $\binom{[n]}{b}$ with uniform distribution.
It is easy to show that $\bbE_{\calK\leftarrow K}[\frac{d}{k}\cdot \select_{\calK}(\bv)] = \bv$. The following lemma is proved in \Appendixref{compression}.
\begin{lemma}\label{lem:randk}
Fix any worker $r\in[R]$.
Suppose the local stochastic gradients at worker $r$ have uniformly bounded second moment, i.e., $\bbE_{i\in_U[n_r]}\|\nabla F_{r,i}(\bx)\|^2 \leq G^2,\forall \bx\in\calC$.\footnote{This is a standard assumption in SGD literature with compressed gradient \cite{memSGD,Qsparse-local-sgd19}.} Then we have
\begin{align}
&\bbE_{\calK\leftarrow K,H_b}\left[\frac{d}{k}\cdot\select_{\calK}\left(\nabla F_{r,H_b}(\bx)\right)\right] = \nabla F_r(\bx) \label{same_mean_randk} \\ 
&\bbE_{\calK\leftarrow K,H_b}\left\|\frac{d}{k}\cdot\select_{\calK}\left(\nabla F_{r,H_b}(\bx)\right) - \nabla F_r(\bx)\right\|^2 \leq  \frac{d}{k}\frac{G^2}{b}. \label{reduced_variance_randk}
\end{align}
\end{lemma}
In our Byzantine-resilient SGD algorithm with compressed gradients, if each worker selects random $k$ coordinates independent of each other, though master receives $R$ gradients, the non-zero entries in compressed gradients from different workers may not confine to the same $k$ coordinates, and may spread across all over $d$ coordinates. So, in this case, to decode, master has to treat them as vectors in $\R^d$, which will lead to a squared approximation error (in the robust gradient estimator of \Theoremref{gradient-estimator})
of $\calO(\widetilde{\sigma}_0^2(\eps+\eps'))$, where $\widetilde{\sigma}_0^2=\frac{24dG^2}{kb\eps'}\left(1 + \frac{d}{(1-(\eps+\eps'))R}\right)+16\kappa^2$ (by replacing $\sigma^2$ by $\frac{d}{k}G^2$ in \Theoremref{gradient-estimator}).
Observe that this scales as $\frac{d^2}{kbR}$, as opposed to $\frac{d}{bR}$, which was achievable without gradient compression (see \Theoremref{gradient-estimator}). 

To mitigate this, we use a simple idea, where, instead of workers sampling $k$ coordinates, master picks $k$ random coordinates and broadcast them, and dictate all the workers to restrict their entries to those coordinates only.
In addition to saving on communication, this will achieve two things: 
\begin{enumerate}
\item Since all $R$ gradients have their non-zero entries confined to the same $k$ coordinates, we can assume that master receives $R$ vectors in $\R^k$. With this, $\widetilde{\sigma}_0^2$ in the approximation error becomes $\frac{24dG^2}{kb\eps'}\left(1 + \frac{k}{(1-(\eps+\eps'))R}\right)+16\kappa^2$, which, assuming $k\geq(1-(\eps+\eps'))R$, scales as $\frac{d}{bR}$, same as in the case without gradient compression.

\item Since the decoding algorithm now works on vectors in $\R^k$ (instead of $\R^d$), the decoding complexity also reduces: 
As mentioned on \Pageref{running-time_robust-grad-est-page} in \Sectionref{robust-grad-est}, the outlier-filtering procedure that we use from \cite{Resilience_SCV18} takes $\calO(dR^2\min\{d,R\})$ time with uncompressed gradients.
When master receives compressed gradients, which are vectors in $\R^k$, master needs to perform at most $\calO(R)$ SVD computations of $k\times R$ matrices, which would take $\calO(kR^2\min\{k,R\})$ time in total. So, the decoding complexity under compression reduces by a factor of at least $\frac{d}{k}$, which could be as large as $\Omega(\frac{d}{R})$ when $k\geq(1-(\eps+\eps'))R$.
\end{enumerate}
Our Byzantine-resilient SGD algorithm with compressed gradients is described in \Algorithmref{algo_compression}.
\begin{algorithm}[t]
   \caption{Byzantine-Resilient SGD with Gradient Compression}\label{algo:algo_compression}
\begin{algorithmic}[1]
   \STATE {\bf Initialize.} Set $\bx^0 := \bzero$. Fix a constant learning rate $\eta$ and a mini-batch size $b$.
   \FOR{$t=0$ {\bfseries to} $T-1$}
  \STATE \textbf{On Workers:}
   \FOR{$r=1$ {\bfseries to} $R$}
   \STATE Receive $\bx^t$ and $\calK^t\in_U\binom{[d]}{k}$ from master.
   \STATE $\bp_r(\bx^t) = \frac{1}{b}\sum_{i\in\calH_b}\nabla F_{r,i}(\bx^t)$, where $\calH_b\in_U\binom{[n_r]}{b}$, i.e., $\bp_r(\bx^t)$ is a mini-batch stochastic gradient with batch size $b$.
   \STATE $\bg_{\calK^t,r}(\bx^t) = \frac{d}{k}\cdot\select_{\calK^t}(\bp_r(\bx^t))$.
   \STATE $\btlg_{\calK^t,r}(\bx^t) = 
   \begin{cases}
   \bg_{\calK^t,r}(\bx^t) & \text{ if worker $r$ is honest}, \\
   \divideontimes & \text{if worker $r$ is corrupt},
   \end{cases}$
   
   where $\divideontimes$ is an arbitrary vector in $\R^d$.
   \STATE Send $\btlg_{\calK^t,r}(\bx^t)$ to master. 
   \ENDFOR
   
   \STATE \textbf{At Master:}
   \STATE Select a $k$-size subset $\calK^t\in_U\binom{[d]}{k}$ uniformly at random from $[d]$ and broadcast to all the workers.
   \STATE Receive $\{\btlg_{\calK^t,r}(\bx^t)\}_{r=1}^R$ from the $R$ workers.
   \STATE Apply the decoding algorithm {\sc RGE} on $\{\btlg_{\calK^t,r}(\bx^t)\}_{r=1}^R$. Let $$\btg_{\calK^t}(\bx^t)=\textsc{RGE}(\btlg_{\calK^t,1}(\bx^t),\hdots,\btlg_{\calK^t,R}(\bx^t)).$$
   \STATE Update the parameter vector: 
      \begin{align*}
   \bhx^{t+1} = \bx^{t} - \eta\btg_{\calK^t}(\bx^{t}); \qquad 
   \bx^{t+1} = \Pi_{\calC}\left(\bhx^{t+1}\right).
   \end{align*}
   \STATE Broadcast $\bx^{t+1}$ to all workers.
   \ENDFOR
\end{algorithmic}
\end{algorithm}
\subsection{Main results}\label{subsec:main-results_randk}
To state our main results, 
we define a set of compressed mini-batch stochastic gradients for any $\calK\in\binom{[d]}{k}$ as follows, where $r\in[R]$.
\begin{align*}
\calF_{\calK,r}^{\otimes b}(\bx)=\left\{\frac{d}{k}\cdot\select_{\calK}\Big(\frac{1}{b}\sum_{i\in\calH_b}\nabla F_{r,i}(\bx)\Big): \calH_b\in \binom{[n_r]}{b}\right\}.
\end{align*}
Note that a uniformly random sample from $\calF_{\calK,r}^{\otimes b}(\bx)$ with $\calK\in_U\binom{[d]}{k}$ will satisfy \eqref{same_mean_randk} and \eqref{reduced_variance_randk}.

As stated in \Algorithmref{algo_compression}, each iteration of our algorithm needs to estimate the true gradient robustly from the {\em compressed} gradients.
For that, we need to prove a robust gradient estimation result (similar to \Theoremref{gradient-estimator}) for compressed gradients. 
\begin{theorem}[Robust Gradient Estimation with Compressed Gradients]\label{thm:gradient-estimator_randk}
Suppose the stochastic gradients at all workers have bounded second moments, 
i.e., for every $r\in[R]$, we have $\bbE_{i\in_U[n_r]}\|\nabla F_{r,i}(\bx)\|^2 \leq G^2,\forall\bx\in\calC$.
Fix an arbitrary $\bx\in\calC$. Let $\calK\in_U\binom{[d]}{k}$.
Suppose an $\eps$ fraction of workers are corrupt and we are given $R$ gradients $\btlg_{\calK,1}(\bx),\hdots,\btlg_{\calK,R}(\bx)$, where $\btlg_{\calK,r}(\bx)=\bg_{\calK,r}(\bx)$ is a uniform sample from $\calF_{\calK,r}^{\otimes b}(\bx)$ if the $r$'th worker is honest, otherwise can be arbitrary.
Take an arbitrary constant $\eps'>0$. If $\eps\leq\frac{1}{4} - \eps'$, then with probability $1-\exp(-\frac{\eps'^2(1-\eps)R}{16})$, there exists a subset $\calS$ of uncorrupted gradients of size $(1-(\eps+\eps'))R$ (with $\bg_{\calK,\calS}(\bx):=\frac{1}{|\calS|}\sum_{i\in\calS}\bg_{\calK,i}(\bx)$ being its sample mean) and an estimate $\btg_{\calK}(\bx)$ of $\bg_{\calK,\calS}(\bx)$ such that $\left\| \btg_{\calK}(\bx) - \bg_{\calK,\calS}(\bx) \right\| \leq \calO\left(\sigma_0\sqrt{\eps+\eps'}\right)$, where 
$\sigma_0^2 = \frac{24dG^2}{kb\eps'}\left(1 + \frac{k}{(1-(\eps+\eps'))R}\right) + 16\kappa^2$.
\end{theorem}
\Theoremref{gradient-estimator_randk} can be proved in the same way as we proved \Theoremref{gradient-estimator}, except for the following changes: Instead of using the variance bound \eqref{reduced_variance} for uncompressed gradient, we use the variance bound \eqref{reduced_variance_randk} for compressed gradient, and also the fact that master performs decoding on vectors in $\R^k$ instead of vectors in $\R^d$.

Our convergence results with compressed gradients are stated below.
\begin{theorem}
\label{thm:convergence_randk}
Suppose an $\epsilon>0$ fraction of $R$ workers are adversarially corrupt and the stochastic gradients at all workers have bounded second moments, i.e., for every $r\in[R]$, we have $\bbE_{i\in_U[n_r]}\|\nabla F_{r,i}(\bx)\|^2 \leq G^2, \forall\bx\in\calC$.
For an $L$-smooth global objective function $F:\calC\to\R$, let \Algorithmref{algo_compression} generate a sequence of updates $\{\bx^t\}_{t=0}^T$ when run with a fixed learning rate $\eta$, where in the $t$'th iteration, master picks a random subset $\calK^t\in_U\binom{[d]}{k}$ and broadcasts it, and every worker $r\in[R]$ samples a compressed mini-batch stochastic gradient from $\calF_{\calK^t,r}^{\otimes b}(\bx^t)$.
Fix an arbitrary $\eps'>0$. If $\eps\leq\frac{1}{4} - \eps'$, then with probability $1-T\exp(-\frac{\eps'^2(1-\eps)R}{16})$, we have the following guarantees:
\begin{itemize}
\item {\bf Strongly-convex:} If $F$ is also $\mu$-strongly convex and we take $\eta=\frac{\mu}{L^2}$, then we have 
\begin{align}\label{convex_convergence_rate_randk}
\bbE\|\bx^{T} - \bx^* \|^2 \leq \left(1-\frac{\mu^2}{2L^2}\right)^T\|\bx^0-\bx^*\|^2 + \frac{2L^2}{\mu^4}\varGamma,
\end{align}
\item {\bf Non-convex:} If we take $\eta=\frac{1}{4L}$, then we have 
\begin{align}\label{nonconvex_convergence_rate_randk}
\frac{1}{T}\sum_{t=0}^T\bbE\|\nabla F(\bx^t)\|^2 &\leq \frac{8L^2}{T}\|\bx^0-\bx^*\|^2 + \varGamma,
\end{align}
\end{itemize}
In both \eqref{convex_convergence_rate_randk} and \eqref{nonconvex_convergence_rate_randk}, expectation is taken over the sampling of mini-batch stochastic gradients and also the random coordinates for compression. In both the expressions, $\varGamma=\frac{9dG^2}{(1-(\eps+\eps'))kbR} + 9\kappa^2 + 9\varUpsilon^2$ with $\varUpsilon = \calO\left(\sigma_0\sqrt{\eps+\eps'}\right)$, where $\sigma_0^2 = \frac{24dG^2}{kb\eps'}\left(1 + \frac{k}{(1-(\eps+\eps'))R}\right) + 16\kappa^2$.
\end{theorem}
\Theoremref{convergence_randk} can be proved along the lines of the proof of \Theoremref{SGD_convergence}, except for one change: Instead of using \eqref{reduced_variance} and \Theoremref{gradient-estimator} (which are for uncompressed gradients), we use \eqref{reduced_variance_randk} and \Theoremref{gradient-estimator_randk}, respectively.
 
\paragraph{Analysis of the approximation error.}
Similar to the discussion in \Subsectionref{remarks_main-result}, the approximation error term $\varGamma$ in \eqref{convex_convergence_rate_randk} and \eqref{nonconvex_convergence_rate_randk} consists of three terms: $\varGamma_1=\calO\left(\frac{dG^2}{(1-(\eps+\eps'))kbR}\right)$, which is the standard variance term in SGD convergence, $\varGamma_2=\calO(\kappa^2)$, which captures the dissimilarity among different local datasets, and $\varGamma_3=\calO\left(\left(\frac{dG^2}{kb\eps'}\left(1 + \frac{k}{(1-(\eps+\eps'))R}\right) + \kappa^2\right)(\eps+\eps')\right)$, which is due to Byzantine attacks. 
Observe that, for constant $\eps$, we have $\varGamma_1+\varGamma_2\leq\varGamma_3$ for all values of the mini-batch size $b$. This implies that, when a constant fraction of workers are corrupt, {\em the error due to combating Byzantine attacks subsumes the error due to sampling stochastic gradients}. 

Furthermore, when $k\geq(1-(\eps+\eps'))R$ (which is satisfied when $k$ is more than the number of honest workers), we have $\varGamma_3=\calO\left(\left(\frac{dG^2}{\eps'(1-(\eps+\eps'))bR}+\kappa^2\right)(\eps+\eps')\right)$, which gives the same dependence on $d,b,R$ we get from our compression-free algorithm; see \Theoremref{SGD_convergence}, where the corresponding error term is $\calO\left(\left(\frac{d\sigma^2}{\eps'(1-(\eps+\eps'))bR}+\kappa^2\right)(\eps+\eps')\right)$.
Thus, when a constant fraction of workers are corrupt (which implies $\varGamma_1+\varGamma_2\leq\varGamma_3$) and sparsity of the compressed gradients is more than the number of honest workers (which implies $\varGamma_3$ remains the same irrespective of whether we are working with full gradients or compressed gradients), the approximation error term in the case of compressed gradients would be the same (order-wise) as that in the case of full gradients.
This essentially means that we get ``compression for free''. On top of that, we get a direct saving by a factor of at least $\frac{d}{k}$ (which could be $\Omega(\frac{d}{R})$) in decoding complexity in each iteration.
Note that (ignoring term $\kappa^2$ term) $\varGamma_3$ can be made as small as possible by taking a sufficiently large mini-batch.

\paragraph{Error feedback does not help in getting faster convergence.}
It is known that compression ($\topk$ or $\randk$) without error feedback leads to slower/sub-optimal convergence. A remedy to this is to accumulate the error (the gradient information that was not sent) into the memory and add it in subsequent iterations \cite{memSGD,alistarh-sparsified}, so that eventually all the information is sent.
One of the revealing implications of the above analysis of the approximation error is that, when $\eps$ is constant and we select sufficiently many coordinates for gradient compression, then error feedback cannot help in speeding up the convergence of our Byzantine-resilient SGD algorithm with gradient compression,
as the approximation error remains the same irrespective of whether we are working with compressed gradients or not.
This is in sharp contrast with the Byzantine-free SGD with compression, where error-feedback provably improves the convergence \cite{memSGD,alistarh-sparsified}.

\section{Bounding the Local Variances and Gradient Dissimilarity in the Statistical Heterogeneous Model}\label{sec:statistical-model}
In this section, we bound the gradient dissimilarity $\kappa^2$ (from \eqref{bounded_local-global}) and local variance $\sigma^2$ (from \eqref{bounded_local-variance}) in the statistical model in heterogeneous setting, where different workers may have local data generated from potentially different distributions. The purpose of this section is to provide upper bounds on $\kappa$ and $\sigma$ in the statistical model.

Let $q_1,q_2,\hdots,q_R$ denote the $R$ probability distributions from which the local data samples at the workers are drawn. Specifically, the data samples at any worker $r$ are drawn from $q_r$ in an i.i.d.\ fashion and independently from other workers.
For $r\in[R]$, let $\calQ_r$ denote the alphabet over which $q_r$ is distributed.
For $r\in[R]$, let $f_r:\calQ_r\times\calC\to\R$ denote the local loss function at worker $r$, where $f_r(\bz,\bx)$ is the loss associated with the sample $\bz\in\calQ_r$ w.r.t.\ the model parameters $\bx\in\calC\subseteq \R^d$.
Linear regression is a classic example of this, where, if $\bz=(\bw,y)$ denote the pair of a feature vector $\bw\in\R^d$ and the response $y\in\R$, then $f_r(\bz,\bx)=\frac{1}{2}(\langle\bw,\bx\rangle - y)^2$.
For each worker $r\in[R]$, we assume that for any fixed $\bz\in\calQ_r$, the local loss function $f_r(\bz,\bx)$ is $L$-smooth w.r.t.\ $\bx$, i.e., for any $\bz\in\calQ_r$, we have $\|\nabla f_r(\bz,\bx) - \nabla f_r(\bz,\by)\|\leq L\|\bx-\by\|, \forall \bx,\by\in\calC$.

Let $\mu_r(\bx) := \bbE_{\bz\sim q_r}[f_r(\bz,\bx)]$ denote the expected value of $f_r(\bz,\bx)$, when $\bz$ is sampled from $\calQ_r$ according to $q_r$.
For any $\bx\in\calC$, let $\mu(\bx):=\frac{1}{R}\sum_{r=1}^R\mu_r(\bx)$ denote the average value of $\mu_r(\bx),r\in[R]$.

We are given $n_r$ i.i.d.\ samples $\bz_{r,1},\bz_{r,2},\hdots,\bz_{r,n_r}$ at the $r$'th worker from $q_r$.
Fix an arbitrary parameter vector $\bx\in\calC$. 
Let $\bar{f}_r(\bx) := \frac{1}{n_r}\sum_{i=1}^{n_r}f_r(\bz_{r,i},\bx)$ denote the average loss at worker $r$ on the $n_r$ samples $\bz_{r,1},\hdots,\bz_{r,n_r}$ w.r.t.\ $\bx$. Let $\bar{f}(\bx) := \frac{1}{R}\sum_{r=1}^r\bar{f}_r(\bx)$ denote the average loss across all workers.
The analogues of \eqref{bounded_local-global} and \eqref{bounded_local-variance} in this statistical heterogeneous model are the following: 
\begin{align}
\left\|\nabla\bar{f}_r(\bx) - \nabla\bar{f}(\bx)\right\|^2 &\leq \kappa^2, \quad \forall \bx\in\calC,\label{local-global_statistical} \\
\bbE_{i\in_U[n_r]}\left\|\nabla f_r(\bz_{r,i},\bx) - \nabla \bar{f}_r(\bx)\right\|^2 &\leq \sigma^2, \quad \forall \bx\in\calC. \label{local-variance_statistical}
\end{align}
We need to find good upper bounds on $\kappa$ and $\sigma$ that hold for all $r\in[R],\bx\in\calC$ with high probability.
We provide two bounds on $\kappa$, one when the local gradients at workers are assumed to be sub-exponential random vectors, and other when they are sub-Gaussian random vectors.
We provide a bound on $\sigma$ assuming that the local gradients are sub-Gaussian random vectors.
These are standard assumptions on gradients in statistical models, where data at all workers are sampled from the {\em same} distribution in an i.i.d.\ fashion \cite{ChenSX_Byz17,SuX_Byz19,Bartlett-Byz_nonconvex19}, which is in contrast to our heterogeneous data setting, where data at different workers may be sampled from {\em different} distributions.
Note that these works minimize the {\em population risk} with {\em full batch} gradient descent, whereas, we minimize the {\em empirical risk} with {\em stochastic} gradient descent. 
In particular, \cite{ChenSX_Byz17,SuX_Byz19} make sub-exponential gradient assumption and give convergence guarantees only for strong-convex objectives. On the other hand, \cite{Bartlett-Byz_nonconvex19} gives convergence guarantees for non-convex objectives, but under a stricter condition of sub-Gaussian distribution on gradients. % to prove a matrix concentration bound. 
In this paper, we provide convergence guarantees for both strongly-convex and non-convex objectives. Moreover, as opposed to \cite{ChenSX_Byz17,SuX_Byz19,Bartlett-Byz_nonconvex19}, our results are in a more general heterogeneous data model. Note that we need sub-Gaussian assumption only to bound the variance, which occurs because workers sample stochastic gradients. 
In case of full batch gradient descent, we only need sub-exponential assumption, as the variance is zero.

Now we state the distributional assumptions on local gradients.
\begin{assumption}[Sub-exponential local gradients]\label{assump:sub-exp-grad}
For every $\bx\in\calC$, the local gradient vectors at any worker $r\in[R]$ are sub-exponential random vectors, i.e., there exist non-negative parameters $(\nu,\alpha)$ such that
\begin{align}
\sup_{\bv\in\R^d:\|\bv\|=1} \bbE_{\bz\sim q_r}\left[\exp\(\lambda\left\langle \nabla f_r(\bz,\bx) - \nabla \mu_r(\bx), \bv \right\rangle\)\right] \leq \exp\(\lambda^2\nu^2/2\), \qquad \forall |\lambda|<\frac{1}{\alpha}.
\end{align}
\end{assumption}
\begin{assumption}[Sub-Gaussian local gradients]\label{assump:sub-gauss-grad}
For every $\bx\in\calC$, the local gradient vectors at any worker $r\in[R]$ are sub-Gaussian random vectors, i.e., there exists a non-negative parameter $\sigmag$ such that
\begin{align}
\sup_{\bv\in\R^d:\|\bv\|=1} \bbE_{\bz\sim q_r}\left[\exp\(\lambda\left\langle \nabla f_r(\bz,\bx) - \nabla \mu_r(\bx), \bv \right\rangle\)\right] \leq \exp\(\lambda^2\sigmag^2/2\), \qquad \forall \lambda \in \R.
\end{align}
\end{assumption}
Though, as stated above in both the assumptions, local gradients at all workers have the same parameters ($(\nu,\alpha)$ for sub-exponential and $\sigmag$ for sub-Gaussian), this is without loss of generality. In case they have different parameters ($(\nu_r,\alpha_r), r\in[R]$ for sub-exponential and $\sigma_r, r\in[R]$ for sub-Gaussian), we can take the final parameters to be the maximum of the respective local parameters -- for sub-exponential, we can take $\nu=\max_{r\in[R]}\nu_r$ and $\alpha=\max_{r\in[R]}\alpha_r$, and for sub-Gaussian, we can take $\sigmag=\max_{r\in[R]}\sigma_r$. 
\subsection{Bounding the gradient dissimilarity $\kappa$}\label{subsec:kappa-bound_stat}
In this section, we provide an upper bound on $\left\|\nabla\bar{f}_r(\bx) - \nabla\bar{f}(\bx)\right\|$.
\begin{align}
&\left\|\nabla\bar{f}_r(\bx) - \nabla\bar{f}(\bx)\right\| 
\leq \left\|\nabla\bar{f}_r(\bx) - \nabla \mu_r(\bx)\right\| + \left\|\nabla \mu_r(\bx) - \nabla \mu(\bx)\right\| + \left\|\nabla\bar{f}(\bx) - \nabla \mu(\bx)\right\| \nonumber \\
&\hspace{1.5cm}\leq \left\|\nabla\bar{f}_r(\bx) - \nabla \mu_r(\bx)\right\| + \left\|\nabla \mu_r(\bx) - \nabla \mu(\bx)\right\|  + \frac{1}{R}\sum_{r=1}^R\left\|\nabla\bar{f}_r(\bx) - \nabla \mu_r(\bx)\right\|,  \label{reduction-local-statistical}
\end{align}
where for the third term, we used $\bar{f}(\bx)=\frac{1}{R}\sum_{r=1}^R\bar{f}_r(\bx)$ and $\mu(\bx)=\frac{1}{R}\sum_{r=1}^R\mu_r(\bx)$, and applied the triangle inequality.
It follows from \eqref{reduction-local-statistical} that in order to bound $\left\|\nabla\bar{f}_r(\bx) - \nabla\bar{f}(\bx)\right\|$ uniformly over $\bx\in\calC$, it suffices to bound $\left\|\nabla \mu_r(\bx) - \nabla \mu(\bx)\right\|$ and $\left\|\nabla\bar{f}_r(\bx) - \nabla \mu_r(\bx)\right\|, \forall r\in[R]$ uniformly over $\bx\in\calC$. 

\paragraph{Bounding $\left\|\nabla \mu_r(\bx) - \nabla \mu(\bx)\right\|$.}
Note that $\nabla \mu_r(\bx) = \bbE_{\bz\sim q_r}[\nabla f_r(\bz,\bx)]$ is a property of the distribution $q_r$ from which the data samples have been drawn and so is $\nabla \mu(\bx) = \frac{1}{R}\sum_{r=1}^R\nabla \mu_r(\bx)$ the property of $q_1,\hdots,q_R$. 
Note that $\left\|\nabla \mu_r(\bx) - \nabla \mu(\bx)\right\|$ captures heterogeneity among distributions through their expected values, and is equal to zero in the i.i.d.\ homogeneous data setting of \cite{ChenSX_Byz17,Bartlett-Byz18,SuX_Byz19,Bartlett-Byz_nonconvex19}.
In order to get a meaningful bound for $\kappa$, it is reasonable to assume that this heterogeneity is bounded.
We assume a uniform bound on the $\|\nabla\mu_r(\bx) - \nabla\mu(\bx)\|$ for every $\bx\in\calC$.
\begin{assumption}\label{assump:uniform-bound-mean-heterogeneous}
For every worker $r\in[R]$, the population mean of the local gradients has a uniformly bounded deviation from the population mean of the global gradient, i.e.,
\begin{align}
\left\|\nabla \mu_r(\bx) - \nabla \mu(\bx)\right\| \leq \kappa_{\text{\em mean}}, \qquad \forall \bx\in\calC.
\end{align}
\end{assumption}

\paragraph{Bounding $\left\|\nabla\bar{f}_r(\bx) - \nabla \mu_r(\bx)\right\|$.}
Now we bound the difference between the sample mean and the true mean under both sub-exponential and sub-Gaussian distributional assumptions on local gradients. For that we use standard tools, such as concentration results for sum of independent sub-Gaussian/sub-exponential random variables and $\eps$-net arguments. 
We prove in \Lemmaref{kappa-bound_sub-exp_all-x} and \Lemmaref{kappa-bound_sub-gauss_all-x}, respectively, in \Appendixref{statistical-model} that under both the assumptions, with high probability, our bounds are $\left\| \nabla \bar{f}_r(\bx) - \nabla \mu_r(\bx)\right\| \leq \calO\(\sqrt{\frac{d\log(n_rd)}{n_r}}\)$ for every $\bx\in\calC$. Note that under the sub-exponential assumption, the bound holds only for sufficiently large $n_r$ such that $n_r=\Omega\(d\log(n_rd)\)$, whereas, under the sub-Gaussian assumption, the bound holds for every $n_r$.

Substituting these bounds in \eqref{reduction-local-statistical} yields the following result, which, for notational convenience, we state for the case when all workers have the same number of data samples. 
Let $D=\max\{\|\bx-\bx'\|:\bx,\bx'\in\calC\}$ be the diameter of $\calC$. 
Note that $D=\Omega(\sqrt{d})$, and we assume that $D$ can grow at most polynomially in $d$.

\begin{theorem}[Gradient dissimilarity]\label{thm:concrete-kappa-bound_stat}
Suppose $n:= n_r, \forall r\in[R]$, and \Assumptionref{uniform-bound-mean-heterogeneous} holds. 
Then, the gradient dissimilarity bound under different distributional assumptions is as follows:
\begin{enumerate}
\item {[Sub-exponential]} Suppose \Assumptionref{sub-exp-grad} holds. 
Let $n\in\mathbb{N}$ be sufficiently large such that $n=\Omega\(d\log(nd)\)$. 
Then, with probability at least $1-\frac{R}{(1+nLD)^d}$, the following bound holds for all $r\in[R]$:
\begin{align}\label{kappa-bound_sub-exp-grad}
\left\|\nabla\bar{f}_r(\bx) - \nabla\bar{f}(\bx)\right\| \leq \kappa_{\text{\em mean}} + \calO\(\sqrt{\frac{d\log(nd)}{n}}\), \qquad \forall \bx\in\calC.
\end{align}
\item {[Sub-Gaussian]} Suppose \Assumptionref{sub-gauss-grad} holds. 
For every $n\in\mathbb{N}$, with probability at least $1-\frac{R}{(1+nLD)^d}$, the following bound holds for all $r\in[R]$:
\begin{align}\label{kappa-bound_sub-gauss-grad}
\left\|\nabla\bar{f}_r(\bx) - \nabla\bar{f}(\bx)\right\| \leq \kappa_{\text{\em mean}} + \calO\(\sqrt{\frac{d\log(nd)}{n}}\), \qquad \forall \bx\in\calC.
\end{align}
\end{enumerate}
\end{theorem}
\begin{remark}
Note that under \Assumptionref{sub-exp-grad} (sub-exponential), the gradient dissimilarity bound \eqref{kappa-bound_sub-exp-grad} holds only when each worker has sufficiently large number of samples $n=\Omega\(d\log(nd)\)$. On the other hand, under \Assumptionref{sub-gauss-grad} (sub-Gaussian), the gradient dissimilarity bound \eqref{kappa-bound_sub-gauss-grad} holds for every $n\in\mathbb{N}$.
\end{remark}
\subsection{Bounding the local variances}\label{subsec:variance-bound_stat}
The local variance bound at the $r$'th worker is $\bbE_{i\in_U[n_r]}\left\|\nabla f_r(\bz_{r,i},\bx) - \nabla \bar{f}_r(\bx)\right\|^2 \leq \sigma^2$ (from \eqref{local-variance_statistical}).
We simplify the LHS:
\begin{align}
\bbE_{i\in_U[n_r]}\left\|\nabla f_r(\bz_{r,i},\bx) - \nabla \bar{f}_r(\bx)\right\|^2 
&\leq 2\bbE_{i\in_U[n_r]}\left\|\nabla f_r(\bz_{r,i},\bx) - \nabla \mu_r(\bx)\right\|^2 \notag \\
&\hspace{2cm} + 2\bbE_{i\in_U[n_r]}\left\|\nabla \bar{f}_r(\bx) - \nabla \mu_r(\bx) \right\|^2 \nonumber \\
&\stackrel{\text{(a)}}{=} 2\left\|\nabla f_r(\bz_{r,1},\bx) - \nabla \mu_r(\bx)\right\|^2 + 2\left\|\nabla \bar{f}_r(\bx) - \nabla \mu_r(\bx) \right\|^2 \nonumber \\
&\stackrel{\text{(b)}}{\leq} 4\left\|\nabla f_r(\bz_{r,1},\bx) - \nabla \mu_r(\bx)\right\|^2 \label{simplified-variance-1}
\end{align}
For the first term on the RHS of (a), we used that $\bz_{r,i},i\in[n_r]$ are i.i.d., and the 
second term follows because it is independent of $i\in[n_r]$.
Inequality (b) follows because $\left\|\nabla \bar{f}_r(\bx) - \nabla \mu_r(\bx) \right\|^2 \leq \left\|\nabla f_r(\bz_{r,1},\bx) - \nabla \mu_r(\bx)\right\|^2$, since the average of i.i.d.\ samples gives tighter concentration in comparison to if we use just one sample.

Note that bounding $\left\|\nabla f_r(\bz_{r,1},\bx) - \nabla \mu_r(\bx)\right\|$ is equivalent to bounding $\left\|\nabla f_r(\bz,\bx) - \nabla \mu_r(\bx)\right\|$ for a random $\bz\sim q_r$. 
We provide a uniform bound on $\left\|\nabla f_r(\bz,\bx) - \nabla \mu_r(\bx)\right\|$ for a random $\bz\sim q_r$ in \Appendixref{variance-bound_stat} using the sub-Gaussian gradient assumption.
Below we state our final bound on the local variances. 
\begin{theorem}[Variance bound]\label{thm:variance-bound_stat}
Suppose $n:= n_r, \forall r\in[R]$, and \Assumptionref{sub-gauss-grad} holds. 
Then, with probability at least $1-\frac{R}{(1+nLD)^d}$, the following bound holds for all $r\in[R]$:
\begin{align}
\bbE_{i\in_U[n]}\left\|\nabla f_r(\bz_{r,i},\bx) - \nabla \bar{f}_r(\bx)\right\|^2 
&\leq \calO\(d\log(d)\), \qquad \forall \bx\in\calC.
\end{align}
\end{theorem}
\begin{remark}[Sub-Gaussian vs.\ sub-exponential assumption]
Note that, we needed sub-Gaussian assumption on local gradients because we wanted to uniformly bound $\bbE_{i\in[n_r]}\left\|\nabla f_r(\bz_{r,i},\bx) - \nabla \mu_r(\bx)\right\|^2$, which is the case when we use {\em only one} data sample in each SGD iteration. 
In this paper, we use {\em mini-batch} SGD with a variable batch size (to control the approximation error of our solution; see the approximation error analysis in \Subsectionref{remarks_main-result}).
So, when the batch-size $b$ is sufficiently large and satisfies $b=\Omega(d\log(bd))$, we can work with the sub-exponential gradient assumption because the large batch size gives a concentration similar to sub-Gaussian. This would give a bound of $\calO\(\frac{d\log(bd)}{b}\)$ on variance.
\end{remark}

\section{Future Work}\label{sec:conclusion}
We leave a few open problems for future research for Byzantine-resilient SGD on heterogeneous data: Improving the Byzantine tolerance to beyond $\frac{1}{4}$ fraction; obtaining second-order convergence guarantees for non-convex objectives for SGD (note that \cite{Bartlett-Byz_nonconvex19} obtained such guarantees for full-batch GD on i.i.d.\ homogeneous data); improving upon the decoding complexity at master node (note that the outlier-filtering procedure in this paper requires SVD computations, which could be expensive); studying local SGD with Byzantine adversaries to improve communication efficiency; and study gradient compression with arbitrary compressors (beyond the $\randk$ sparsifier).

\bibliography{reference}
\bibliographystyle{alpha}

\appendix

\section{Proofs of \Lemmaref{subset_variance} and the First Part of \Theoremref{gradient-estimator}}\label{app:subset_bounded-variance_proof}
\begin{lemma*}[Restating \Lemmaref{subset_variance}]
Suppose there are $m$ independent distributions $p_1,p_2,\hdots,p_m$ in $\R^d$ such that $\bbE_{\by\sim p_i}[\by]=\vct{\mu}_i, i\in[m]$ and each $p_i$ has bounded variance in all directions, i.e., $\bbE_{\by\sim p_i}[\langle \by - \vct{\mu}_i, \bv\rangle^2] \leq \sigma_{p_i}^2$ holds for all unit vectors $\bv\in\R^d$. Take an arbitrary $\eps'\in(0,1]$. Then, given $m$ independent samples $\by_1,\by_2,\hdots,\by_m$, where $\by_i\sim p_i$, with probability $1- \exp(-\eps'^2m/16)$, there is a subset $\calS$ of $(1-\eps')m$ points such that
\begin{align*}
\lambda_{\max}\(\frac{1}{|\calS|}\sum_{i\in\calS} \(\by_i-\vct{\mu}_i\)\(\by_i-\vct{\mu}_i\)^T \) \leq \frac{4\sigma_{p_{\max}}^2}{\eps'}\left(1 + \frac{d}{(1-\eps')m}\right), \qquad \text{ where } \sigma_{p_{\max}}^2=\max_{i\in[m]}\sigma_{p_{i}}^2.
\end{align*}\end{lemma*}
As mentioned in \Subsectionref{bounded-variance_subset}, \Lemmaref{subset_variance} generalizes \cite[Proposition B.1]{Untrusted-data_Charikar17}, where the $m$ samples $\by_1,\hdots,\by_m$ are drawn independently from a {\em single} distribution $p$ with mean $\bmu$ and variance bound of $\sigma_p^2$,
whereas, in our setting, different $\by_i$'s may come from different distributions, which may have different means and variances. 
\Lemmaref{subset_variance} can be proved using similar arguments given in the proof of \cite[Proposition B.1]{Untrusted-data_Charikar17}, and we provide a complete proof of this in this section.

Proof of \Lemmaref{subset_variance} relies on the following claim.
\begin{lemma}\label{lem:subset_variance_interim-lemma}
Let $p$ be a distribution on $\R^d$ such that $\bbE_{\by\sim p}[\by]=\bmu$ and $\bbE_{\by\sim p}[\langle \by - \bmu, \bv\rangle^2] \leq \sigma^2$ for all unit vectors $\bv\in\R^d$.
Let $\bM$ be a symmetric matrix such that $\bzero\prec\bM\prec c\bI$ for some constant $c>0$ and $\emph{tr}\big((c\bI-\bM)^{-1}\big) \leq \frac{1}{4\sigma_{\emph{prev}}^2}$, where $\sigma_{\emph{prev}}\geq\sigma$. Take an arbitrary $\eps'\in(0,1]$. Then, for $\by\sim p$, with probability at least $1-\frac{\eps'}{2}$, we have $\big(\bM+\eps'(\by-\bmu)(\by-\bmu)^T\big) \prec (c+4\sigma^2)\bI$ and $\emph{tr}\Big(\big((c+4\sigma^2)\bI-(\bM+\eps'(\by-\bmu)(\by-\bmu)^T)\big)^{-1}\Big) \leq \frac{1}{4\sigma_{\emph{prev}}^2}$.
\end{lemma}
Before proving \Lemmaref{subset_variance_interim-lemma}, first we show how we can use it to prove \Lemmaref{subset_variance}.
\begin{proof}[Proof of \Lemmaref{subset_variance}]
Initialize a matrix $\bM:=\bzero$, a set $\calS:=\emptyset$, and $c:=4\sigma_{p_{\max}}^2d$. Note that the preconditioning of \Lemmaref{subset_variance_interim-lemma} (i.e., $\bzero\prec\bM\prec c\bI$ and $\text{tr}\big((c\bI-\bM)^{-1}\big) \leq \frac{1}{4\sigma_{\text{prev}}^2}$) is satisfied with $\sigma_{\text{prev}}=\sigma_{p_{\max}}$.
Go through the stream of $m$ samples from $\by_1$ to $\by_m$.
Note that $\sigma_{p_{\max}} \geq \sigma_{p_i}$ holds for all $i\in[m]$.
For notational convenience, let $\bty_i=\by_i-\bmu_i$ for $i=1,2,\hdots,m$.
If $(\bM+\eps'\bty_i\bty_i^T)$ satisfies the conclusion of \Lemmaref{subset_variance_interim-lemma}, i.e., $\big(\bM+\eps'\bty_i\bty_i^T\big) \prec (c+4\sigma_{p_i}^2)\bI$ and $\text{tr}\Big(\big((c+4\sigma_{p_i}^2)\bI-(\bM+\eps'\bty_i\bty_i^T)\big)^{-1}\Big) \leq \frac{1}{4\sigma_{p_{\max}}^2}$ (which we know holds with probability at least $1-\frac{\eps'}{2}$), then update $\calS\leftarrow\calS\cup\{i\}$, $\bM\leftarrow\bM+\eps'\bty_i\bty_i^T$, and $c\leftarrow c+4\sigma_{p_i}^2$.\footnote{Note that we only observe $\by_i$'s, not $(\by_i-\bmu_i)$. In the context of distributed SGD, the $\by_i$'s correspond to the stochastic gradients that the master receives from workers, and there, the master does not know the true local gradients at any worker -- the true local gradient at worker $i$ corresponds to the mean $\bmu_i$ here. Yet, in each iteration $i$, we probabilistically add $\eps'(\by_i-\bmu_i)(\by_i-\bmu_i)^T$ to $\bM$. We can do that, because we just want to show an existence of a set $\calS$ that satisfies the required properties stated in \Lemmaref{subset_variance}. This is just for the purpose of analysis, and we are not giving an algorithm to construct $\calS$.}

Note that, in the next iteration, when we consider the sample $\by_{i+1}$, the preconditioning of \Lemmaref{subset_variance_interim-lemma} is automatically satisfied: 
If the conclusion in the $i$'th step did not hold and we did not update $\calS,\bM,c$, then the preconditioning of \Lemmaref{subset_variance_interim-lemma} in the $(i+1)$'st iteration trivially holds, as it used to hold in the $i$'th iteration.
If the conclusion in the $i$'th step held and we updated $\calS,\bM,c$, then the preconditioning of \Lemmaref{subset_variance_interim-lemma} in the $(i+1)$'st iteration holds, as it is the same condition that we checked in the conclusion of the $i$'th iteration for updating $\calS,\bM,c$.

When we have gone through the stream of $m$ samples, in the end, we have $c=4\sigma_{p_{\max}}^2d + \sum_{i\in\calS}4\sigma_{p_i}^2 \leq 4\sigma_{p_{\max}}^2(d+|\calS|)$ and $\bM\prec\left(4\sigma_{p_{\max}}^2(d+|\calS|)\right)\bI$, which implies that $\lambda_{\max}(\bM)\leq 4\sigma_{p_{\max}}^2(d+|\calS|)$.
Since $\bM=\sum_{i\in\calS}\eps'\bty_i\bty_i^T$, we have $\lambda_{\max}\left(\frac{1}{|\calS|}\sum_{i\in\calS}\bty_i\bty_i^T\right) = \frac{1}{\eps'|\calS|}\lambda_{\max}\left(\bM\right)\leq \frac{4\sigma_{p_{\max}}^2}{\eps'}(1+\frac{d}{|\calS|})$. 
It only remains to show that $|\calS|\geq(1-\eps')m$ holds with high probability.

By the above discussion, note that for each element $i$, we add $i$ to $\calS$ with probability at least $1-\frac{\eps'}{2}$. Since the $m$ samples $\by_i, i\in[m]$ are independent of each other, we have that the distribution of $|\calS|$ is lower-bounded by the sum of $m$ independent indicator random variables, where each of them is equal to 1 with probability $1-\frac{\eps'}{2}$. So, by Chernoff bound, we have $\Pr[|\calS|\leq(1-\eps')m] \leq \exp(-\frac{\eps'^2m}{16})$, which implies that $\Pr[|\calS|\geq(1-\eps')m] \geq 1-\exp(-\frac{\eps'^2m}{16})$.

We have shown that with probability $1-exp(-\frac{\eps'^2m}{16})$, there exists a subset $\calS$ of $\by_1,\hdots,\by_m$ such that $|\calS|\geq(1-\eps')m$ and $\lambda_{\max}\left(\frac{1}{|\calS|}\sum_{i\in\calS}\bty_i\bty_i^T\right) \leq\frac{4\sigma_{p_{\max}}^2}{\eps'}\left(1+\frac{d}{(1-\eps')m}\right)$. 
Substituting $\bty_i=\by_i-\bmu_i$ for $i=1,2,\hdots,m$ concludes the proof of \Lemmaref{subset_variance}.
\end{proof}
Now we proceed to proving \Lemmaref{subset_variance_interim-lemma}.
\begin{proof}[Proof of \Lemmaref{subset_variance_interim-lemma}]
A version of this lemma has appeared in \cite[Lemma B.2]{Untrusted-data_Charikar17}, which, in turn, is essentially the same as \cite[Lemma 3.3]{Twice-Ramanujan_BatsonSpSr12}. Our proof is along the lines of the proof of \cite[Lemma B.2]{Untrusted-data_Charikar17}.

For simplicity of notation, let $\bty=\by-\bmu$. Instead of $\big(\bM-\eps'\bty\bty^T\big)$, it will be helpful later to consider $\big(\bM-t\bty\bty^T\big)$ for arbitrary $t\in[0,\eps']$.

By the Sherman-Morrison matrix inversion formula, we have that if a square matrix $\bA\in\R^{n\times n}$ is invertible 
and $\bu,\bv\in\R^n$ are column vectors such that $(1+\bv^T\bA^{-1}\bu)\neq 0$, then $(\bA+\bu\bv^T)$ is invertible and its inverse is equal to $(\bA+\bu\bv^T)^{-1} = \bA^{-1} - \frac{\bA^{-1}\bu\bv^T\bA^{-1}}{1+\bv^T\bA^{-1}\bu}$.

We want to apply this formula on $\big(((c+4\sigma^2)\bI - \bM) - t\bty\bty^T\big)^{-1}$ with $\bA=((c+4\sigma^2)\bI - \bM)$, $\bu=\sqrt{t}\bty$, and $\bv=-\sqrt{t}\bty$.
For that, we need to show two things: first, that $((c+4\sigma^2)\bI - \bM)$ is invertible, and second, that $(1-t\bty^T((c+4\sigma^2)\bI-\bM)^{-1}\bty)\neq 0$. 
For the first requirement, note that $\bM \prec (c+4\sigma^2)\bI$, which follows because $\bM\prec c\bI$ (by assumption), and $\sigma>0$. This implies that $((c+4\sigma^2)\bI - \bM)$ is invertible. 
It follows from the analysis below (see the paragraph after \eqref{subset_variance_interim-lemma4}) that the second requirement also holds for every $t\in[0,\eps']$ with probability at least $1-\frac{\eps'}{2}$. Now, applying the Sherman-Morrison matrix inversion formula on $\big(((c+4\sigma^2)\bI - \bM) - t\bty\bty^T\big)^{-1}$:

\begin{align}\label{subset_variance_interim-lemma1}
\big(((c+4\sigma^2)\bI - \bM) - t\bty\bty^T\big)^{-1} &= \big((c+4\sigma^2)\bI - \bM\big)^{-1} \notag \\
&\quad + t\frac{\big((c+4\sigma^2)\bI - \bM\big)^{-1}\bty\bty^T\big((c+4\sigma^2)\bI - \bM\big)^{-1}}{1-t\bty^T\big((c+4\sigma^2)\bI-\bM\big)^{-1}\bty}
\end{align}
Taking trace on both sides gives
\begin{align*}
\text{tr}\left(\left(((c+4\sigma^2)\bI - \bM) - t\bty\bty^T\right)^{-1}\right) &= \text{tr}\left(\left((c+4\sigma^2)\bI - \bM\right)^{-1}\right) \\
&\quad+ \frac{\text{tr}\left(\big((c+4\sigma^2)\bI - \bM\big)^{-1}\bty\bty^T\big((c+4\sigma^2)\bI - \bM\big)^{-1}\right)}{\frac{1}{t}-\bty^T\big((c+4\sigma^2)\bI-\bM\big)^{-1}\bty}.
\end{align*}
Let $\Phi_c(\bM)=\text{tr}\big((c\bI-\bM)^{-1}\big)$. Using $\text{tr}(\bA\bB)=\text{tr}(\bB\bA)$ on the last term and using the fact that trace of a scalar is the scalar itself, we get
\begin{align}\label{subset_variance_interim-lemma3}
\Phi_{c+4\sigma^2}(\bM + t\bty\bty^T) &= \Phi_{c+4\sigma^2}(\bM) + \frac{\bty^T\big((c+4\sigma^2)\bI - \bM\big)^{-2}\bty}{\frac{1}{t}-\bty^T\big((c+4\sigma^2)\bI-\bM\big)^{-1}\bty}.
\end{align}
We are given $\Phi_c(\bM)\leq\frac{1}{4\sigma_{\text{prev}}^2}$, and we want to show $\Phi_{c+4\sigma^2}(\bM + t\bty\bty^T)\leq\frac{1}{4\sigma_{\text{prev}}^2}$. So, it suffices to prove that $\Phi_{c+4\sigma^2}(\bM + t\bty\bty^T)\leq\Phi_c(\bM)$. This, in light of \eqref{subset_variance_interim-lemma3}, is equivalent to the condition
\begin{align}\label{subset_variance_interim-lemma4}
\frac{1}{t}\ \geq\ \bty^T\big((c+4\sigma^2)\bI-\bM\big)^{-1}\bty + \frac{\bty^T\big((c+4\sigma^2)\bI - \bM\big)^{-2}\bty}{\Phi_{c}(\bM) - \Phi_{c+4\sigma^2}(\bM)} =: \Psi,
\end{align}
which, as we show in the analysis below, will hold with probability at least $1-\frac{\eps'}{2}$ for all $t\in[0,\eps']$. 
(Assume that \eqref{subset_variance_interim-lemma4} holds with probability at least $1-\frac{\eps'}{2}$ for all $t\in[0,\eps']$.
Note that $\Phi_{c}(\bM) > \Phi_{c+4\sigma^2}(\bM)$ (from \Claimref{subset_variance_interim-lemma-claim2} below) and $((c+4\sigma^2)\bI - \bM)\succ\bzero$ hold. Using these in \eqref{subset_variance_interim-lemma4} imply that $\frac{1}{t} > \bty^T\big((c+4\sigma^2)\bI-\bM\big)^{-1}\bty$ holds with probability at least $1-\frac{\eps'}{2}$ for all $t\in[0,\eps']$. Thus the second requirement $(1-t\bty^T((c+4\sigma^2)\bI-\bM)^{-1}\bty)\neq 0$ also holds, which was necessary for applying the matrix inversion formula on $\big(((c+4\sigma^2)\bI - \bM) - t\bty\bty^T\big)^{-1}$ to write \eqref{subset_variance_interim-lemma1}.)

Since $\Psi$ is a scalar, we have $\text{tr}(\Psi)=\Psi$. Taking trace in \eqref{subset_variance_interim-lemma4}, and using $\text{tr}(\bty^T\bA\bty)=\text{tr}(\bA\bty\bty^T)$, and then taking expectation, we get 
\begin{align}\label{subset_variance_interim-lemma5}
\bbE\left[\Psi\right] &= \bbE\left[\text{tr}\left(\big((c+4\sigma^2)\bI-\bM\big)^{-1}\bty\bty^T\right)\right] + \frac{\bbE\left[\text{tr}\left(\big((c+4\sigma^2)\bI - \bM\big)^{-2}\bty\bty^T\right)\right]}{\Phi_{c}(\bM) - \Phi_{c+4\sigma^2}(\bM)}.
\end{align}
Since $\big((c+4\sigma^2)\bI-\bM\big)\succ\bzero$, we also have that $\big((c+4\sigma^2)\bI-\bM\big)^{-i}\succ\bzero$, for $i=1,2$.
Let $\bA=\big((c+4\sigma^2)\bI-\bM\big)^{-i}$ for any $i\in\{1,2\}$.
Now we argue that $\bbE\left[\text{tr}\left(\bA\bty\bty^T\right)\right] \leq \sigma^2\text{tr}\left(\bA\right)$, where $\sigma^2$ is such that $\bbE_{\by\sim p}[\langle \bty, \bv\rangle^2] \leq \sigma^2$ for all unit vectors $\bv\in\R^d$. Note that the last condition is equivalent to $\sup_{\bv\in\R^d:\|\bv\|=1}\bv^T\(\bbE_{\by\sim p}[\bty\bty^T]\)\bv \leq \sigma^2$, which, in view \eqref{alternate-defn_max-eigenvalue}, is equivalent to saying that $\lambda_{\max}\(\bbE_{\by\sim p}[\bty\bty^T]\)\leq\sigma^2$.
\begin{claim}\label{claim:subset_variance_interim-lemma-claim}
$\bbE\left[\emph{tr}\left(\bA\bty\bty^T\right)\right] \leq \sigma^2\emph{tr}\left(\bA\right)$.
\end{claim}
\begin{proof}
The claim follows from the following set of inequalities.
\begin{align*}
\bbE\left[\text{tr}\left(\bA\bty\bty^T\right)\right] &
\stackrel{\text{(a)}}{=} \bbE\left[\sum_{i,j}\bA_{ij}(\bty\bty^T)_{ji}\right] = \sum_{i,j}\bA_{ij}\big(\bbE[\bty\bty^T]\big)_{ji}  \notag \\
&\stackrel{\text{(b)}}{=} \text{tr}\big(\bA\bbE[\bty\bty^T]\big) \stackrel{\text{(c)}}{\leq} \left\|\bbE\left[\bty\bty^T\right]\right\| \left\|\bA\right\|_* 
\stackrel{\text{(d)}}{\leq} \sigma^2\text{tr}(\bA)
\end{align*}
In (a) and (b), we used the definition of trace: $\text{tr}(\bA\bB)=\sum_i(\bA\bB)_{ii}=\sum_{i,j}\bA_{ij}\bB_{ji}$. 
In (c), we used $\text{tr}(\bA\bB) = \text{tr}(\bB\bA) \leq \|\bB\|\|\bA\|_*$ (see \Claimref{holder} in \Appendixref{poly-time-lemma_proof}), where $\|\cdot\|_*$ denotes the nuclear norm, which is equal to the sum of singular values.
In (d), we used two things, first, since $\bA\succeq\bzero$, we have $\|\bA\|_*=\text{tr}(\bA)$, and second, that $\left\|\bbE\left[\bty\bty^T\right]\right\| \leq \sigma^2$, which follows because $\left\|\bbE\left[\bty\bty^T\right]\right\|=\lambda_{\max}\(\bbE_{\by\sim p}[\bty\bty^T]\)\leq\sigma^2$.
\end{proof}
Using \Claimref{subset_variance_interim-lemma-claim} in \eqref{subset_variance_interim-lemma5} gives
\begin{align}\label{subset_variance_interim-lemma6}
\bbE\left[\Psi\right] &\leq \sigma^2\left(\text{tr}\left(\big((c+4\sigma^2)\bI-\bM\big)^{-1}\right) + \frac{\text{tr}\left(\big((c+4\sigma^2)\bI - \bM\big)^{-2}\right)}{\Phi_{c}(\bM) - \Phi_{c+4\sigma^2}(\bM)}\right).
\end{align} 
\begin{claim}\label{claim:subset_variance_interim-lemma-claim2}
$\Phi_{c}(\bM) - \Phi_{c+4\sigma^2}(\bM) \geq 4\sigma^2\emph{tr}\left(\big((c+4\sigma^2)\bI - \bM\big)^{-2}\right)$.
\end{claim}
\begin{proof}
Since $(c\bI-\bM)\succ\bzero$, let its eigen-decomposition be $(c\bI-\bM) = \sum_{i}\lambda_i\bu_i\bu_i^T$, where $\lambda_i$'s are the eigenvalues of $(c\bI-\bM)$ and $\bu_i$'s are the corresponding eigenvectors. It follows that $\left((c\bI-\bM) + 4\sigma^2\bI\right) = \sum_{i}(\lambda_i+4\sigma^2)\bu_i\bu_i^T$. These imply that $(c\bI-\bM)^{-1} = \sum_{i}\frac{1}{\lambda_i}\bu_i\bu_i^T$ and $\left((c\bI-\bM) + 4\sigma^2\bI\right)^{-1} = \sum_{i}\frac{1}{(\lambda_i+4\sigma^2)}\bu_i\bu_i^T$.

Substituting the definition of $\Phi_c(\bM)=\text{tr}\big((c\bI-\bM)^{-1}\big)$, we have
\begin{align*}
\Phi_{c}(\bM) - \Phi_{c+4\sigma^2}(\bM) &= \text{tr}\left((c\bI-\bM)^{-1} - \left(((c+4\sigma^2)\bI-\bM)\right)^{-1}\right) \\
&= \text{tr}\left(\sum_{i}\frac{1}{\lambda_i}\bu_i\bu_i^T - \sum_{i}\frac{1}{(\lambda_i+4\sigma^2)}\bu_i\bu_i^T\right) \\
&= 4\sigma^2\text{tr}\left(\sum_{i}\frac{1}{\lambda_i(\lambda_j+4\sigma^2)}\bu_i\bu_i^T \right) \\
&\stackrel{\text{(g)}}{=} 4\sigma^2\sum_{i}\frac{1}{\lambda_i(\lambda_j+4\sigma^2)} \\
&\stackrel{\text{(h)}}{\geq} 4\sigma^2\sum_{i}\frac{1}{(\lambda_j+4\sigma^2)^2} \\
&= 4\sigma^2\text{tr}\left(\sum_{i}\frac{1}{(\lambda_j+4\sigma^2)^2}\bu_i\bu_i^T \right) \\
&= 4\sigma^2\text{tr}\left(\left((c\bI-\bM) + 4\sigma^2\bI\right)^{-2}\right)
\end{align*}
Here (g) follows from the fact that trace of a square matrix is equal to the sum of its eigenvalues and (h) follows because $\frac{1}{\lambda} \geq \frac{1}{\lambda+4\sigma^2}$.
\end{proof}
Substituting $\Phi_{c+4\sigma^2}(\bM)=\text{tr}\left(\left(((c+4\sigma^2)\bI-\bM)\right)^{-1}\right)$ for the first term in \eqref{subset_variance_interim-lemma6} and the bound from \Claimref{subset_variance_interim-lemma-claim2} for the second term gives
$\bbE\left[\Psi\right] \leq \sigma^2\left(\Phi_{c+4\sigma^2}(\bM) + \frac{1}{4\sigma^2} \right)$. Note that \Claimref{subset_variance_interim-lemma-claim2} trivially implies $\Phi_{c+4\sigma^2}(\bM)\leq\Phi_c(\bM)$, where $\Phi_c(\bM)=\text{tr}\big((c\bI-\bM)^{-1}\big)\leq\frac{1}{4\sigma_{\text{prev}}^2}$ (which follows from the hypothesis of \Lemmaref{subset_variance_interim-lemma}). So, we have
\begin{align}\label{subset_variance_interim-lemma7}
\bbE\left[\Psi\right] \leq \sigma^2\left(\frac{1}{4\sigma_{\text{prev}}^2} + \frac{1}{4\sigma^2} \right) \stackrel{\text{(h)}}{\leq} \sigma^2\left(\frac{1}{4\sigma^2} + \frac{1}{4\sigma^2}\right) \leq \frac{1}{2},
\end{align} 
where (h) follows from our assumption that $\sigma_{\text{prev}}\geq\sigma$.

Note that $\Psi$ is a non-negative random variable (see \eqref{subset_variance_interim-lemma4}). So, by the Markov's inequality, we have $\Pr[\Psi\geq\frac{1}{\eps'}]\leq \frac{\bbE[\Psi]}{\nicefrac{1}{\eps'}} \leq \frac{\eps'}{2}$, which implies that $\Pr[\Psi\leq\frac{1}{\eps'}]\geq 1-\frac{\eps'}{2}$. Substituting the value of $\Psi$ in \eqref{subset_variance_interim-lemma4} implies that \eqref{subset_variance_interim-lemma4} holds with probability at least $1-\frac{\eps'}{2}$ for all $t\in[0,\eps']$. Note that the condition in \eqref{subset_variance_interim-lemma4} is equivalent to the condition that $\Phi_{c+4\sigma^2}(\bM + t\bty\bty^T)\leq\Phi_c(\bM)$, where $\Phi_c(\bM)\leq\frac{1}{4\sigma_{\text{prev}}^2}$. 
Thus, with probability at least $1-\frac{\eps'}{2}$, we have that $\text{tr}\left(\left(((c+4\sigma^2)\bI-(\bM+t\bty\bty^T))\right)^{-1}\right)=\Phi_{c+4\sigma^2}(\bM + t\bty\bty^T) \leq \frac{1}{4\sigma_{\text{prev}}^2}$, for every $t\in[0,\eps']$.

It only remains to show that $\big(\bM+\eps'\bty\bty^T\big) \prec (c+4\sigma^2)\bI$, which is equivalent to the condition that $\lambda_{\max}\big(\bM+\eps'\bty\bty^T\big)<(c+4\sigma^2)$. Suppose not, i.e., $\lambda_{\max}\big(\bM+\eps'\bty\bty^T\big)\geq(c+4\sigma^2)$. Note that we have $\lambda_{\max}(\bM) < c$ (by the hypothesis of \Lemmaref{subset_variance_interim-lemma}). Since $\lambda_{\max}\big(\bM+t\bty\bty^T\big)$ is a continuous function of $t$ and $\lambda_{\max}(\bM) < c$, $\lambda_{\max}\big(\bM+\eps'\bty\bty^T\big)\geq(c+4\sigma^2)$, we have from the intermediate value theorem that there exists a $t'\in[0,\eps']$ such that $\lambda_{\max}\big(\bM+t'\bty\bty^T\big)=(c+4\sigma^2)$.
This implies that the matrix $\left(((c+4\sigma^2)\bI-(\bM+t'\bty\bty^T))\right)^{-1}$ is not invertible (as $\big((c+4\sigma^2)\bI-(\bM+t'\bty\bty^T)\big)$ has a zero eigenvalue), implying that $\text{tr}\left(\left(((c+4\sigma^2)\bI-(\bM+t\bty\bty^T))\right)^{-1}\right)$ is unbounded. But, we have already shown that $\text{tr}\left(\left(((c+4\sigma^2)\bI-(\bM+t\bty\bty^T))\right)^{-1}\right)\leq\frac{1}{4\sigma_{\text{prev}}^2}<\infty$, for all $t\in[0,\eps']$. A contradiction!

This completes the proof of \Lemmaref{subset_variance_interim-lemma}.
\end{proof}

\subsection{Remaining proof of the first part of \Theoremref{gradient-estimator} from \Subsectionref{bounded-variance_subset}}\label{app:remaining_part1-robust-grad}
We have from \eqref{bounded_subset_variance} that
\begin{align}\label{bounded_subset_variance_app1}
\lambda_{\max}\(\frac{1}{|\calS|}\sum_{i\in\calS} \(\by_i-\bmu_i\)\(\by_i-\bmu_i\)^T \) \leq \frac{4\sigma^2}{b\eps'}\left(1 + \frac{d}{(1-(\eps+\eps'))R}\right),
\end{align}
where, for simplicity of notation, we replaced $\bg_i(\bx), \nabla F_i(\bx)$ by $\by_i,\bmu_i$, respectively.
From the alternate definition of the largest eigenvalue of symmetric matrices, we have
\begin{align}\label{alternate-defn_max-eigenvalue}
\text{For a symmetric matrix $\bA\in\R^{d\times d}$, we have, } \lambda_{\max}(\bA)=\sup_{\bv\in\R^d,\|\bv\|=1}\bv^T\bA\bv.
\end{align}
Applying this with $\bA=\frac{1}{|\calS|}\sum_{i\in\calS} \(\by_i-\bmu_i\)\(\by_i-\bmu_i\)^T$ and using
$\bv^T\(\frac{1}{|\calS|}\sum_{i\in\calS} \(\by_i-\bmu_i\)\(\by_i-\bmu_i\)^T\)\bv = \frac{1}{|\calS|}\sum_{i\in\calS} \langle\by_i-\bmu_i, \bv\rangle^2$, we can equivalently write \eqref{bounded_subset_variance_app1} as
\begin{align}\label{bounded_subset_variance_app}
\sup_{\bv\in\R^d:\|\bv\|=1}\(\frac{1}{|\calS|}\sum_{i\in\calS} \langle\by_i-\bmu_i, \bv \rangle^2\) \leq \widehat{\sigma}_0^2,
\end{align}
where $\widehat{\sigma}_0^2 = \frac{4\sigma^2}{b\eps'}\left(1 + \frac{d}{(1-(\eps+\eps'))R}\right)$.

Note that \eqref{bounded_subset_variance_app} is bounding the deviation of the points in $\calS$ from their respective means $\bmu_i,i\in\calS$. However, in the first part of \Theoremref{gradient-estimator}, we need to bound the deviation of the points in $\calS$ from their sample mean $\frac{1}{|\calS|}\sum_{i\in\calS}\by_i$. For that, define $\by_{\calS}:=\frac{1}{|\calS|}\sum_{i\in\calS} \by_i$ to be the sample mean of the points in $\calS$. Take an arbitrary $\bv\in\R^d$ such that $\|\bv\|=1$.
\begin{align}
\frac{1}{|\calS|}\sum_{i\in\calS} \langle \by_i-\by_{\calS}, \bv \rangle^2 &= \frac{1}{|\calS|}\sum_{i\in\calS} \left[\langle \by_i-\bmu_i, \bv \rangle + \langle \bmu_i - \by_{\calS}, \bv \rangle\right]^2 \notag \\
&\leq \frac{2}{|\calS|}\sum_{i\in\calS} \langle \by_i-\bmu_i, \bv \rangle^2 + \frac{2}{|\calS|}\sum_{i\in\calS} \langle \bmu_i - \by_{\calS}, \bv \rangle^2 \tag{using $(a+b)^2 \leq 2a^2 + 2b^2$} \\
\intertext{Using \eqref{bounded_subset_variance_app} to bound the first term, we get}
&\leq 2\widehat{\sigma}_0^2 +  \frac{2}{|\calS|}\sum_{i\in\calS} \left\langle \bmu_i - \frac{1}{|\calS|}\sum_{j\in\calS}\by_j, \bv \right\rangle^2 \notag \\
&= 2\widehat{\sigma}_0^2 +  \frac{2}{|\calS|}\sum_{i\in\calS} \left[\frac{1}{|\calS|}\sum_{j\in\calS}\langle \by_j - \bmu_i, \bv \rangle\right]^2 \notag \\
&\leq 2\widehat{\sigma}_0^2 +  \frac{2}{|\calS|}\sum_{i\in\calS} \frac{1}{|\calS|}\sum_{j\in\calS}\langle \by_j - \bmu_i, \bv \rangle^2 \tag{using the Jensen's inequality} \\
&= 2\widehat{\sigma}_0^2 +  \frac{2}{|\calS|}\sum_{i\in\calS} \frac{1}{|\calS|}\sum_{j\in\calS}\left[\langle \by_j -\bmu_j, \bv \rangle + \langle \bmu_j - \bmu_i, \bv \rangle\right]^2 \notag \\
&\leq 2\widehat{\sigma}_0^2 +  \frac{2}{|\calS|}\sum_{i\in\calS} \frac{2}{|\calS|}\sum_{j\in\calS}\langle \by_j -\bmu_j, \bv \rangle^2 + \frac{2}{|\calS|}\sum_{i\in\calS} \frac{2}{|\calS|}\sum_{j\in\calS}\langle \bmu_j - \bmu_i, \bv \rangle^2 \tag{using $(a+b)^2 \leq 2a^2 + 2b^2$} \\
&\leq 2\widehat{\sigma}_0^2 +  \frac{4}{|\calS|}\sum_{i\in\calS} \frac{1}{|\calS|}\sum_{j\in\calS}\langle \by_j -\bmu_j, \bv \rangle^2 + \frac{4}{|\calS|}\sum_{i\in\calS} \frac{1}{|\calS|}\sum_{j\in\calS}\|\bmu_j - \bmu_i\|^2 \tag{using the Cauchy-Schwarz inequality and that $\|\bv\|\leq1$} \\
&= 2\widehat{\sigma}_0^2 +  \frac{4}{|\calS|}\sum_{j\in\calS}\langle \by_j -\bmu_j, \bv \rangle^2 + \frac{4}{|\calS|}\sum_{i\in\calS} \frac{1}{|\calS|}\sum_{j\in\calS}\|\bmu_j -\bmu + \bmu - \bmu_i\|^2 \notag \\
\intertext{where $\bmu=\frac{1}{R}\sum_{i=1}^R\bmu_i$. Using \eqref{bounded_subset_variance_app} for the 2nd term and $\|\bu+\bv\|^2 \leq 2\|\bu\|^2 + 2\|\bv\|^2$ for the third term, we get}
&\leq 6\widehat{\sigma}_0^2 + \frac{4}{|\calS|}\sum_{i\in\calS} \frac{1}{|\calS|}\sum_{j\in\calS} 2\|\bmu_j - \bmu\|^2 + 2\|\bmu_i - \bmu\|^2  \notag \\
&= 6\widehat{\sigma}_0^2 + \frac{8}{|\calS|}\sum_{i\in\calS} \frac{1}{|\calS|}\sum_{j\in\calS} \|\bmu_j - \bmu\|^2 + \notag \frac{8}{|\calS|}\sum_{i\in\calS} \frac{1}{|\calS|}\sum_{j\in\calS} \|\bmu_i - \bmu\|^2 \\
&= 6\widehat{\sigma}_0^2 + \frac{16}{|\calS|}\sum_{i\in\calS}\|\bmu_i - \bmu\|^2 \notag \\
&\leq 6\widehat{\sigma}_0^2 + 16\kappa^2. \label{remaining_part1-robust-grad1}
\end{align}
For the last inequality, first we substituted $\bmu_i=\nabla F_i(\bx)$, $\bmu=\nabla F(\bx)$, and then used \eqref{bounded_local-global} to bound $\|\bmu_i - \bmu\|^2\leq\kappa^2$ for all $i\in\calS$.

Note that \eqref{remaining_part1-robust-grad1} holds for every unit vector $\bv\in\R^d$.
By substituting $\by_i=\bg_i(\bx),\by_{\calS}=\bg_{\calS}(\bx)$ and the value of $\widehat{\sigma}_0^2$, we get
\begin{align}
\sup_{\bv\in\R^d:\|\bv\|=1}\(\frac{1}{|\calS|}\sum_{i\in\calS} \langle \bg_i(\bx) - \bg_{\calS}(\bx), \bv \rangle^2\) &\leq \frac{24\sigma^2}{b\eps'}\left(1 + \frac{d}{(1-(\eps+\eps'))R}\right) + 16\kappa^2.
\end{align}
In view of \eqref{alternate-defn_max-eigenvalue}, the above is equivalent to the matrix concentration bound \eqref{mat-concen_gradient-estimator} in the first part of \Theoremref{gradient-estimator}.

\section{Convergence Proof of \Theoremref{SGD_convergence}}\label{sec:convergence_proofs}
We prove convergence results for the strongly-convex part of \Theoremref{SGD_convergence} in \Appendixref{convex_convergence} and for the non-convex part in \Appendixref{nonconvex_convergence}.
\subsection{Proof of \Theoremref{SGD_convergence} (strongly-convex)}\label{app:convex_convergence}
Recall the update rule of our algorithm: 
$\bhx^{t+1}=\bx^t - \eta\btg(\bx^t);\  \bx^{t+1} = \Pi_{\calC}\left(\bhx^{t+1}\right),\ t=1,2,3,\hdots$.
Since $\bx^{t+1}$ is the projection of $\bhx^{t+1}$ onto the convex set $\calC$, and by assumption $\bx^*\in\calC$, we have $\|\bx^{t+1}-\bx^*\|\leq\|\bhx^{t+1}-\bx^*\|$. Now we proceed with the proof.
\begin{align}
\|\bx^{t+1} - \bx^* \|^2 &\leq \|\bhx^{t+1}-\bx^*\|^2 \nonumber \\
&= \|\bx^t - \bx^* - \eta\nabla F(\bx^t) + \eta(\nabla F(\bx^t) - \btg(\bx^t))\|^2 \nonumber \\
&= \|\bx^t - \bx^* - \eta\nabla F(\bx^t)\|^2 + \eta^2\|\nabla F(\bx^t) - \btg(\bx^t)\|^2 \nonumber \\
&\hspace{3cm}+ 2\left\langle \bx^t - \bx^* -\eta\nabla F(\bx^t), \eta(\nabla F(\bx^t) - \btg(\bx^t)) \right\rangle \label{convex_interim1}
\end{align}
First we bound the last term of \eqref{convex_interim1}. 
For this, we use a simple but very powerful trick for inner-products, that allows us to get a contracting recurrence on $\|\bx^{t+1}-\bx^*\|^2$. Let $\bu = \bx^t - \bx^* -\eta\nabla F(\bx^t)$ and $\bv = \nabla F(\bx^t) - \btg(\bx^t)$. With this notation, the last term of \eqref{convex_interim1} is equal to $2\langle \bu, \eta\bv \rangle$.
\begin{align*}
2\langle \bu, \eta\bv \rangle = 2\left\langle \sqrt{\frac{\mu\eta}{2}}\bu, \sqrt{\frac{2\eta}{\mu}}\bv \right\rangle \stackrel{\text{(a)}}{\leq} \frac{\mu\eta}{2} \|\bu\|^2 + \frac{2\eta}{\mu}\|\bv\|^2,
\end{align*}
where in (a) we used the inequality $2\langle \bu, \bv \rangle \leq \|\bu\|^2 + \|\bv\|^2$.
Substituting this with the values of $\bu$ and $\bv$ in \eqref{convex_interim1} gives
\begin{align}
\|\bx^{t+1} - \bx^* \|^2 &\leq \left(1+\frac{\mu\eta}{4}\right)\|\bx^t - \bx^* - \eta\nabla F(\bx^t)\|^2 + \eta\left(\eta + \frac{4}{\mu}\right)\left\|\nabla F(\bx^t) - \btg(\bx^t)\right\|^2. \label{convex_interim2}
\end{align}
Now we bound the first term on the RHS of \eqref{convex_interim2}:
\begin{align}
\|\bx^t - \bx^* - \eta\nabla F(\bx^t)\|^2 &= \|\bx^t - \bx^*\|^2 + \eta^2\|\nabla F(\bx^t)\|^2 + 2\eta\left\langle \bx^*-\bx^t, \nabla F(\bx^t) \right\rangle \label{convex_interim3}
\end{align}
We can bound the second term on the RHS of \eqref{convex_interim3} using $L$-smoothness of $F$:
\begin{align}\label{convex_interim3-1}
\|\nabla F(\bx^t)\|^2 = \|\nabla F(\bx^t) - \nabla F(\bx^*)\|^2 \leq L^2\|\bx^t-\bx^*\|^2,
\end{align}
where in the first equality we used $\nabla F(\bx^*)=\bzero$.
To bound the third term of \eqref{convex_interim3}, we use $\mu$-strong convexity of $F$:
\begin{align*}
F(\bx^*) &\geq F(\bx^t) + \left\langle \nabla F(\bx^t), \bx^*-\bx^t\right\rangle + \frac{\mu}{2}\|\bx^t-\bx^*\|^2 \\
F(\bx^t) &\geq F(\bx^*) + \frac{\mu}{2}\|\bx^t-\bx^*\|^2
\end{align*}
In the second inequality we used $\nabla F(\bx^*)=\bzero$. Adding the above two inequalities gives
\begin{align}\label{convex_interim3-2}
\left\langle \nabla F(\bx^t), \bx^*-\bx^t\right\rangle \leq - \mu\|\bx^t-\bx^*\|^2
\end{align}
Substituting these bounds in \eqref{convex_interim3} gives
\begin{align}
\|\bx^t - \bx^* - \eta\nabla F(\bx^t)\|^2 &\leq (1+ \eta^2L^2  - 2\eta\mu)\|\bx^t-\bx^*\|^2.\label{convex_interim4}
\end{align}
Substituting this in \eqref{convex_interim2} and taking expectation w.r.t.\ sampling at the $t$'th iteration (while conditioning on the past) gives:
\begin{align}
\bbE_t\|\bx^{t+1} - \bx^* \|^2 \leq \left(1+\frac{\mu\eta}{2}\right)(1+ \eta^2L^2  - 2\eta\mu)\|\bx^t-\bx^*\|^2 + \eta\left(\eta + \frac{2}{\mu}\right)\bbE_t\left\|\nabla F(\bx^t) - \btg(\bx^t)\right\|^2  \label{convex_interim5}
\end{align}
Note that sampling at the $t$'th iteration does not affect $\bx^t$, and, therefore, does not affect the $\|\bx^t-\bx^*\|^2$ term in \eqref{convex_interim5}.
Now we bound the last term of \eqref{convex_interim5}. 
\begin{claim}\label{claim:deviation_true-est-grad}
With probability at least $1- \exp(-\frac{\eps'^2(1-\eps)R}{16})$, we have
$$\bbE_t\left\|\nabla F(\bx^t) - \btg(\bx^t)\right\|^2 \leq \frac{3\sigma^2}{(1-(\eps+\eps'))bR} +  3\kappa^2 + 3\varUpsilon^2,$$ where $\varUpsilon = \calO\left(\sigma_0\sqrt{\eps+\eps'}\right)$ and
 $\sigma_0^2 = \frac{24\sigma^2}{b\eps'}\left(1 + \frac{d}{(1-(\eps+\eps'))R}\right)+16\kappa^2$.
\end{claim}
\begin{proof}
Let $\calS^t$ denote the subset of uncorrupted gradients of size $(1-(\eps+\eps'))R$ that our algorithm in \Theoremref{gradient-estimator} approximates at the $t$'th iteration. Note that $\calS^t$ exists with probability at least $1- \exp(-\frac{\eps'^2(1-\eps)R}{16})$.
Let $\bg_{\calS^t}(\bx^t)=\sum_{r\in\calS^t}\bg_{r}(\bx^t)$ denote the average gradients of workers in $\calS^t$. 
\begin{align*}
\bbE_t\left\|\nabla F(\bx^t) - \btg(\bx^t)\right\|^2 = \bbE_t\left\|\left(\nabla F(\bx^t) - \nabla F_{\calS^t}(\bx^t)\right) + \left(\nabla F_{\calS^t}(\bx^t) - \bg_{\calS^t}(\bx^t)\right) + \left(\bg_{\calS^t}(\bx^t) - \btg(\bx^t)\right)\right\|^2, 
\end{align*}
where $F_{\calS^t}(\bx^t)=\frac{1}{|\calS^t|}\sum_{r\in\calS^t}F_r(\bx^t)$. Now using $\left\|\sum_{i=1}^k\bu_i\right\|^2 \leq k\sum_{i=1}^k\|\bu_i\|^2$, which holds for any positive integer $k$ and an arbitrary set of $k$ vectors, we get
\begin{align}
\bbE_t\left\|\nabla F(\bx^t) - \btg(\bx^t)\right\|^2 &\leq 3\bbE_t\left\|\nabla F(\bx^t) - \nabla F_{\calS^t}(\bx^t)\right\|^2 + 3\bbE_t\left\|\nabla F_{\calS^t}(\bx^t) - \bg_{\calS^t}(\bx^t)\right\|^2 \notag \\
&\hspace{3cm} + 3\bbE_t\left\|\bg_{\calS^t}(\bx^t) - \btg(\bx^t)\right\|^2. \label{convex_interim55} 
\end{align}
Now we bound each of the three terms on the RHS of \eqref{convex_interim55} separately.

\paragraph{Bounding the first term of \eqref{convex_interim55}:} Note that $F_r(\bx^t)$ for any $r\in[R]$, is a deterministic quantity w.r.t.\ the randomness used in the $t$'th iteration.
\begin{align}
\bbE_t\left\|\nabla F(\bx^t) - \nabla F_{\calS^t}(\bx^t)\right\|^2 &= \left\|\nabla F(\bx^t) - \nabla F_{\calS^t}(\bx^t)\right\|^2 \nonumber \\
&= \left\|\frac{1}{|\calS^t|}\sum_{r\in\calS^t}\left(F_r(\bx^t)-\nabla F(\bx^t)\right)\right\|^2 \nonumber \\
&\leq \frac{1}{|\calS^t|}\sum_{r\in\calS^t}\left\|F_r(\bx^t)-\nabla F(\bx^t)\right\|^2 \tag{using the Jensen's inequality} \\
&\leq \kappa^2 \label{convex_proof_bound1}
\end{align}
The last inequality follows from \eqref{bounded_local-global}.

\paragraph{Bounding the second term of \eqref{convex_interim55}:}
\begin{align}
\bbE_t\left\|\bg_{\calS^t}(\bx^t) - \nabla F_{\calS^t}(\bx^t)\right\|^2 &= \frac{1}{|\calS^t|^2}\bbE_t\left\|\sum_{r\in\calS^t}\bg_{r}(\bx^t) - \sum_{r\in\calS^t}\nabla F_r(\bx^t)\right\|^2 \nonumber \\
&\stackrel{\text{(a)}}{=} \frac{1}{|\calS^t|^2}\sum_{r\in\calS^t}\bbE_t\left\|\bg_{r}(\bx^t) - \nabla F_r(\bx^t)\right\|^2 \nonumber \\
&\stackrel{\text{(b)}}{\leq} \frac{1}{|\calS^t|^2}\sum_{r\in\calS^t}\frac{\sigma^2}{b} \leq \frac{\sigma^2}{(1-(\eps+\eps'))bR} \label{convex_proof_bound2}
\end{align}
Since the local stochastic gradients are sampled independently across different workers, we have that $\bg_{r}(\bx^t), r=1,\hdots,R$ are independent random variables with $\bbE_t\left[\bg_r(\bx^t)\right] = \nabla F_r(\bx^t)$, implying that $\bbE_t\left[\sum_{r\in\calS^t}\bg_r(\bx^t)\right] = \sum_{r\in\calS^t}\nabla F_r(\bx^t)$.
Now (a) follows from the fact that if $X$ and $Y$ are independent random variables, then $Var(X+Y)=Var(X)+Var(Y)$. In (b) we used the bounded local variance assumption \eqref{reduced_variance}, as $\bg_r(\bx^t)$ is sampled uniformly from $\calF_r^{\otimes b}(\bx^t)$. In the last inequality we used that $|\calS^t|\geq(1-(\eps+\eps'))R$. 

\paragraph{Bounding the third term of \eqref{convex_interim55}:} 
It follows from \Theoremref{gradient-estimator} that
\begin{align}
\left\| \bg_{\calS^t}(\bx^t) - \btg(\bx^t) \right\|^2 \leq \calO\left(\sigma_0^2(\eps+\eps')\right), \label{convex_proof_bound3}
\end{align}
where $\sigma_0^2 = \frac{24\sigma^2}{b\eps'}\left(1 + \frac{d}{(1-(\eps+\eps'))R}\right)+16\kappa^2$.

Substituting the bounds from \eqref{convex_proof_bound1}-\eqref{convex_proof_bound3} in \eqref{convex_interim55} proves \Claimref{deviation_true-est-grad}.
\end{proof}

\noindent Substituting the bound from \Claimref{deviation_true-est-grad} in \eqref{convex_interim5} gives the following with probability at least $1- \exp(-\frac{\eps'^2(1-\eps)R}{16})$: 
\begin{align*}
\bbE_t\|\bx^{t+1} - \bx^* \|^2 &\leq \left(1+\frac{\mu\eta}{2}\right)(1+ \eta^2L^2  - 2\eta\mu)\|\bx^t-\bx^*\|^2 \\
&\hspace{2cm} + \eta\left(\eta + \frac{2}{\mu}\right)\left(\frac{3\sigma^2}{(1-(\eps+\eps'))bR} +  3\kappa^2 + 3\varUpsilon^2\right).
\end{align*}
Using $\eta=\frac{\mu}{L^2}$ implies $\left(1+\frac{\mu\eta}{2}\right)(1+ \eta^2L^2 - 2\eta\mu) = (1+\frac{\mu^2}{2L^2})(1-\frac{\mu^2}{L^2})\leq 1-\frac{\mu^2}{2L^2}$ for the first term.
Note that $L\geq\mu$, which implies $\eta\leq\frac{1}{\mu}$. Using this, we can bound the coefficient of the last term as $\eta\left(\eta + \frac{2}{\mu}\right) \leq \frac{3}{\mu^2}$.
Substituting these, and taking expectation w.r.t.\ the entire process yield
\begin{align}
\bbE\|\bx^{t+1} - \bx^* \|^2 &\leq \left(1-\frac{\mu^2}{2L^2}\right)\bbE\|\bx^t-\bx^*\|^2 + \frac{\varGamma}{\mu^2}, \label{convex_interim6}
\end{align}
where $\varGamma := \frac{9\sigma^2}{(1-(\eps+\eps'))bR} + 9\kappa^2 + 9\varUpsilon^2$. 
Solving the recurrence in \eqref{convex_interim6} gives
\begin{align}
\bbE\|\bx^{T} - \bx^* \|^2 &\leq \left(1-\frac{\mu^2}{2L^2}\right)^T\|\bx^0-\bx^*\|^2 + \frac{2L^2}{\mu^4}\varGamma. \label{convex_interim7}
\end{align}
If we take $T\geq \frac{\log\left(\frac{\mu^4}{L^2\varGamma}\|\bx^0-\bx^*\|^2\right)}{\log(\frac{1}{1-\nicefrac{\mu^2}{2L^2}})}$, we get $\bbE\|\bx^{T} - \bx^* \|^2 \leq \frac{3L^2}{\mu^4}\varGamma$. 

\paragraph{Error probability analysis.}
Note that the recurrence in \eqref{convex_interim6} holds with probability at least $1- \exp(-\frac{\eps'^2(1-\eps)R}{16})$. Since we apply this recurrence $T$ times to get \eqref{convex_interim7}, it follows by the union bound that \eqref{convex_interim7} holds with probability at least $1- T\exp(-\frac{\eps'^2(1-\eps)R}{16})$, which is at least $(1-\delta)$, for any $\delta>0$, provided we run our algorithm for $T \leq \delta\exp(\frac{\eps'^2(1-\eps)R}{16})$ iterations.

This completes our proof of the strongly-convex part of \Theoremref{SGD_convergence}.

\subsection{Proof of \Theoremref{SGD_convergence} (non-convex)}\label{app:nonconvex_convergence}
First we prove the result when the parameter space $\calC=\R^d$. In this case we do not need projection of iterates onto $\calC$.
Recall the update rule of our algorithm: $\bx^{t+1} = \bx^t - \eta\btg(\bx^t),\ t=1,2,3,\hdots$
\begin{align}
F(\bx^{t+1}) &\stackrel{\text{(a)}}{\leq} F(\bx^t) + \langle \nabla F(\bx^t), \bx^{t+1}-\bx^t \rangle + \frac{L}{2}\|\bx^{t+1}-\bx^t\|^2 \nonumber \\
&= F(\bx^t) - \eta \langle \nabla F(\bx^t), \btg(\bx^t) \rangle + \frac{\eta^2L}{2}\|\btg(\bx^t)\|^2 \nonumber \\
&= F(\bx^t) - \eta \langle \nabla F(\bx^t), \btg(\bx^t) -\nabla F(\bx^t) + \nabla F(\bx^t) \rangle + \frac{\eta^2L}{2}\|\btg(\bx^t) - \nabla F(\bx^t) + \nabla F(\bx^t)\|^2 \nonumber \\
&\stackrel{\text{(b)}}{\leq} F(\bx^t) - \eta \langle \nabla F(\bx^t), \btg(\bx^t) -\nabla F(\bx^t) \rangle - \eta\|\nabla F(\bx^t)\|^2 \notag \\
&\hspace{4cm} + \frac{\eta^2L}{2}\left(2\|\btg(\bx^t) - \nabla F(\bx^t)\|^2 + 2\|\nabla F(\bx^t)\|^2 \right)\nonumber \\
&\stackrel{\text{(c)}}{\leq} F(\bx^t) + \eta\left(\frac{1}{2}\|\nabla F(\bx^t)\|^2 + \frac{1}{2}\|\btg(\bx^t) -\nabla F(\bx^t)\|^2 \right) - \eta(1-\eta L)\|\nabla F(\bx^t)\|^2 \notag \\
&\hspace{4cm} + \eta^2L\|\btg(\bx^t) - \nabla F(\bx^t)\|^2  \nonumber \\
&= F(\bx^t) - \eta\left(\frac{1}{2} - \eta L\right)\|\nabla F(\bx^t)\|^2 + \eta\left(\frac{1}{2}+\eta L\right)\|\btg(\bx^t) - \nabla F(\bx^t)\|^2\label{convergence_interim1}
\end{align}
In (a) we used our assumption that $F$ is $L$-smooth; 
in (b) we used the inequality $\|\bu+\bv\|^2 \leq 2(\|\bu\|^2 + \|\bv\|^2)$; 
in (c) we used the inequality $2\langle \bu, \bv \rangle \leq \|\bu\|^2 + \|\bv\|^2$.
For $\eta \leq \nicefrac{1}{4L}$, we have $\nicefrac{1}{2}-\eta L \geq \nicefrac{1}{4}$ and $\nicefrac{1}{2} + \eta L \leq \nicefrac{3}{4}$. Substituting these in \eqref{convergence_interim1} and taking expectation w.r.t.\ the sampling at the $t$'th iteration (while conditioning on the past) gives 
\begin{align}
\bbE_t[F(\bx^{t+1})] &\leq F(\bx^t) - \frac{\eta}{4}\|\nabla F(\bx^t)\|^2 + \frac{3\eta}{4}\bbE_t\|\btg(\bx^t) - \nabla F(\bx^t)\|^2. \label{convergence_interim2}
\end{align} 
As argued in \Claimref{deviation_true-est-grad} in the proof of the strongly-convex part of  \Theoremref{SGD_convergence} above, we can similarly bound the last term of \eqref{convergence_interim2} as $\bbE_t\|\btg(\bx^t) - \nabla F(\bx^t)\|^2 \leq \frac{3\sigma^2}{(1-(\eps+\eps'))bR} + 3\kappa^2 + 3\varUpsilon^2$, where $\varUpsilon \leq  \calO(\sigma_0\sqrt{\epsilon+\epsilon'})$, with probability at least $1- \exp(-\frac{\eps'^2(1-\eps)R}{16})$, which comes from \Theoremref{gradient-estimator}. Using this in \eqref{convergence_interim2} gives
\begin{align}
\bbE_t[F(\bx^{t+1})] &\leq F(\bx^t) - \frac{\eta}{4}\|\nabla F(\bx^t)\|^2 + \frac{3\eta}{4}\left(\frac{3\sigma^2}{(1-(\eps+\eps'))bR} + 3\kappa^2 + 3\varUpsilon^2\right). \label{convergence_interim2-1}
\end{align} 
By taking a telescoping sum from $t=0$ to $T$, and also taking expectation w.r.t.\ the sampling at all workers throughout the process, we get
\begin{align}
\frac{1}{T}\sum_{t=0}^T\bbE\|\nabla F(\bx^t)\|^2 &\leq \frac{4}{\eta T}\bbE[F(\bx^0) - F(\bx^{T+1})] + \varGamma, \label{convergence_interim2-2}
\end{align}
where $\varGamma = \frac{9\sigma^2}{(1-(\eps+\eps'))bR} + 9\kappa^2 + 9\varUpsilon^2$.
Substituting $\bbE[F(\bx^{T+1}))] \geq F(\bx^*)$ and using $F(\bx^0) - F(\bx^*) \leq \frac{L}{2}\|\bx^0-\bx^*\|^2$ (which follows from $L$-smoothness of $F$) gives
\begin{align}
\frac{1}{T}\sum_{t=0}^T\bbE\|\nabla F(\bx^t)\|^2 &\leq \frac{2L}{\eta T}\|\bx^0-\bx^*\|^2 + \varGamma. \label{convergence_interim3}
\end{align}
Note that the last term $\varGamma$ in \eqref{convergence_interim3} is a constant. 
So, it would be best to take the learning rate $\eta$ to be as large as possible such that it satisfies $\eta\leq\nicefrac{1}{4L}$ (we used this bound to arrive at \eqref{convergence_interim2}). We take $\eta=\nicefrac{1}{4L}$. Substituting this in \eqref{convergence_interim3} gives
\begin{align}\label{convergence_interim4}
\frac{1}{T}\sum_{t=0}^T\bbE\|\nabla F(\bx^t)\|^2 &\leq \frac{8L^2}{T}\|\bx^0-\bx^*\|^2 + \varGamma.
\end{align}
If we run our algorithm for $T= \frac{8L^2\|\bx^0-\bx^*\|^2}{\varGamma}$ iterations, we get $\frac{1}{T}\sum_{t=0}^T\bbE\|\nabla F(\bx^t)\|^2 \leq 2\varGamma$.

This concludes our proof of the non-convex part of \Theoremref{SGD_convergence} when $\calC=\R^d$.

\paragraph{When $\calC$ is a bounded set.} 
When $\calC$ is bounded, we need projection in general.
But, under \Assumptionref{non-convex_size-params-space}, we show that all iterates $\bhx^t, t=1,\hdots,T$ stay within $\calC$, i.e., we do not projection in the parameter update rule. 
This is what we show in the following; and, as a result, the above convergence analysis (which we did for $\calC=\R^d$) suffices under \Assumptionref{non-convex_size-params-space}.
\begin{lemma}\label{lem:non-convex_projection-not-needed}
Suppose \Assumptionref{non-convex_size-params-space} holds. Then, in the setting of \Theoremref{SGD_convergence}, with probability at least $1- T\exp(-\frac{\eps'^2(1-\eps)R}{16})$, we have $\bhx^t\in\calC$ for every $t=0,1,2,\hdots,T$.
\end{lemma}
\begin{proof}
Define $T$ events $\calE_1,\calE_2,\hdots,\calE_T$ as follows: For every $t\in[T]$, define
$\calE_t := \{\bhx^t\in\calC \ | \ \bhx^i\in\calC, \forall i\in[t-1]\}$.
Note that 
\begin{align*}
\Pr[\bhx^t\in\calC, \forall t\in[T]] = \Pi_{t=1}^T\Pr[\bhx^t\in\calC \ | \ \bhx^i\in\calC, \forall i\in[t-1]] = \Pi_{t=1}^T\Pr[\calE_t].
\end{align*}
Now we calculate $\Pr[\calE_t], t\in[T]$. 

Fix an arbitrary $t\in[T]$. In $\calE_t$, we are given that $\bhx^i\in\calC, \forall i\in[t-1]$, 
which, by \Assumptionref{non-convex_size-params-space}, implies that $\bhx^i=\bx^i$ and $\|\nabla F(\bx^i)\|\leq M$ for all $i\in[t-1]$. Now we show that $\bhx^{k+1}\in\calC$ holds with probability at least $1- \exp(-\frac{\eps'^2(1-\eps)R}{16})$. For this, in view of \Assumptionref{non-convex_size-params-space}, it suffices to show that $\|\bhx^{t} - \bx^0\| \leq \frac{2L}{\varGamma}(M+\varGamma_1)\|\bx^0-\bx^*\|^2$
\begin{align}
\|\bhx^{t} - \bx^0\| &= \|\bx^{t-1} - \eta\btg(\bx^{t-1}) - \bx^0\| \tag{Since $\bhx^i=\bx^i, \forall i\in[t-1]$} \\
&\leq \|\bx^{t-1} - \bx^0\| + \eta\left(\|\nabla F(\bx^{t-1})\| + \|\btg(\bx^{t-1})-\nabla F(\bx^{t-1})\|\right) \nonumber \\
&\leq \sum_{i=0}^{t-1}\eta\left(\|\nabla F(\bx^{i})\| + \|\btg(\bx^{i})-\nabla F(\bx^{i})\|\right) \nonumber \\
&\stackrel{\text{(a)}}{\leq} \frac{t}{4L}\left(M+\varGamma_1\right) 
\leq \frac{T}{4L}\left(M+\varGamma_1\right)  \nonumber \\
&\stackrel{\text{(b)}}{\leq} \frac{2L}{\varGamma}(M+\varGamma_1)\|\bx^0-\bx^*\|^2. \nonumber 
\end{align}
In (a), we used \Claimref{robust-grad-est-under-deterministic-variance} to bound $\|\btg(\bx^k)-\nabla F(\bx^k)\|\leq\varGamma_1$, which holds with probability at least $1- \exp(-\frac{\eps'^2(1-\eps)R}{16})$. Note that, we used a deterministic bound on $\|\btg(\bx^k)-\nabla F(\bx^k)\|$, as opposed to a bound which only holds in expectation (even though we are doing SGD). See the proof of \Claimref{robust-grad-est-under-deterministic-variance} for details.
In (b), we used $T=\frac{8L^2\|\bx^0-\bx^*\|^2}{\varGamma}$.

We have shown that for every $t\in[T]$, we have $\Pr[\calE_t] \geq 1- \exp(-\frac{\eps'^2(1-\eps)R}{16})$. This implies that $\Pr[\bhx^t\in\calC, \forall t\in[T]] \geq \left(1- \exp(-\frac{\eps'^2(1-\eps)R}{16})\right)^T \geq 1- T\exp(-\frac{\eps'^2(1-\eps)R}{16})$.

\end{proof}

\begin{claim}\label{claim:robust-grad-est-under-deterministic-variance}
For any $\bx\in\calC$, with probability at least $1- \exp(-\frac{\eps'^2(1-\eps)R}{16})$, we have
$$\varGamma_1 := \left\|\nabla F(\bx) - \btg(\bx)\right\| \leq \frac{n_{\max}\sigma}{b} + \kappa + \varUpsilon,$$ 
where $n_{\max}=\max_{r\in[R]}n_r$, $\varUpsilon = \calO\left(\sigma_0\sqrt{\eps+\eps'}\right)$ and
 $\sigma_0^2 = \frac{24\sigma^2}{b\eps'}\left(1 + \frac{d}{(1-(\eps+\eps'))R}\right)+16\kappa^2$.
\end{claim}
\begin{proof}
Let $\calS$ denote the subset of uncorrupted gradients of size $(1-(\eps+\eps'))R$ that our algorithm in \Theoremref{gradient-estimator} approximates, and let $\bg_{\calS}(\bx)$ denote their average gradient. 
\begin{align}
\left\|\nabla F(\bx) - \btg(\bx)\right\| \leq \left\|\nabla F(\bx) - \nabla F_{\calS}(\bx)\right\| + \left\|\nabla F_{\calS}(\bx) - \bg_{\calS}(\bx)\right\| + \left\|\bg_{\calS}(\bx) - \btg(\bx)\right\|. \label{nonconvex_deterministic1} 
\end{align}
The first and the third terms on the RHS of \eqref{nonconvex_deterministic1} can be bounded similarly as we bounded them in \eqref{convex_proof_bound1} and \eqref{convex_proof_bound3}, respectively.
\begin{align}
\left\|\nabla F(\bx) - \nabla F_{\calS}(\bx)\right\| &\leq \kappa \label{nonconvex_deterministic_bound1} \\
\left\| \bg_{\calS}(\bx) - \btg(\bx) \right\| &\leq \calO\left(\sigma_0\sqrt{(\eps+\eps')}\right)\label{nonconvex_deterministic_bound3}
\end{align}
In \eqref{nonconvex_deterministic_bound3}, $\sigma_0^2 = \frac{24\sigma^2}{b\eps'}\left(1 + \frac{d}{(1-(\eps+\eps'))R}\right)+16\kappa^2$.

Now we bound the second term of \eqref{nonconvex_deterministic1}.
Note that in \eqref{convex_proof_bound2} we bounded $\bbE\left\|\nabla F_{\calS}(\bx) - \bg_{\calS}(\bx)\right\|^2$, where expectation is taken over the sampling of the mini-batch stochastic gradients at workers in $\calS$. However, for the second term of \eqref{nonconvex_deterministic1}, we need a deterministic bound, which holds for every choice of the mini-batch at workers in $\calS$. 
Fix arbitrary sets $\calH_n^r\in_U\binom{[n_r]}{b}$ for every $r\in\calS$.
\begin{align}
\left\|\bg_{\calS}(\bx) - \nabla F_{\calS}(\bx)\right\| &= \left\|\frac{1}{|\calS|}\sum_{r\in\calS}\left(\bg_{r}(\bx) - \nabla F_r(\bx)\right)\right\| \nonumber \\
&\leq \frac{1}{|\calS|}\sum_{r\in\calS}\left\|\bg_{r}(\bx) - \nabla F_r(\bx)\right\| \nonumber \\
&= \frac{1}{|\calS|}\sum_{r\in\calS}\left\|\frac{1}{b}\sum_{i\in\calH_b^r}\left(\nabla F_{r,i}(\bx) - \nabla F_r(\bx)\right)\right\| \tag{Since $\bg_{r}(\bx) = \frac{1}{b}\sum_{i\in\calH_b^r}\nabla F_{r,i}(\bx)$} \\
&\leq \frac{1}{|\calS|}\sum_{r\in\calS}\frac{1}{b}\sum_{i\in\calH_b^r}\left\|\nabla F_{r,i}(\bx) - \nabla F_r(\bx)\right\| \nonumber \\
&\stackrel{\text{(a)}}{\leq} \frac{1}{|\calS|}\sum_{r\in\calS}\frac{1}{b}n_r\sigma 
\leq \frac{n_{\max}\sigma}{b} \label{nonconvex_deterministic_bound2}
\end{align}
In the last inequality we used $n_{\max}=\max_{r\in[R]}n_r$.
The explanation for (a) is as follows:
We have from \Assumptionref{bounded_local-variance} that the stochastic gradients at any worker have bounded variance, i.e., for every $r\in[R],\bx\in\calC$, we have $\mathbb{E}_{i\in_U[n_r]}\|\nabla F_{r,i}(\bx) - \nabla F_r(\bx)\|^2\leq \sigma^2$. Using Jensen's inequality gives $\mathbb{E}_{i\in_U[n_r]}\|\nabla F_{r,i}(\bx) - \nabla F_r(\bx)\| \leq \sigma$. This is equivalent to 
$\frac{1}{n_r}\sum_{i=1}^{n_r}\|\nabla F_{r,i}(\bx) - \nabla F_r(\bx)\| \leq \sigma$, which implies that $\sum_{i\in\calT}\|\nabla F_{r,i}(\bx) - \nabla F_r(\bx)\| \leq n_r\sigma$ holds for every $\calT\subseteq[n_r]$.
Since \eqref{nonconvex_deterministic_bound2} holds for arbitrary $\calH_b^r$'s, it holds for all $\calH_b^r\in\binom{[n_r]}{b},r\in\calS^t$.

Substituting the bounds from \eqref{nonconvex_deterministic_bound1}-\eqref{nonconvex_deterministic_bound2} in \eqref{nonconvex_deterministic1} proves \Claimref{robust-grad-est-under-deterministic-variance}.
\end{proof}

\paragraph{Error probability analysis.}
Note that the recurrence in \eqref{convergence_interim2-1} holds with probability at least $1- \exp(-\frac{\eps'^2(1-\eps)R}{16})$. We use this recurrence $T$ times to get \eqref{convergence_interim2-2}, it follows by the union bound that \eqref{convergence_interim2-2} holds with probability at least $1- T\exp(-\frac{\eps'^2(1-\eps)R}{16})$, which is at least $(1-\delta)$, for any $\delta>0$, provided we run our algorithm for $T \leq \delta\exp(\frac{\eps'^2(1-\eps)R}{16})$ iterations. Note that this is the same probability with which all iterates $\bhx^t,t\in[T]$ lie in $\calC$ (see \Lemmaref{non-convex_projection-not-needed}).

\subsection{Proof of \Theoremref{full-batch-GD}}\label{app:convergence_full-batch-GD}
In this section, we focus on the case when workers compute {\em full-batch} gradients (instead of computing mini-batch stochastic gradients) and prove \Theoremref{full-batch-GD}.
Note that the robust gradient estimator of \Theoremref{gradient-estimator} (which is for stochastic gradients) is the main ingredient behind the convergence analyses of \Algorithmref{Byz-SGD} for strongly-convex and non-convex objectives. So, in order to prove \Theoremref{full-batch-GD}, first we need to show a robust gradient estimation result (similar to \Theoremref{gradient-estimator}). Since we are working with full-batch gradients, we can show an analogous result with a much simplified analysis.

Note that, to prove \Theoremref{gradient-estimator}, we showed an existence of a subset of honest workers from which the stochastic gradients are well concentrated, as stated in form of a matrix concentration bound \eqref{mat-concen_gradient-estimator} in the first part of \Theoremref{gradient-estimator}.
It turns out that for full-batch gradients, an analogous result can be proven directly 
(as there is no randomness due to stochastic gradients); and below we provide and prove such a result.
Note that \Theoremref{gradient-estimator} is a probabilistic statement, where we show that with high probability, there exists a large subset of honest workers whose stochastic gradients are well concentrated. In contrast, in the following result, we can deterministically take the set of all honest workers to be that subset for which we can directly show the concentration.

\begin{theorem}[Robust Gradient Estimation for Full-Batch Gradient Descent]\label{thm:gradient-estimator_GD}
Fix an arbitrary $\bx\in\R^d$.
Suppose an $\eps$ fraction of workers are corrupt and we are given $R$ full-batch gradients $\nabla\btlF_1(\bx),\hdots,\nabla\btlF_R(\bx)\in\R^d$, where $\nabla\btlF_r(\bx)=\nabla F_r(\bx)$ if the $r$'th worker is honest, otherwise can be arbitrary. Let $\calS$ be the set of all honest workers and $\nabla F_{\calS}(\bx):=\frac{1}{|\calS|}\sum_{i\in\calS}\nabla F_i(\bx)$ be the sample average of uncorrupted gradients. We can find an estimate $\nabla\btF(\bx)$ of $\nabla F_{\calS}(\bx)$ in polynomial-time such that $\left\| \nabla\btF(\bx) - \nabla F_{\calS}(\bx) \right\| \leq \calO\left(\kappa\sqrt{\eps}\right)$ holds with probability 1.
\end{theorem}

\begin{proof}
First we prove that $\lambda_{\max}\(\frac{1}{|\calS|}\sum_{i\in\calS} \(\nabla F_i(\bx) - \nabla F_{\calS}(\bx)\)\(\nabla F_i(\bx) - \nabla F_{\calS}(\bx)\)^T\) \leq 4\kappa^2$. 
In view of \eqref{alternate-defn_max-eigenvalue}, this is equivalent to showing that for every unit vector $\bv\in\R^d$, we have $\frac{1}{|\calS|}\sum_{i\in\calS} \langle \nabla F_i(\bx) - \nabla F_{\calS}(\bx), \bv \rangle^2 \leq 4\kappa^2$.
\begin{align*}
\frac{1}{|\calS|}\sum_{i\in\calS} &\left\langle \nabla F_i(\bx) - \nabla F_{\calS}(\bx), \bv \right\rangle^2
= \frac{1}{|\calS|}\sum_{i\in\calS} \left[\left\langle \nabla F_i(\bx) - \nabla F(\bx) + \nabla F(\bx) - \nabla F_{\calS}(\bx), \bv \right\rangle \right]^2 \\
&\leq \frac{2}{|\calS|}\sum_{i\in\calS} \left\langle \nabla F_i(\bx) - \nabla F(\bx), \bv \right\rangle^2 + \frac{2}{|\calS|}\sum_{i\in\calS} \left\langle \nabla F_{\calS}(\bx) - \nabla F(\bx), \bv \right\rangle^2 \\
&\leq \frac{2}{|\calS|}\sum_{i\in\calS} \left\langle \nabla F_i(\bx) - \nabla F(\bx), \bv \right\rangle^2 + 2\left\langle \nabla F_{\calS}(\bx) - \nabla F(\bx), \bv \right\rangle^2 \\
&= \frac{2}{|\calS|}\sum_{i\in\calS} \left\langle \nabla F_i(\bx) - \nabla F(\bx), \bv \right\rangle^2 + 2\left[\frac{1}{|\calS|}\sum_{i\in\calS}\left\langle \nabla F_{i}(\bx) - \nabla F(\bx), \bv \right\rangle\right]^2 \\
&\leq \frac{2}{|\calS|}\sum_{i\in\calS} \left\langle \nabla F_i(\bx) - \nabla F(\bx), \bv \right\rangle^2 + \frac{2}{|\calS|}\sum_{i\in\calS}\left\langle \nabla F_{i}(\bx) - \nabla F(\bx), \bv \right\rangle^2 \\
&= \frac{4}{|\calS|}\sum_{i\in\calS}\left\langle \nabla F_{i}(\bx) - \nabla F(\bx), \bv \right\rangle^2 \\
&\leq \frac{4}{|\calS|}\sum_{i\in\calS}\left\|\nabla F_{i}(\bx) - \nabla F(\bx)\right\|^2 \\
&\leq 4\kappa^2
\end{align*}
Now apply the second part of \Theoremref{gradient-estimator} with $\calS$ being the set of honest workers$,  \bg_i=\nabla F_i(\bx)$, $\eps'=0$, and $\sigma_0^2=4\kappa^2$. We would get that 
we can find an estimate $\nabla \btF(\bx)$ of $\nabla F_{\calS}(\bx)$ in polynomial-time, such that $\left\| \nabla\btF(\bx) - \nabla F_{\calS}(\bx) \right\| \leq \calO\left(\kappa\sqrt{\eps}\right)$ holds with probability 1.
\end{proof}
Now we proceed with proving \Theoremref{full-batch-GD}.
\begin{proof}[Proof of \Theoremref{full-batch-GD}]
This can be proved along the lines of the proof of \Theoremref{SGD_convergence}. 
Here we only write what changes in those proofs.

Recall the update rule of our algorithm with the full-batch gradient descent: 
$\bhx^{t+1}=\bx^t - \eta\nabla\btF(\bx^t);\  \bx^{t+1} = \Pi_{\calC}\left(\bhx^{t+1}\right),\ t=1,2,3,\hdots$.

\subsubsection{Strongly-convex}
Following that proof of the strongly-convex part of \Theoremref{SGD_convergence} given in \Appendixref{convex_convergence} until \eqref{convex_interim5} gives
\begin{align}
\|\bx^{t+1} - \bx^* \|^2 \leq \left(1+\frac{\mu\eta}{2}\right)(1+ \eta^2L^2  - 2\eta\mu)\|\bx^t-\bx^*\|^2 + \eta\left(\eta + \frac{2}{\mu}\right)\left\|\nabla F(\bx^t) - \nabla \btF(\bx^t)\right\|^2  \label{convex-GD_interim5}
\end{align}
Now we bound the last term of \eqref{convex-GD_interim5}. For this, we can simplify the proof of 
\Claimref{deviation_true-est-grad}: Firstly, note that, with full-batch gradients, the variance $\sigma^2$ becomes zero; secondly, as shown in \Theoremref{gradient-estimator_GD}, the robust estimation of gradients holds with probability 1. Following the proof of \Claimref{deviation_true-est-grad} with these changes, we get
\begin{align}\label{deviation_true-est-grad_GD}
\left\|\nabla F(\bx^t) - \nabla\btF(\bx^t)\right\|^2 \leq 2\kappa^2 + 2\varUpsilon_{\GD}^2,
\end{align}
where $\varUpsilon_{\GD}=\calO(\kappa\sqrt{\eps})$.
Substituting the bound from \eqref{deviation_true-est-grad_GD} in \eqref{convex-GD_interim5} and 
then following the proof given in \Appendixref{convex_convergence} until \eqref{convex_interim7} gives
\begin{align}
\|\bx^{T} - \bx^* \|^2 &\leq \left(1-\frac{\mu^2}{2L^2}\right)^T\|\bx^0-\bx^*\|^2 + \frac{2L^2}{\mu^4}\varGamma_{\GD}, \label{convex-GD_interim7}
\end{align}
where $\varGamma_{\GD}=6\kappa^2 + 6\varUpsilon_{\GD}^2$.
Note that \eqref{convex-GD_interim7} holds with probability 1.

\subsubsection{Non-convex}
Using the bound from \eqref{deviation_true-est-grad_GD} and following the proof of the non-convex part of \Theoremref{SGD_convergence} given in \Appendixref{nonconvex_convergence} 
exactly until \eqref{convergence_interim4} gives
\begin{align}\label{convergence-GD_interim4}
\frac{1}{T}\sum_{t=0}^T\|\nabla F(\bx^t)\|^2 &\leq \frac{8L^2}{T}\|\bx^0-\bx^*\|^2 + \varGamma_{\GD}, 
\end{align}
where $\varGamma_{\GD}=6\kappa^2 + 6\varUpsilon_{\GD}^2$.
Note that \eqref{convergence-GD_interim4} holds with probability 1.

\paragraph{Projection.}
When $\calC$ is a bounded set,
then under \Assumptionref{non-convex_size-params-space_GD}, we can show that the iterates $\bhx^t$ for $t=0,1,2,\hdots,T$ lie in $\calC$ and we do not need projection.
For this, we can show a result analogous to \Lemmaref{non-convex_projection-not-needed}, but which holds with probability 1:
\begin{lemma}\label{lem:non-convex_projection-not-needed_GD}
Suppose \Assumptionref{non-convex_size-params-space_GD} holds. Then, in the setting of \Theoremref{full-batch-GD} (non-convex), with probability 1, we have $\bhx^t\in\calC$ for every $t=0,1,2,\hdots,T$.
\end{lemma}
\begin{proof}
We can prove this lemma along the lines of how we proved \Lemmaref{non-convex_projection-not-needed}, except with one modification. In order to bound $\left\|\nabla F(\bx) - \nabla\btF(\bx)\right\|$, we can put $\sigma=0$ in \Claimref{robust-grad-est-under-deterministic-variance}, and that gives $\left\|\nabla F(\bx) - \nabla\btF(\bx)\right\| \leq \kappa + \varUpsilon_{\GD}$. Furthermore, with full-batch gradients, this bound holds with probability 1. This is because with full-batch gradients, our robust gradient estimator from \Theoremref{gradient-estimator_GD} succeeds with probability 1.
\end{proof}
This concludes the proof of \Theoremref{full-batch-GD}.
\end{proof}

\section{Proof of \Lemmaref{randk}}\label{app:compression}
\begin{lemma*}[Restating \Lemmaref{randk}]
Fix any worker $r\in[R]$.
Suppose the local stochastic gradients at worker $r$ have bounded second moment, i.e., $\bbE_{i\in_U[n_r]}[\|\nabla F_{r,i}(\bx)\|^2] \leq G^2$. Then we have
\begin{align}
&\bbE_{\calK\leftarrow K,H_b}\left[\frac{d}{k}\cdot\select_{\calK}\left(\nabla F_{r,H_b}(\bx)\right)\right] = \nabla F_r(\bx) \label{same_mean_randk_app} \\
&\bbE_{\calK\leftarrow K,H_b}\left\|\frac{d}{k}\cdot\select_{\calK}\left(\nabla F_{r,H_b}(\bx)\right) - \nabla F_r(\bx)\right\|^2 \leq  \frac{d}{k}\frac{G^2}{b}. \label{variance_reduction_randk_app}
\end{align}
\end{lemma*}
\begin{proof}
For notational convenience, in this proof, we suppress the explicitly dependence on $r$ in $n_r,F_{r,H_b}(\bx),F_{r}(\bx)$, and denote them simply by $n,F_{H_b}(\bx),F(\bx)$, respectively.

Recall from \eqref{same_mean} that $\bbE_{H_b}[\nabla F_{H_b}(\bx)] = \nabla F(\bx)$.
Now showing \eqref{same_mean_randk_app} is straightforward:
\begin{align*}
\bbE_{\calK\leftarrow K,H_b}\left[\frac{d}{k}\cdot\select_{\calK}\left(\nabla F_{H_b}(\bx)\right)\right] &= 
\bbE_{\calK\leftarrow K}\left[\frac{d}{k}\cdot\select_{\calK}\left(\bbE_{H_b}\left[\nabla F_{H_b}(\bx)\right]\right)\right] \\
&= \bbE_{\calK\leftarrow K}\left[\frac{d}{k}\cdot\select_{\calK}\left(\nabla F(\bx)\right)\right] \\
&= \nabla F(\bx).
\end{align*} 
To show \eqref{variance_reduction_randk_app}, observe that once we fix $\calK$, $\select_{\calK}(\bv_1+\bv_2) = \select_{\calK}(\bv_1) + \select_{\calK}(\bv_2)$. 
This simple observation is crucial for getting the variance reduction.
Note that $H_b$ is a collection of $b$ elements drawn uniformly at random from $[n]$ {\em without} replacement.
However, we show \eqref{variance_reduction_randk_app} when $H_b$  is a collection of $b$ elements drawn uniformly at random from $[n]$ {\em with} replacement.
Observe that, the bound derived under the latter condition is an upper-bound on the quantity of interest, which is under the former condition.
Let $H_b=(I_1,I_2,\hdots,I_b)$, where $I_j$'s are i.i.d.~random variables taking values in $[n]$ with uniform distribution.
\begin{align*}
&\bbE_{\calK\leftarrow K,I_1,\hdots,I_b}\left\|\frac{d}{k}\cdot\select_{\calK}\left(\frac{1}{b}\sum_{j=1}^b\nabla F_{I_j}(\bx)\right) - \nabla F(\bx)\right\|^2 \\
&\quad= \bbE_{\calK\leftarrow K}\left[\bbE_{I_1,\hdots,I_b}\left\|\frac{d}{k}\cdot\select_{\calK}\left(\frac{1}{b}\sum_{j=1}^b\nabla F_{I_j}(\bx)\right) - \nabla F(\bx)\right\|^2\right] \\
&\quad= \bbE_{\calK\leftarrow K}\left[\bbE_{I_1,\hdots,I_b}\left\|\frac{1}{b}\sum_{j=1}^b\frac{d}{k}\cdot\select_{\calK}\left(\nabla F_{I_j}(\bx)\right) - \nabla F(\bx)\right\|^2\right] \\
&\quad= \bbE_{\calK\leftarrow K,I_1,\hdots,I_b}\left\|\frac{1}{b}\sum_{j=1}^b\left(\frac{d}{k}\cdot\select_{\calK}\left(\nabla F_{I_j}(\bx)\right) - \nabla F(\bx)\right)\right\|^2 \\
&\quad= \frac{1}{b^2}\cdot\bbE_{\calK\leftarrow K,I_1,\hdots,I_b}\left\|\sum_{j=1}^b\left(\frac{d}{k}\cdot\select_{\calK}\left(\nabla F_{I_j}(\bx)\right) - \nabla F(\bx)\right)\right\|^2 \\ 
&\quad\stackrel{\text{(a)}}{=} \frac{1}{b^2}\cdot\bbE_{\calK\leftarrow K,I_1,\hdots,I_b}\sum_{j=1}^b\left\|\frac{d}{k}\cdot\select_{\calK}\left(\nabla F_{I_j}(\bx)\right) - \nabla F(\bx)\right\|^2 \\
&\qquad + \frac{1}{b^2}\cdot\bbE_{\calK\leftarrow K,I_1,\hdots,I_b}\sum_{j\neq k}\left\langle\frac{d}{k}\cdot\select_{\calK}\left(\nabla F_{I_j}(\bx)\right) - \nabla F(\bx),\  \frac{d}{k}\cdot\select_{\calK}\left(\nabla F_{I_k}(\bx)\right) - \nabla F(\bx)\right\rangle \\ 
&\quad= \frac{1}{b^2}\cdot\sum_{j=1}^b\bbE_{\calK\leftarrow K,I_j}\left\|\frac{d}{k}\cdot\select_{\calK}\left(\nabla F_{I_j}(\bx)\right) - \nabla F(\bx)\right\|^2 \\
&\quad\stackrel{\text{(b)}}{\leq} \frac{1}{b^2}\cdot\sum_{j=1}^b\bbE_{\calK\leftarrow K,I_j}\left\|\frac{d}{k}\cdot\select_{\calK}\left(\nabla F_{I_j}(\bx)\right)\right\|^2 \\
&\quad \stackrel{\text{(c)}}{=} \frac{1}{b^2}\cdot\frac{d^2}{k^2}\frac{k}{d}\sum_{j=1}^b\bbE_{I_j}\left\|\nabla F_{I_j}(\bx)\right\|^2 \\
&\quad\stackrel{\text{(d)}}{\leq} \frac{d}{k}\cdot\frac{G^2}{b}.
\end{align*}
The second term on the RHS of (a) is equal to 0, which follows from the fact that $I_j$ and $I_k$ for $j\neq k$ are independent random variables and that $\bbE_{\calK\leftarrow K,I_j}\left[\frac{d}{k}\cdot\select_{\calK}(\nabla F_{I_j}(\bx))\right] = \nabla F(\bx)$. 
In (b) we used the inequality that for any random vector $X$, we have $\bbE\|X-\bbE[X]\|^2 \leq \bbE\|X\|^2$.
In (c) we used that for any vector $\bv\in\R^d$, we have $\bbE_{\randk}\|\randk(\bv)\|^2 = \frac{k}{d}\|\bv\|^2$. 
In (d) we used the second moment bound assumption $\bbE_i\|\nabla F_i(\bx)\|^2 \leq G^2$.
This completes the proof of \Lemmaref{randk}.
\end{proof}

\section{Omitted Details from \Sectionref{statistical-model}}\label{app:statistical-model}
In this section, we bound $\left\| \nabla \bar{f}_r(\bx) - \nabla \mu_r(\bx)\right\|$ under both the sub-exponential and sub-Gaussian gradient distributional assumptions.
First we give some definitions.
\begin{definition}[Sub-exponential distribution]
A random variable $Z$ with mean $\mu=\bbE[Z]$ is {\em sub-exponential} if there are non-negative parameters $(\nu,\alpha)$ such that
\begin{align*}
\bbE\left[\exp\left(\lambda(Z-\mu)\right)\right] \leq \exp\left(\lambda^2\nu^2/2\right), \qquad \forall |\lambda|<\frac{1}{\alpha}.
\end{align*}
A random vector $Z$ with mean $\mu=\bbE[Z]$ is {\em sub-exponential} if its projection on every unit vector is sub-exponential, i.e., there are non-negative parameters $(\nu,\alpha)$ such that
\begin{align*}
\sup_{\bv\in\R^d:\|\bv\|=1}\bbE\left[\exp\(\lambda\langle Z-\mu, \bv \rangle\)\right] \leq \exp\(\lambda^2\nu^2/2\), \qquad \forall |\lambda|<\frac{1}{\alpha}.
\end{align*}
\end{definition}
Now we state a concentration inequality for sums of independent sub-exponential random variables.
\begin{fact}[Sub-exponential concentration inequality]\label{fact:sub-exp-concentration}
Suppose $X_1,X_2,\hdots,X_n$ are independent random variables, where for every $i\in[n]$, $X_i$ is sub-exponential with parameters $(\nu_i,\alpha_i)$ and mean $\mu_i$. Then $\sum_{i=1}^nX_i$ is sub-exponential with parameters $(\nu,\alpha)$, where $\nu^2=\sum_{i=1}^n\nu_i^2$ and $\alpha=\max_{1\leq i\leq n}\alpha_i$.
Moreover, we have
\begin{equation}\label{sub-exp-concentration}
\Pr\left[\sum_{i=1}^n(X_i-\mu_i) \geq t\right] \leq \exp\left(-\frac{1}{2}\min\left\{\frac{t^2}{\nu^2},\frac{t}{\alpha}\right\}\right), \qquad \forall t \geq 0
\end{equation}
\end{fact}
\begin{definition}[Sub-Gaussian distribution]
A random variable $Z$ with mean $\mu=\bbE[Z]$ is {\em sub-Gaussian} if there is a non-negative parameter $\sigmag$ such that
\begin{align*}
\bbE\left[\exp\left(\lambda(Z-\mu)\right)\right] \leq \exp\left(\lambda^2\sigmag^2/2\right), \qquad \forall \lambda\in\R.
\end{align*}
A random vector $Z$ with mean $\mu=\bbE[Z]$ is {\em sub-Gaussian} if its projection on every unit vector is sub-Gaussian, i.e., there is a non-negative parameter $\sigmag$ such that
\begin{align*}
\sup_{\bv\in\R^d:\|\bv\|=1}\bbE\left[\exp\(\lambda\langle Z-\mu, \bv \rangle\)\right] \leq \exp\(\lambda^2\sigmag^2/2\), \qquad \forall \lambda\in\R.
\end{align*}
\end{definition}
Now we state a concentration inequality for sums of independent sub-Gaussian random variables.
\begin{fact}[Sub-Gaussian concentration inequality]\label{fact:sub-gauss-concentration}
Suppose $X_1,X_2,\hdots,X_n$ are independent random variables, where for every $i\in[n]$, $X_i$ is sub-Gaussian with parameter $\sigma_i>0$ and mean $\mu_i$. Then $\sum_{i=1}^nX_i$ is sub-Gaussian with parameter $\sigmag=\sqrt{\sum_{i=1}^n\sigma_i^2}$. Moreover, we have
\begin{equation}\label{sub-gauss-concentration}
\Pr\left[\sum_{i=1}^n(X_i-\mu_i) \geq t\right] \leq \exp\left(-t^2/2\sigmag^2\right), \qquad \forall t \geq 0.
\end{equation}
\end{fact}

Let $D=\max\{\|\bx-\bx'\|:\bx,\bx'\in\calC\}$ be the diameter of $\calC$. 
Note that $\calC$ is contained in $\calB_{\small D/2}^d$, which is the Euclidean ball of radius $\frac{D}{2}$ in $d$ dimensions that contains $\calC$.
Note that $D=\Omega(\sqrt{d})$, and we assume that $D$ can grow at most polynomially in $d$.

Now we state two lemmas (which will be used to prove \Theoremref{concrete-kappa-bound_stat}),
each of which uniformly bounds $\left\| \nabla \bar{f}_r(\bx) - \nabla \mu_r(\bx)\right\|$ over all $\bx\in\calC$ under different distributional assumptions on gradients. We prove these one by one in subsequent subsections.
\begin{lemma}[Sub-exponential gradients]\label{lem:kappa-bound_sub-exp_all-x}
Suppose \Assumptionref{sub-exp-grad} holds. Take an arbitrary $r\in[R]$.
Let $n_r\in\mathbb{N}$ be sufficiently large such that $n_r=\Omega\(d\log(n_rd)\)$. 
Then, with probability at least $1-\frac{1}{(1+n_rLD)^d}$, we have
\begin{align}\label{kappa-bound_sub-exp_all-x}
\left\| \nabla \bar{f}_r(\bx) - \nabla \mu_r(\bx)\right\| \leq 3\nu\sqrt{\frac{8d\log(1+n_rLD)}{n_r}}, \qquad \forall x\in\calC.
\end{align}
\end{lemma}
\begin{lemma}[Sub-Gaussian gradients]\label{lem:kappa-bound_sub-gauss_all-x}
Suppose \Assumptionref{sub-gauss-grad} holds. Take an arbitrary $r\in[R]$.
For any $n_r\in\mathbb{N}$, with probability at least $1-\frac{1}{(1+n_rLD)^d}$, we have
\begin{align}\label{kappa-bound_sub-gauss_all-x}
\left\| \nabla \bar{f}_r(\bx) - \nabla \mu_r(\bx)\right\| \leq 3\sigmag\sqrt{\frac{8d\log(1+n_rLD)}{n_r}}, \qquad \forall x\in\calC.
\end{align}
\end{lemma}
\begin{proof}[Proof of \Theoremref{concrete-kappa-bound_stat}]
In order to prove \Theoremref{concrete-kappa-bound_stat}, we need to show two bounds, one (stated in \eqref{kappa-bound_sub-exp-grad}) under the sub-exponential gradient assumption, and the other (stated in \eqref{kappa-bound_sub-gauss-grad}) under the sub-Gaussian assumption. 
We can show \eqref{kappa-bound_sub-exp-grad} using \Lemmaref{kappa-bound_sub-exp_all-x} and \eqref{kappa-bound_sub-gauss-grad} using \Lemmaref{kappa-bound_sub-gauss_all-x}.
Here we only show \eqref{kappa-bound_sub-exp-grad}; and \eqref{kappa-bound_sub-gauss-grad} can be shown similarly.

Using \Assumptionref{uniform-bound-mean-heterogeneous} (i.e., $\left\|\nabla \mu_r(\bx) - \nabla \mu(\bx)\right\| \leq \kappa_{\text{mean}}, \forall \bx\in\calC$) in \eqref{reduction-local-statistical} gives
\begin{align}
\left\|\nabla\bar{f}_r(\bx) - \nabla\bar{f}(\bx)\right\| 
\leq \left\|\nabla\bar{f}_r(\bx) - \nabla \mu_r(\bx)\right\| + \kappa_{\text{mean}} + \frac{1}{R}\sum_{r=1}^R\left\|\nabla\bar{f}_r(\bx) - \nabla \mu_r(\bx)\right\|.  \label{reduction-local-statistical_app}
\end{align}
Note that \eqref{kappa-bound_sub-exp_all-x} holds for any fixed worker $r\in[R]$. By the union bound, we have that with probability at least $1-\frac{R}{(1+n_rLD)^d}$, for every $r\in[R]$, we have $\left\| \nabla \bar{f}_r(\bx) - \nabla \mu_r(\bx)\right\| \leq 3\nu\sqrt{\frac{8d\log(1+n_rLD)}{n_r}}, \forall \bx\in\calC$.

Let $n_r=n,\forall r\in[R]$. 
Using these in \eqref{reduction-local-statistical_app}, we get that with probability at least $1-\frac{R}{(1+n_rLD)^d}$, for every worker $r\in[R]$, we have 
$\left\|\nabla\bar{f}_r(\bx) - \nabla\bar{f}(\bx)\right\|  \leq \kappa_{\text{mean}} + \calO\(\sqrt{\frac{d\log(nd)}{n}}\), \forall \bx\in\calC,$ which proves \eqref{kappa-bound_sub-exp-grad}.
This completes the proof of \Theoremref{concrete-kappa-bound_stat}.
\end{proof}

\subsection{Proof of \Lemmaref{kappa-bound_sub-exp_all-x} (sub-exponential gradients)}\label{app:kappa-bound_sub-exp}
We prove \Lemmaref{kappa-bound_sub-exp_all-x} with the help of the following result, which holds for any fixed $\bx\in\calC$. Then we extend this bound to all $\bx\in\calC$ using an $\eps$-net argument.
These are standard calculations and have appeared in literature \cite{ChenSX_Byz17,Bartlett-Byz_nonconvex19}. 
\begin{lemma}\label{lem:eps-net-unit-ball_sub-exp}
Suppose \Assumptionref{sub-exp-grad} holds. 
Take an arbitrary $r\in[R]$. For any $\delta\in(0,1)$ and $n_r\in\mathbb{N}$, define $\Delta = \sqrt{2}\nu\sqrt{\frac{d\log5+\log(1/\delta)}{n_r}}$. If $n_r$ is such that $\Delta \leq \frac{\nu^2}{\alpha}$, then, for any fixed $\bx\in\calC$, with probability at least $1-\delta$, we have
\begin{align}
\left\| \nabla \bar{f}_r(\bx) - \nabla \mu_r(\bx)\right\| \leq 2\sqrt{2}\nu\sqrt{\frac{d\log5 + \log(1/\delta)}{n_r}},
\end{align}
where randomness is due to the sub-exponential distribution of local gradients.
\end{lemma}
\begin{proof}
Let $\calB^d=\{\bv\in\R^d:\|\bv\|\leq1\}$. Let $\calV=\{\bv_1,\bv_2,\hdots,\bv_{N_{1/2}}\}$ denote an $\frac{1}{2}$-net of $\calB^d$, which implies that for every $\bv\in\calB^d$, there exists a $\bv'\in\calV$ such that $\|\bv-\bv'\| \leq \frac{1}{2}$. We have from \cite[Lemma 5.2]{Vershynin-rand-mat10} that $N_{1/2} = |\calV| \leq 5^d$.

Fix an arbitrary $\bx\in\calC$. Note that there exists a $\bv^*\in \calB^{d}$ (namely, $\bv^* = \frac{\nabla \bar{f}_r(\bx) - \nabla \mu_r(\bx)}{\|\nabla \bar{f}_r(\bx) - \nabla \mu_r(\bx)\|}$) such that $\left\| \nabla \bar{f}_r(\bx) - \nabla \mu_r(\bx) \right\| = \left\langle \nabla \bar{f}_r(\bx) - \nabla \mu_r(\bx), \bv^* \right\rangle$. 
By the property of $\calV$, there exists an index $i^*\in[N_{1/2}]$ such that $\|\bv^*-\bv_{i^*}\|\leq\frac{1}{2}$.
Now we bound $\left\| \nabla \bar{f}_r(\bx) - \nabla \mu_r(\bx) \right\|$.
\begin{align}
\left\| \nabla \bar{f}_r(\bx) - \nabla \mu_r(\bx) \right\| &= \left\langle \nabla \bar{f}_r(\bx) - \nabla \mu_r(\bx), \bv^* \right\rangle \nonumber \\
&= \left\langle \nabla \bar{f}_r(\bx) - \nabla \mu_r(\bx), \bv_{i^*} \right\rangle + \left\langle \nabla \bar{f}_r(\bx) - \nabla \mu_r(\bx), \bv^* - \bv_{i^*} \right\rangle \nonumber \\
&\leq \left\langle \nabla \bar{f}_r(\bx) - \nabla \mu_r(\bx), \bv_{i^*} \right\rangle + \left\|\nabla \bar{f}_r(\bx) - \nabla \mu_r(\bx)\right\|\left\| \bv^* - \bv_{i^*}\right\|  \nonumber \\
&\leq \left\langle \nabla \bar{f}_r(\bx) - \nabla \mu_r(\bx), \bv_{i^*} \right\rangle + \frac{1}{2}\left\|\nabla \bar{f}_r(\bx) - \nabla \mu_r(\bx)\right\|  \nonumber \\
&\leq \max_{\bv\in\calV}\left\langle \nabla \bar{f}_r(\bx) - \nabla \mu_r(\bx), \bv \right\rangle + \frac{1}{2}\left\|\nabla \bar{f}_r(\bx) - \nabla \mu_r(\bx)\right\|  \nonumber
\end{align}
By collecting similar terms together, we get
\begin{align}
\left\| \nabla \bar{f}_r(\bx) - \nabla \mu_r(\bx) \right\| \leq 2\max_{\bv\in\calV}\left\langle \nabla \bar{f}_r(\bx) - \nabla \mu_r(\bx), \bv \right\rangle \label{eps-net-norm-bound-0}
\end{align}
Note that the RHS of \eqref{eps-net-norm-bound-0} is a non-negative number (because LHS is).
Note also that, since $\calV\subset \calB^d$, for every $\bv\in\calV$, we have $\|\bv\|\leq 1$.
This implies that $\max_{\bv\in\calV}\left\langle \nabla \bar{f}_r(\bx) - \nabla \mu_r(\bx), \bv \right\rangle \leq \max_{\bv\in\calV}\left\langle \nabla \bar{f}_r(\bx) - \nabla \mu_r(\bx), \frac{\bv}{\|\bv\|} \right\rangle$. Using this in \eqref{eps-net-norm-bound-0}, we get
\begin{align}
\left\| \nabla \bar{f}_r(\bx) - \nabla \mu_r(\bx) \right\| \leq 2\max_{\bv\in\calV}\left\langle \nabla \bar{f}_r(\bx) - \nabla \mu_r(\bx), \frac{\bv}{\|\bv\|} \right\rangle. \label{eps-net-norm-bound}
\end{align}
Fix any $\bv\in \calV$. It follows from \Assumptionref{sub-exp-grad} that $\left\langle \nabla f_r(\bz,\bx) - \nabla \mu_r(\bx), \frac{\bv}{\|\bv\|}\right\rangle$, where $\bz\sim q_r$, is a sub-exponential random variable (with mean zero) with parameters $(\nu,\alpha)$. From \Factref{sub-exp-concentration} (stated on \Pageref{sub-exp-concentration}), we have that $\sum_{i=1}^{n_r}\left\langle \nabla f_r(\bz_{r,i},\bx) - \nabla \mu_r(\bx), \frac{\bv}{\|\bv\|}\right\rangle$ (where $\bz_{r,i}\sim q_r, i\in[n_r]$ are i.i.d.) is a sub-exponential random variable with parameters $(\sqrt{n_r}\nu,\alpha)$.

Now, apply the concentration bound from \eqref{sub-exp-concentration} with $t=n_r\Delta$.
Substituting this and the parameters $(\sqrt{n_r}\nu,\alpha)$, the bound becomes $\exp(-\frac{1}{2}\min\{\frac{n_r^2\Delta^2}{n_r\nu^2},\frac{n_r\Delta}{\alpha}\}) \stackrel{\text{(a)}}{=} \exp(-\frac{1}{2}\frac{n_r\Delta^2}{\nu^2})$, where (a) follows because $\Delta \leq \frac{\nu^2}{\alpha}$. This gives
\begin{align}\label{sub-exp-grad-concen-1}
\Pr\left[\sum_{i=1}^{n_r}\left\langle \nabla f_r(\bz_{r,i},\bx) - \nabla \mu_r(\bx), \frac{\bv}{\|\bv\|} \right\rangle \geq n_r\Delta\right] 
\leq \exp\left(-\frac{n_r\Delta^2}{2\nu^2}\right). 
\end{align}

Note that $\sum_{i=1}^{n_r}\left\langle \nabla f_r(\bz_{r,i},\bx) - \nabla \mu_r(\bx), \frac{\bv}{\|\bv\|} \right\rangle = n_r\left\langle \nabla \bar{f}_r(\bx) - \nabla \mu_r(\bx), \frac{\bv}{\|\bv\|} \right\rangle$. Using this in \eqref{sub-exp-grad-concen-1} yields
\begin{align}\label{sub-exp-grad-concen-2}
\Pr\left[ \left\langle \nabla \bar{f}_r(\bx) - \nabla \mu_r(\bx), \frac{\bv}{\|\bv\|} \right\rangle \geq \Delta \right] \leq \exp\left(-\frac{n_r\Delta^2}{2\nu^2}\right)
\end{align}
This implies that
\begin{align}
\Pr\left[ \max_{\bv\in\calV}\left\langle \nabla \bar{f}_r(\bx) - \nabla \mu_r(\bx), \frac{\bv}{\|\bv\|} \right\rangle \geq \Delta \right] 
&\leq \sum_{\bv\in\calV}\Pr\left[ \left\langle \nabla \bar{f}_r(\bx) - \nabla \mu_r(\bx), \frac{\bv}{\|\bv\|} \right\rangle \geq \Delta \right] \nonumber \\
&\leq |\calV|\exp\left(-\frac{n_r\Delta^2}{2\nu^2}\right) \leq 5^d \exp\left(-\frac{n_r\Delta^2}{2\nu^2}\right)\nonumber \\
&= \exp\left(-\frac{n_r\Delta^2}{2\nu^2} + d\log5\right) \label{sub-exp-grad-concen-3}
\end{align}
Together with \eqref{eps-net-norm-bound}, which implies that $$\Pr\left[\left\| \nabla \bar{f}_r(\bx) - \nabla \mu_r(\bx) \right\| \geq t\right] \leq \Pr\left[ 2\max_{\bv\in\calV}\left\langle \nabla \bar{f}_r(\bx) - \nabla \mu_r(\bx), \frac{\bv}{\|\bv\|} \right\rangle\geq t\right]$$ holds for every $t>0$, \eqref{sub-exp-grad-concen-3} gives
\begin{align}
\Pr\left[\left\| \nabla \bar{f}_r(\bx) - \nabla \mu_r(\bx) \right\| \geq 2\Delta\right] \leq \exp\left({- \frac{n_r\Delta^2}{2\nu^2}} + d\log5 \right) \leq \delta, \label{sub-exp-grad-concen-4}
\end{align}
where in the last inequality we used $\Delta = \sqrt{2}\nu\sqrt{\frac{d\log5+\log(1/\delta)}{n_r}}$.

This completes the proof of \Lemmaref{eps-net-unit-ball_sub-exp}.
\end{proof}
\begin{proof}[Proof of \Lemmaref{kappa-bound_sub-exp_all-x}]
We have from \Lemmaref{eps-net-unit-ball_sub-exp} that for each fixed $\bx\in\calC$, with probability at least $1-\delta$, we have 
\begin{align}\label{concen-bound_fixed-x}
\left\| \nabla \bar{f}_r(\bx) - \nabla \mu_r(\bx)\right\| \leq 2\nu\sqrt{\frac{2d\log5 + 2\log(1/\delta)}{n_r}}.
\end{align} 
To extend this argument uniformly over the entire set $\calC$, we use another covering argument.
Recall that $D$ is the diameter of $\calC$.
Note that $\calC$ is contained in $\calB_{\small D/2}^d$, which is the Euclidean ball of radius $\frac{D}{2}$ in $d$ dimensions that contains $\calC$.
For some $\delta_0>0$, let $\calC_{\delta_0}=\{\bx_0,\bx_2,\hdots,\bx_{N_{\delta_0}}\}$ be the $\delta_0$-net of $\calC$. It follows from \cite[Lemma 5.2]{Vershynin-rand-mat10} that $N_{\delta_0}\leq \left(1+\frac{D}{\delta_0}\right)^d$.

Applying the union bound in \eqref{concen-bound_fixed-x}, we get that with probability at least $1-\delta$, we have for all $\bx_i\in\calC_{\delta_0}$,
\begin{align}
\left\| \nabla \bar{f}_r(\bx_i) - \nabla \mu_r(\bx_i)\right\| \leq 2\nu\sqrt{\frac{2d\log5 + 2\log\left(\frac{N_{\delta_0}}{\delta}\right)}{n_r}}. \label{concen-bound_x-eps-net}
\end{align}
We want to bound $\left\| \nabla \bar{f}_r(\bx) - \nabla \mu_r(\bx)\right\|$ for all $\bx\in\calC$.
Take any $\bx\in\calC$. Since $\calC_{\delta_0}$ is a $\delta_0$-net of $\calC$, there exists an $\bx'\in\calC_{\delta_0}$ such that $\|\bx-\bx'\|\leq\delta_0$.
\begin{align}
\left\| \nabla \bar{f}_r(\bx) - \nabla \mu_r(\bx)\right\| &= \left\| \nabla \bar{f}_r(\bx) - \nabla \bar{f}_r(\bx') + \nabla \bar{f}_r(\bx') - \nabla \mu_r(\bx) + \nabla \mu_r(\bx') - \nabla \mu_r(\bx')\right\| \nonumber \\
&\leq \underbrace{\left\| \nabla \bar{f}_r(\bx) - \nabla \bar{f}_r(\bx') \right\|}_{=:\ T_1} + \underbrace{\left\| \nabla \mu_r(\bx) - \nabla \mu_r(\bx')\right\|}_{=:\ T_2} + \left\| \nabla \bar{f}_r(\bx') - \nabla \mu_r(\bx') \right\| \label{eps-net-final-1}
\end{align}
Now we bound each term on the RHS of \eqref{eps-net-final-1}.

\begin{align*}
T_1 &= \left\| \frac{1}{n_r}\sum_{i=1}^{n_r}\left(\nabla f_r(\bz_{r,i},\bx) - \nabla f_r(\bz_{r,i},\bx')\right) \right\| 
\leq \frac{1}{n_r}\sum_{i=1}^{n_r}\left\|\nabla f_r(\bz_{r,i},\bx) - \nabla f_r(\bz_{r,i},\bx') \right\|   \\
&\leq L\|\bx-\bx'\| \leq L\delta_0 \\
T_2 &= \left\| \bbE_{\bz\sim q_r}[\nabla f_r(\bz,\bx) - \nabla f_r(\bz,\bx;)]\right\| 
\leq \bbE_{\bz\sim q_r}\left\|\nabla f_r(\bz,\bx) - \nabla f_r(\bz,\bx;)\right\| \\
&\leq \bbE_{\bz\sim q_r}L\|\bx-\bx'\| \leq L\delta_0
\end{align*}
Substituting the above bounds on $T_1,T_2$ in \eqref{eps-net-final-1} and bounding the third term of \eqref{eps-net-final-1} using \eqref{concen-bound_x-eps-net} gives
\begin{align}
\left\| \nabla \bar{f}_r(\bx) - \nabla \mu_r(\bx)\right\| \leq  2L\delta_0 + 2\nu\sqrt{\frac{2d\log5 + 2\log\left(\frac{N_{\delta_0}}{\delta}\right)}{n_r}}. \label{eps-net-final-2}
\end{align}
Note that $N_{\delta_0}\leq\left(1+\frac{D}{\delta_0}\right)^d$. Take $\delta=1/\left(1+\frac{D}{\delta_0}\right)^d$.
If we take $\delta_0=\frac{1}{n_rL}$, which implies $\delta=\frac{1}{(1+n_rLD)^d}$, we would get $2d\log5 + 2\log\left(\frac{N_{\delta_0}}{\delta}\right) \leq 4d + 4d\log(1+n_rLD) \leq 8d\log(1+n_rLD)$.
Substituting these in above gives
\begin{align}
\left\| \nabla \bar{f}_r(\bx) - \nabla \mu_r(\bx)\right\| \leq \frac{2}{n_r} + \frac{2\nu}{\sqrt{n_r}}\sqrt{8d\log(1+n_rLD)}.
\end{align}
When $n_r\geq\frac{1}{2\nu^2d\log(1+n_rLD)}$ (which is a very small number less than 1), with probability at least $1-\frac{1}{(1+n_rLD)^d}$, we have
\begin{align}
\left\| \nabla \bar{f}_r(\bx) - \nabla \mu_r(\bx)\right\| \leq 3\nu\sqrt{\frac{8d\log(1+n_rLD)}{n_r}}, \qquad \forall x\in\calC.
\end{align}

\paragraph{Lower bound on $n_r$.}
Note that \Lemmaref{eps-net-unit-ball_sub-exp} requires $\Delta \leq \frac{\nu^2}{\alpha}$, where $\Delta = \sqrt{2}\nu\sqrt{\frac{d\log5+\log(1/\delta)}{n_r}}$. Substituting the value of $\delta=\frac{1}{(1+n_rLD)^d}$ gives $n_r\geq\frac{2\alpha^2}{\nu^2}\(d\log5 + d\log(1+n_rLD)\)$, which is $\Omega(d\log(n_rLD))$ for constant $\alpha,\nu$. 
Treating the smoothness parameter $L$ a constant, we get $n_r=\Omega(d\log(n_rd))$ to be requirement on the sample size at the $r$'th worker for the bound in \Lemmaref{kappa-bound_sub-exp_all-x} to hold.

This completes the proof of \Lemmaref{kappa-bound_sub-exp_all-x}.
\end{proof}

\subsection{Proof of \Lemmaref{kappa-bound_sub-gauss_all-x} (sub-Gaussian gradients)}\label{app:kappa-bound_sub-gauss}

We prove \Lemmaref{kappa-bound_sub-gauss_all-x} with the help of the following result, which holds for any fixed $\bx\in\calC$. 
\begin{lemma}\label{lem:eps-net-unit-ball_sub-gauss}
Suppose \Assumptionref{sub-gauss-grad} holds. 
Take an arbitrary $r\in[R]$. For any $\delta\in(0,1)$ and $n_r\in\mathbb{N}$, with probability at least $1-\delta$, we have for any fixed $\bx\in\calC$:
\begin{align}
\left\| \nabla \bar{f}_r(\bx) - \nabla \mu_r(\bx)\right\| \leq 2\sqrt{2}\sigmag\sqrt{\frac{d\log5 + \log(1/\delta)}{n_r}},
\end{align}
where randomness is due to the sub-Gaussian distribution of local gradients.
\end{lemma}
\begin{proof}
Follow the proof of \Lemmaref{eps-net-unit-ball_sub-exp} exactly until \eqref{eps-net-norm-bound}.
Then instead of the sub-exponential assumption, use the sub-Gaussian assumption (\Assumptionref{sub-gauss-grad}) on local gradients. 
Then apply the concentration bound from \eqref{sub-gauss-concentration} with $t=n_r\Delta$. This gives that for any fixed $\bv\in\calV$ and any $\Delta\geq0$, we have
\begin{align}\label{sub-gauss-grad-concen-2}
\Pr\left[ \left\langle \nabla \bar{f}_r(\bx) - \nabla \mu_r(\bx), \frac{\bv}{\|\bv\|} \right\rangle \geq \Delta \right] \leq \exp\left(-\frac{n_r\Delta^2}{2\sigmag^2}\right).
\end{align}
Now following the proof of \Lemmaref{eps-net-unit-ball_sub-exp} from \eqref{sub-exp-grad-concen-2} to \eqref{sub-exp-grad-concen-4} gives
\begin{align}
\Pr\left[\left\| \nabla \bar{f}_r(\bx) - \nabla \mu_r(\bx) \right\| \geq 2\Delta\right] \leq \exp\left({- \frac{n_r\Delta^2}{2\sigmag^2}} + d\log5 \right) \leq \delta, \label{sub-gauss-grad-concen-4}
\end{align}
where in the last inequality we used $\Delta = \sqrt{2}\sigmag\sqrt{\frac{d\log5+\log(1/\delta)}{n_r}}$.
\end{proof}
We can extend the bound from \Lemmaref{eps-net-unit-ball_sub-gauss} to all $\bx\in\calC$ (and prove \Lemmaref{kappa-bound_sub-gauss_all-x}) using an $\eps$-net argument exactly in the same way as used in the proof of \Lemmaref{kappa-bound_sub-exp_all-x}.
So, to avoid repetition, we do not show this extension here.

\subsection{Bounding the local variances}\label{app:variance-bound_stat}
In \Subsectionref{variance-bound_stat}, we showed that in order to bound $\bbE_{i\in_U[n_r]}\left\|\nabla f_r(\bz_{r,i},\bx) - \nabla \bar{f}_r(\bx)\right\|^2$ uniformly over all $\bx\in\calC$, it suffices to bound $\left\|\nabla f_r(\bz,\bx) - \nabla \mu_r(\bx)\right\|$ for a random $\bz\sim q_r$ uniformly over all $\bx\in\calC$.

\paragraph{Bounding $\left\|\nabla f_r(\bz,\bx) - \nabla \mu_r(\bx)\right\|$.}
To bound this, we need sub-Gaussian assumption on local gradients (we can also bound this using sub-exponential assumption, but that will give a bound that scales as $\widetilde{\Omega}(d)$ as opposed to $\widetilde{\Omega}(\sqrt{d})$).
Note that \Lemmaref{kappa-bound_sub-gauss_all-x} holds for any $n_r\in\mathbb{N}$. In particular, it also holds for $n_r=1$. So, under \Assumptionref{sub-gauss-grad}, with probability at least $1-\frac{1}{(1+n_rLD)^d}$, we have
\begin{align}\label{single-sample_variance_stat}
\left\| \nabla f_r(\bz,\bx) - \nabla \mu_r(\bx)\right\| \leq 3\sigmag\sqrt{8d\log(1+LD)}, \qquad \forall x\in\calC,
\end{align}
where $\bz\sim q_r$, and probability is over the randomness due to the sub-Gaussian distribution of local gradients.
So, with probability at least $1-\frac{1}{(1+n_rLD)^d}$, we have
\begin{align}\label{single-sample_variance_stat-1}
\bbE_{i\in_U[n_r]}\left\|\nabla f_r(\bz_{r,i},\bx) - \nabla \bar{f}_r(\bx)\right\|^2 
&\leq 288\sigmag^2d\log(1+LD), \qquad \forall \bx\in\calC.
\end{align}
Note that \eqref{single-sample_variance_stat-1} holds for a fixed worker $r\in[R]$. By taking the union bound over all workers $r\in[R]$ proves \Theoremref{variance-bound_stat}.

\section{Robust Mean Estimation}\label{app:robust-mean-estimation}
In this section, we present the robust mean estimation algorithm of \cite{Resilience_SCV18} that we use to filter-out the corrupt gradients and estimate the average of the good gradients.
The procedure is presented in \Algorithmref{robust-grad-est}, which was used in the second part of \Theoremref{gradient-estimator}.

\begin{algorithm}[t]
   \caption{Robust Gradient Estimation ({\sc RGE}) \cite{Resilience_SCV18}}\label{algo:robust-grad-est}
\begin{algorithmic}[1]
   \STATE {\bf Initialize.} $c_i:=1$ for all $i\in[R]$, $\alpha:=(1-\tilde{\eps})\geq\nicefrac{3}{4}$, $\calA:=\{1,2,\hdots,R\}$; $\bG:=[\bg_1,\ \bg_2,\ \hdots,\ \bg_R]\in\R^{d\times R}$.
   \WHILE{true}
   \STATE Let $\bW^*\in\R^{|\calA|\times|\calA|}$ and $\bY^*\in\R^{d\times d}$ be the minimizer/maximizer of the saddle point problem:
   \begin{align}\label{saddle-point-opt}
   \displaystyle \max_{\substack{\bY\succeq\bzero, \\ \text{tr}(\bY)\leq 1}} \ \min_{\substack{0\leq W_{ji}\leq \frac{4-\alpha}{\alpha(2+\alpha)R}, \\ \sum_{j\in\calA}W_{ji}=1, \forall i\in\calA}} \Phi(\bW,\bY),
   \end{align}
   where the cost function $\Phi(\bW,\bY)$ is defined as   
    \begin{align}\label{cost-function_defn}
   \Phi(\bW,\bY) := \sum_{i\in\calA}c_i(\bg_i-\bG_{\calA}\bw_i)^T\bY(\bg_i-\bG_{\calA}\bw_i),
   \end{align}
      To avoid cluttered notation, we index the $|\calA|$ rows/columns of $\bW$ by the elements of $\calA$; $\bG_{\calA}$ denotes the restriction of $\bG$ to the columns in $\calA$; for $i\in\calA$, $\bw_i$ denotes the column of $\bW$ indexed by $i$.
   \STATE For $i\in\calA$, let 
   \begin{align}\label{tau-defn}
   \tau_i = (\bg_i-\bG_{\calA}\bw_i^*)^T\bY^*(\bg_i-\bG_{\calA}\bw_i^*)
   \end{align}
   \IF{$\sum_{i\in\calA}c_i\tau_i > 4R\sigma_0^2$}
   \STATE For $i\in\calA$, $c_i\leftarrow \left(1-\frac{\tau_i}{\tau_{\max}}\right)c_i$, where $\tau_{\max}=\max_{j\in\calA}\tau_{j}$.
   \STATE For all $i$ with $c_i<\frac{1}{2}$, remove $i$ from $\calA$.
   \ELSE
   \STATE Break {\bf while}-loop
   \ENDIF
   \ENDWHILE
   \STATE {\bf return} $\btg = \frac{1}{|\calA|}\sum_{i\in\calA}\bg_i$.
\end{algorithmic}
\end{algorithm}
\subsection{Intuition behind \Algorithmref{robust-grad-est}}\label{app:intuition_algo}
First we argue that \eqref{saddle-point-opt} is a saddle-point optimization problem. This follows from the Von Neumann Minimax theorem, for which we can verify that {\sf (i)} $\Phi(\bW,\bY)$ is a continuous function of its arguments,
{\sf (ii)} $\Phi(\bW,\bY)$ is convex in $\bW$ (for a fixed $\bY$) and concave in $\bY$ (for a fixed $\bW$), and 
{\sf (iii)} the sets over which the maximum/minimum are taken in \eqref{saddle-point-opt} are compact convex sets. 
See the proof of \Claimref{convex-concave} in \Appendixref{poly-time-lemma_proof} for more details. 
This implies that we can switch min and max in \eqref{saddle-point-opt}, i.e., $\max_{\bY}\min_{\bW}\Phi(\bW,\bY)=\min_{\bW}\max_{\bY}\Phi(\bW,\bY)$.

For a matrix $\bZ$, let $\|\bZ\|_2$ denote the matrix norm induced by the $\ell_2$-norm, which is equal to the largest singular value of $\bZ$.
We have the following identity (which we prove in \Claimref{matrix-norm_trace-maximization} in \Appendixref{poly-time-lemma_proof}): $\|\bG-\bG\bW\|_2^2 = \max_{\bY\succeq,\text{tr}(\bY)\leq1}\sum_{i=1}^R (\bg_i-\bG\bw_i)^T\bY(\bg_i-\bG\bw_i)$, where $\bG=[\bg_1,\ \bg_2,\ \hdots,\ \bg_R]\in\R^{d\times R}$.
Using this and the Minimax theorem, the optimization problem \eqref{saddle-point-opt} in the first iteration of the {\bf while}-loop (when $\calA=[R]$ and $c_i=1, \forall i\in\calA$) can be equivalently written as
   \begin{align}\label{main_opt-prob}
   &\text{minimize}_{\bW\in\R^R\times\R^R}\quad \|\bG-\bG\bW\|_2^2 \notag \\
   &\text{subject to}\quad 0\leq W_{ji}\leq \frac{4-\alpha}{\alpha(2+\alpha)R}, \forall i,j, \quad \sum_{j=1}^R W_{ji}=1, \forall i=1,\hdots,R.
   \end{align}

At the heart of \Algorithmref{robust-grad-est} is to solve \eqref{main_opt-prob}, which can be efficiently solved using the singular value decomposition (SVD); see \cite{Resilience_SCV18} for details.
The idea behind \eqref{main_opt-prob} is to represent each gradient vector as a weighted average of other $\frac{\alpha(2+\alpha)}{4-\alpha}R$ gradients, where $\alpha=1-\tilde{\eps}$. Since the set $\calS$ in the first part of \Theoremref{gradient-estimator} in \Sectionref{robust-grad-est} (also see \Lemmaref{poly-time_grad-est} in \Appendixref{poly-time-lemma_proof}) is sufficiently large and is well concentrated, the gradients in the set $\calS$ can be represented well using its empirical mean $\bg_{\calS}=\frac{1}{|\calS|}\sum_{i\in\calS}\bg_i$. So, gradients that cannot be represented as such must be outliers/corrupted and cannot lie in $\calS$, and we must remove them. For removing corrupted gradients, we do a soft removal, by maintaining weights $c_i$'s, which are all initialized to 1. When we find that the reconstruction error is large, we reduce the weights of the gradients with high reconstruction error (denoted by $\tau_i$). We remove a gradient when its associated $c_i$ becomes less than $\nicefrac{1}{2}$. We maintain an active set $\calA$ consisting of all the gradients with $c_i\geq\nicefrac{1}{2}$.

We can describe \Algorithmref{robust-grad-est} informally as follows -- for simplicity, we describe it w.r.t.\ the first iteration of the {\bf while}-loop only, where $\calA=[R]$ and $c_i=1,\forall i\in[R]$:
\begin{enumerate}
\item Solve the optimization problem \eqref{main_opt-prob} (with taking into account the weights $c_i$'s for the later iterations).
\item If the optimum value (which is equal to $\sum_{i}c_i\tau_i$ (line 5 of \Algorithmref{robust-grad-est})) is bigger than $4R\sigma_0^2$, then down-weight the gradient vectors in proportion to their corresponding reconstruction error -- larger the reconstruction error, lesser the weight it gets. If the weight goes below $\frac{1}{2}$ for any gradient vector, discard it.
\item If the optimum value is less than $4R\sigma_0^2$, then output the average of the remaining gradients.
\end{enumerate}

\subsection{Running time analysis of \Algorithmref{robust-grad-est}}\label{app:running-time_robust-grad-est}
As noted earlier, we can find the optimum value of \eqref{main_opt-prob} using SVD.
Let $\bW^*$ be a minimizer of \eqref{main_opt-prob}.
For down-weighting the $i$'th point, we need to compute the reconstruction error $\tau_i$, 
for which we have to find an optimum $\bY$ (satisfying $\bY\succeq\bzero$ and $\text{tr}(\bY)\leq1$) that maximizes $\sum_{i=1}^R (\bg_i-\bG\bw_i^*)^T\bY(\bg_i-\bG\bw_i^*)$. This can again be found using SVD, as $\bY$ that attains the maximum in this expression is equal to the $\bv\bv^T$, where $\bv$ is the principal eigenvector (i.e., eigenvector corresponding to the largest eigenvalue) of $(\bG-\bG\bW^*)(\bG-\bG\bW^*)^T$ (see the discussion on the time complexity of finding the optimal $\bY$ below).
This can be seen as follows:
\begin{align*}
\sum_{i=1}^R (\bg_i-\bG\bw_i^*)^T\bY(\bg_i-\bG\bw_i^*) &= \text{tr}\big((\bG-\bG\bW^*)^T\bY(\bG-\bG\bW^*)\big) \\
&\stackrel{\text{(a)}}{=} \text{tr}\big((\bG-\bG\bW^*)(\bG-\bG\bW^*)^T\bY\big) \\
&\stackrel{\text{(b)}}{\leq} \|(\bG-\bG\bW^*)(\bG-\bG\bW^*)^T\| \|\bY\|_* \\
&\stackrel{\text{(c)}}{\leq} \|(\bG-\bG\bW^*)(\bG-\bG\bW^*)^T\|
\end{align*}
Here (a) follows because $\text{tr}(\bA\bB)=\text{tr}(\bB\bA)$ holds for any two dimension compatible matrices $\bA,\bB$. (b) follows from $\text{tr}(\bA\bB)\leq\|\bA\|\|\bB\|_*$ (which is proved in \Claimref{holder} in \Appendixref{poly-time-lemma_proof}), where $\|\cdot\|$ is the matrix 2-norm and $\|\cdot\|_*$ is the nuclear norm, i.e., sum of the singular values. (c) holds because for a positive semi-definite matrix $\bY$, we have $\|\bY\|_*=\text{tr}(\bY)$, which is at most 1 (due to the constraint).
Note that, since $\bY\succeq\bzero$ with $\text{tr}(\bY)\leq1$, equality holds in (b) and (c) above when $\bY$ is a rank one matrix such that $\bY=\bv\bv^T$, where $\bv$ is the unit norm principal eigenvector of $(\bG-\bG\bW^*)(\bG-\bG\bW^*)^T$. See also the proof of \Claimref{holder} in \Appendixref{poly-time-lemma_proof}.\\

\paragraph{Time complexity of finding the optimal $\bY$.} Since $(\bG-\bG\bW^*)(\bG-\bG\bW^*)^T$ is a $d\times d$ matrix, the principal eigenvector $\bv$ can be found in $\calO(d^3)$ time using SVD, which may be computationally very expensive, as $d$ could be very large. Note, however, that the principal eigenvector of $(\bG-\bG\bW^*)(\bG-\bG\bW^*)^T$ is the same as the principal left singular vector (i.e., left singular vector corresponding to the largest singular value) of $\bG-\bG\bW^*$.\footnote{This can be seen as follows: By doing SVD, we can write $\bG-\bG\bW^*={\bf U\Sigma}{\bf V}^T$, where ${\bf U, V}$ are unitary matrices, and ${\bf \Sigma}$ is the diagonal matrix whose entries are the singular values of $\bG-\bG\bW^*$. Note that $(\bG-\bG\bW^*)(\bG-\bG\bW^*)^T={\bf U\Sigma}{\bf V}^T{\bf V}{\bf \Sigma}{\bf U}^T={\bf U\Sigma}^2{\bf U}^T$, which is the eigen-decomposition of $(\bG-\bG\bW^*)(\bG-\bG\bW^*)^T$. This implies that the eigenvector corresponding to the largest eigenvalue of $(\bG-\bG\bW^*)(\bG-\bG\bW^*)^T$ is the same as the left singular vector corresponding to the largest singular value of $\bG-\bG\bW^*$.} Since $\bG-\bG\bW^*$ is a $d\times R$ matrix, its principal left singular vector can be computed in $\calO(dR\min\{d,R\})$ time using SVD. \\

\paragraph{Polynomial-time complexity of \Algorithmref{robust-grad-est}.}
Observe that, in each iteration of the {\bf while}-loop of \Algorithmref{robust-grad-est}, we remove at least one $\bg_i$ (which corresponds to the $\tau_i$ for which $\tau_i=\tau_{\max}$) from $\calA$, which means that the {\bf while}-loop terminates after executing for at most $\calO(R)$ iterations.
Therefore, \Algorithmref{robust-grad-est} is a repeated (and at most $\calO(R)$) number of applications of SVD. Since SVD of a $d\times R$ matrix can be performed in $\calO(dR\min\{d,R\})$ time, \Algorithmref{robust-grad-est} runs in $\calO(dR^2\min\{d,R\})$ time; hence, runs in polynomial-time.

\section{Comprehensive Analysis of \Algorithmref{robust-grad-est}}\label{app:poly-time-lemma_proof}
In this section, we prove the second part of \Theoremref{gradient-estimator} by giving a comprehensive analysis of \Algorithmref{robust-grad-est}. For convenience, we state the second part of \Theoremref{gradient-estimator} as a separate lemma.
\begin{lemma}[Proposition 16 in \cite{Resilience_SCV18}]\label{lem:poly-time_grad-est}
Suppose we are given $R$ arbitrary vectors $\bg_1,\hdots,\bg_R\in\R^d$ with the promise that there exists a subset $\calS\subset[R]$ of these vectors such that  $|\calS|=(1-\tilde{\eps})R$ for some $\tilde{\eps}>0$ and $\calS$ satisfies $\lambda_{\max}\(\frac{1}{|\calS|}\sum_{i\in\calS} \(\bg_i - \bg_{\calS}\)\(\bg_i - \bg_{\calS}\)^T\) \leq \sigma_0^2$, where $\bg_{\calS}=\frac{1}{|\calS|}\sum_{i\in\calS}\bg_i$ denotes the sample mean of the vectors in $\calS$.
Then, if $\tilde{\eps}\leq\frac{1}{4}$, \Algorithmref{robust-grad-est} can find an estimate $\btg$ of $\bg_{\calS}$ in polynomial-time, such that $\|\btg-\bg_{\calS}\|\leq \calO(\sigma_0\sqrt{\tilde{\eps}})$.
\end{lemma}
The above lemma is from \cite{Resilience_SCV18} and has also been used in \cite{SuX_Byz19,Bartlett-Byz_nonconvex19} in the context of Byzantine-resilient {\em full batch} gradient descent, where workers have i.i.d.~homogeneous data drawn from a probability distribution. We apply it in the context of {\em mini-batch stochastic} gradient descent on {\em heterogeneous} data, where different workers may have different local datasets and we do not assume any probabilistic assumption on data. Our results are under the standard SGD assumption of bounded local variance \eqref{bounded_local-variance} and also under the bounded gradient dissimilarity bound \eqref{bounded_local-global}. For completeness, we provide a comprehensive proof of \Lemmaref{poly-time_grad-est} in this section. 
The proof details are along the lines of those from \cite{Resilience_SCV18,SuX_Byz19}.

To prove \Lemmaref{poly-time_grad-est}, we will use the techniques developed by \cite{Resilience_SCV18} and later used by \cite{SuX_Byz19}. Note that the results in \cite{Resilience_SCV18} are stated for a general norm $\|\cdot\|$, and we use those results in this paper by specializing them to the $\ell_2$-norm, but for simplicity, we write it as $\|\cdot\|$. 
First we define a notion called {\em $(\eps,\delta)$-resilience} for a given set $\calS$, which says that if $\calS$ is resilient around a point $\bmu$ (which need not be its mean), then dropping some elements from $\calS$ does not change the concentration of the resulting set around $\bmu$ by much.

\begin{definition}[Resilience, \cite{Resilience_SCV18}]\label{defn:resilience}
A set $\calS = \{\by_1,\by_2,\hdots,\by_m\}$ of $m$ points, each lying in $\R^d$, is {\em $(\eps,\delta)$-resilient} around a point $\bmu\in\R^d$, if every subset $\calT\subseteq \calS$ of cardinality at least $(1-\eps)m$ satisfies $\left\|\frac{1}{|\calT|}\sum_{\by\in\calT}(\by-\bmu)\right\| \leq \sigma$.
\end{definition}

The notion of resilience is useful for robust gradient estimation, because of the following reasons:
{\sf (i)} Since the subset $\calS$ in the statement of \Lemmaref{poly-time_grad-est} is such that the vectors in $\calS$ are concentrated around their mean, this implies using \Claimref{resilience_honest-grad} (on \Pageref{resilience_honest-grad}) that $\calS$ is resilient around its sample mean $\bg_{\calS}$.
{\sf (ii)} \Algorithmref{robust-grad-est} in fact finds a sufficiently large subset $\calA\subset[R]$ (see \eqref{auxilliary1-2} in \Lemmaref{auxilliary1}) and outputs its empirical mean $\btg$. We show in \Claimref{resilience_grad-est} (on \Pageref{resilience_grad-est}) that $\calA$ is also resilient around its mean $\btg$.
{\sf (iii)} Since both $\calS$ and $\calA$ are large, they must have a large intersection, and by resilience, their means must both be close to the mean of their intersection, and hence to each-other, by the triangle inequality; see \Claimref{final-claim} on \Pageref{final-claim}. This implies that $\btg$ must be close to $\bg_{\calS}$.

In the following discussion, we make these statements precise.

\begin{lemma}\label{lem:auxilliary1}
Let $\alpha := 1-\tilde{\eps}$.
The following invariants are maintained in every iteration of the {\bf while}-loop of \Algorithmref{robust-grad-est}:
\begin{align}
\sum_{i\in\calS\cap\calA}c_i\tau_i \ &\leq \ \alpha R\sigma_0^2 \label{auxilliary1-3} \\
\sum_{i\in\calS}(1-c_i) \ &\leq \ \frac{\alpha}{4}\sum_{i=1}^R(1-c_i) \label{auxilliary1-1} \\
|\calS\cap\calA| \ &\geq \ \frac{\alpha(2+\alpha)}{4-\alpha}R \label{auxilliary1-2}
\end{align}
\end{lemma}
In \Lemmaref{auxilliary1}, 
\eqref{auxilliary1-3} ensures that the reconstruction error in the non-corrupted gradients is small; \eqref{auxilliary1-1} ensures that we reduce the weights of non-corrupted data at least 4 times slower than overall; and
\eqref{auxilliary1-2} ensures that we do not remove too many non-corrupted points.
\begin{proof}[Proof of \Lemmaref{auxilliary1}]
Note that $c_i,\tau_i,\calA$ get updated throughout the execution of the {\bf while}-loop in \Algorithmref{robust-grad-est}. So, for convenience, we denote these quantities at the beginning of the $t$'th iteration by $c_i(t),\tau_i(t),\calA(t)$.
We prove \Lemmaref{auxilliary1} by induction on \eqref{auxilliary1-1} and \eqref{auxilliary1-2}. Note that induction is not on \eqref{auxilliary1-3}.

{\bf Base case, $t=1$:} Note that $\calA(1)=[R]$ and $c_i(1)=1,\forall i\in\calA$. So \eqref{auxilliary1-1} trivially holds. For \eqref{auxilliary1-2}, note that $\calS\cap\calA(1) = \calS$, and $|\calS|=\alpha R\geq \frac{\alpha(2+\alpha)}{4-\alpha}R$, where the first inequality follows from hypothesis of \Lemmaref{poly-time_grad-est} and the second inequality holds because $\frac{2+\alpha}{4-\alpha}\leq1$ for $\alpha\geq\nicefrac{3}{4}$. We show that \eqref{auxilliary1-3} holds for all $t\geq1$ in the induction step below.

{\bf Induction step:} Suppose the {\bf while}-loop has not terminated in the $t$'th iteration, which means that we update $c_i(t),\calA(t)$ to $c_i(t+1),\calA(t+1)$. 
This implies that the condition on line 5 of \Algorithmref{robust-grad-est} is satisfied, i.e.,
\begin{equation}\label{auxilliary1_interim1}
\sum_{i\in\calA(t)}c_i(t)\tau_i(t) > 4R\sigma_0^2.
\end{equation}
By induction hypothesis, \eqref{auxilliary1-1} and \eqref{auxilliary1-2} hold at the $t$'th iteration, i.e., 
\begin{align}
\sum_{i\in\calS}\big(1-c_i(t)\big) &\leq \frac{\alpha}{4}\sum_{i=1}^R\big(1-c_i(t)\big) \label{auxilliary1_interim2} \\
|\calS\cap\calA(t)| &\geq \frac{\alpha(2+\alpha)}{4-\alpha}R. \label{auxilliary1_interim3}
\end{align} 
Before proving that they also hold at the $(t+1)$'th iteration, we first show that \eqref{auxilliary1-3} holds at the $t$'th iteration for all $t\geq 1$.

\begin{claim}\label{claim:auxilliary1-3_proof}
For all $t\geq1$, we have $\sum_{i\in\calS\cap\calA(t)}c_i(t)\tau_i(t) \ \leq \ \alpha R\sigma_0^2$.
\end{claim}
\begin{proof}
First we define the sets over which we maximize/minimize the expression in \eqref{saddle-point-opt} in the $t$'th iteration:
\begin{align*}
\W(t) &:= \left\{\bW\in\R^{|\calA(t)|\times|\calA(t)|} : \sum_{j\in\calA(t)}W_{ji}=1, \forall i\in\calA(t) \text{ and } W_{ji}\in\left[0,\frac{4-\alpha}{\alpha(2+\alpha)R}\right]\right\}, \\
\Y(t) &:= \left\{\bY\in\R^{d\times d} : \bY\succeq\bzero, \text{tr}(\bY)\leq 1 \right\}.
\end{align*}
It is immediate from the definitions that $\W(t)$ and $\Y(t)$ are compact and convex sets.
Fix any $\bW\in\W(t),\bY\in\Y(t)$, and
consider the cost function $\Phi$ from \eqref{cost-function_defn} in \Algorithmref{robust-grad-est}:
\begin{align}\label{eq:cost-function}
\Phi(\bW,\bY) = \sum_{i\in\calA(t)}c_i(t)(\bg_i-\bG_{\calA(t)}\bw_i)^T\bY(\bg_i-\bG_{\calA(t)}\bw_i).
\end{align}
It can be verified that $\Phi$ is a continuous function of its arguments.
The following claim is proved at the end of this section.
\begin{claim}\label{claim:convex-concave}
$\Phi$ is convex in $\bW$ (for a fixed $\bY$) and concave in $\bY$ (for a fixed $\bW$).
\end{claim}

\noindent Given that \Claimref{convex-concave} holds, and $\W(t),\Y(t)$ are compact and convex sets, and $\Phi$ is a continuous function, it follows from the Von-Neumann Minimax theorem that 
\begin{align}
\min_{\bW\in\W(t)}\max_{\bY\in\Y(t)} \Phi\big(\bW,\bY\big) = \max_{\bY\in\Y(t)}\min_{\bW\in\W(t)} \Phi\big(\bW,\bY\big).
\end{align}
Furthermore, there is a saddle-point $\bW^*(t),\bY^*(t)$ (that achieves the optimum in \eqref{saddle-point-opt} in the $t$'th iteration) such that
\[\max_{\bY\in\Y(t)} \Phi\big(\bW^*(t),\bY\big) = \min_{\bW\in\W(t)} \Phi\big(\bW,\bY^*(t)\big) = \Phi\big(\bW^*(t),\bY^*(t)\big).\] 
In particular, we have $\bW^*(t) \in \arg\min_{\bW\in\W(t)} \Phi(\bW,\bY^*(t))$. 
Substituting the expression for $\Phi$ from \eqref{eq:cost-function}, we get
\begin{align}\label{auxilliary1-3_proof_interim1}
\bW^*(t) \ \in \ \arg\min_{\bW\in\W(t)} \sum_{i\in\calA(t)}c_i(t)\big(\bg_i-\bG_{\calA(t)}\bw_i\big)^T\bY^*(t)\big(\bg_i-\bG_{\calA(t)}\bw_i\big).
\end{align}
Observe that the summation in \eqref{auxilliary1-3_proof_interim1} has $|\calA(t)|$ terms, and $\bw_i$ (the $i$'th column of $\bW$) appears only in the $i$'th summand. This implies that we can separately optimize for each column of $\bW$ in \eqref{auxilliary1-3_proof_interim1}. 
Define \[\W'(t) := \left\{\bw\in\R^{|\calA(t)|} : \sum_{j\in\calA(t)}w_{j}=1 \text{ and } w_{j}\in\left[0,\frac{4-\alpha}{\alpha(2+\alpha)R}\right]\right\}\]
Let $\bw_i^*(t)$ denote the $i$'th column of $\bW^*(t)$. We can write
\begin{align}\label{auxilliary1-3_proof_interim2}
\bw_i^*(t) \ \in \ \arg\min_{\bw\in\W'(t)} \big(\bg_i-\bG_{\calA(t)}\bw\big)^T\bY^*(t)\big(\bg_i-\bG_{\calA(t)}\bw\big).
\end{align}
We have from \eqref{tau-defn} in \Algorithmref{robust-grad-est} that $\tau_i(t) = \big(\bg_i-\bG_{\calA(t)}\bw_i^*(t)\big)^T\bY^*(t)\big(\bg_i-\bG_{\calA(t)}\bw_i^*(t)\big)$.
Note that $\tau_i(t)$ is equal to the expression on the RHS of \eqref{auxilliary1-3_proof_interim2} evaluated with its minimizer $\bw_i^*(t)$.
Define $\btw(t)\in\R^{|\calA(t)|}$ such that $\btw(t)_j := \frac{\mathbbm{1}_{j\in\calS\cap\calA(t)}}{|\calS\cap\calA(t)|}$. Since $|\calS\cap\calA(t)| \geq \frac{\alpha(2+\alpha)}{4-\alpha}R$ (from the induction hypothesis \eqref{auxilliary1_interim3}), it follows that $\btw(t)\in\W'(t)$.
This, together with the definition of $\tau_i(t)$, implies that 
\begin{align}\label{auxilliary1-3_proof_interim3}
\tau_i(t) \leq \big(\bg_i-\bG_{\calA(t)}\btw_i(t)\big)^T\bY^*(t)\big(\bg_i-\bG_{\calA(t)}\btw_i(t)\big)
\end{align}
Let $\bG_{\calA(t)}\in\R^{d\times|\calA(t)|}$ be the matrix with $\bg_{i}, i\in\calA(t)$ as its columns. We have $$\bG_{\calA(t)}\btw(t) = \frac{1}{|\calS\cap\calA(t)|}\sum_{i\in\calS\cap\calA(t)}\bg_i =: \bg_{\calS\cap\calA(t)}.$$

In order to prove \Claimref{auxilliary1-3_proof}, we have to bound 
$\sum_{i\in\calS\cap\calA(t)}c_i(t)\tau_i(t)$:
\begin{align*}
\sum_{i\in\calS\cap\calA(t)}c_i(t)\tau_i(t) &\leq \sum_{i\in\calS\cap\calA(t)}c_i(t)\big(\bg_i-\bG_{\calA(t)}\btw_i(t)\big)^T\bY^*(t)\big(\bg_i-\bG_{\calA(t)}\btw_i(t)\big) \tag{From \eqref{auxilliary1-3_proof_interim3}} \\
&\stackrel{\text{(a)}}{=} \sum_{i\in\calS\cap\calA(t)}c_i(t)\big(\bg_i-\bg_{\calS\cap\calA(t)}\big)^T\bY^*(t)\big(\bg_i-\bg_{\calS\cap\calA(t)}\big) \\
&\stackrel{\text{(b)}}{\leq} \sum_{i\in\calS\cap\calA(t)}\big(\bg_i-\bg_{\calS\cap\calA(t)}\big)^T\bY^*(t)\big(\bg_i-\bg_{\calS\cap\calA(t)}\big) \\
&\stackrel{\text{(c)}}{\leq} \sum_{i\in\calS\cap\calA(t)}\big(\bg_i-\bg_{\calS}\big)^T\bY^*(t)\big(\bg_i-\bg_{\calS}\big) \\
&\stackrel{\text{(d)}}{\leq} \sum_{i\in\calS}\big(\bg_i-\bg_{\calS}\big)^T\bY^*(t)\big(\bg_i-\bg_{\calS}\big) \\
&\stackrel{\text{(e)}}{\leq}\left\|\sum_{i\in\calS}\big(\bg_i-\bg_{\calS}\big)\big(\bg_i-\bg_{\calS}\big)^T\right\| \\
&\stackrel{\text{(f)}}{=} \max_{\bv\in\R^d, \|\bv\|=1}\sum_{i\in\calS}\bv^T[(\bg_i-\bg_{\calS})(\bg_i-\bg_{\calS})^T]\bv \\
&= \max_{\bv\in\R^d, \|\bv\|=1}\sum_{i\in\calS}\langle\bg_i-\bg_{\calS}, \bv\rangle^2 \\
&\stackrel{\text{(g)}}{\leq} |\calS|\sigma_0^2 = \alpha R\sigma_0^2.
\end{align*}
In (a) we used $\bg_{\calS\cap\calA(t)}=\bG_{\calA(t)}\btw(t)$. 
In (b) we used that $c_i(t)\leq1$ for every $i\in\calA(t)$.
In (c) we used the fact that $\bg_{\calS\cap\calA(t)}$ is the minimizer of $\min_{\bu\in\R^d}\sum_{i\in\calS\cap\calA(t)}(\bg_i-\bu)^TY^*(t)(\bg_i-\bu)$. 
In (d) we used that $\bY^*(t)\succeq \bzero$, which implies that the additional terms corresponding to $i\in\calS\setminus\calA(t)$ are non-zero.
The reasoning for (e) is given below.
In (f) we used the definition of the matrix norm.
In (g) we used the hypothesis in \Lemmaref{poly-time_grad-est} about the set $\calS$.
\begin{align*}
\sum_{i\in\calS}\big(\bg_i-\bg_{\calS}\big)^T\bY^*(t)\big(\bg_i-\bg_{\calS}\big) 
&\stackrel{\text{(a)}}{=} \sum_{i\in\calS}\text{tr}\left(\big(\bg_i-\bg_{\calS}\big)^T\bY^*(t)\big(\bg_i-\bg_{\calS}\big)\right) \\
&\stackrel{\text{(b)}}{=} \sum_{i\in\calS}\text{tr}\left(\big(\bg_i-\bg_{\calS}\big)\big(\bg_i-\bg_{\calS}\big)^T\bY^*(t)\right) \\
&\stackrel{\text{(c)}}{=} \text{tr}\left[\left(\sum_{i\in\calS}\big(\bg_i-\bg_{\calS}\big)\big(\bg_i-\bg_{\calS}\big)^T\right)\bY^*(t)\right] \\
&\stackrel{\text{(d)}}{\leq} \left\|\sum_{i\in\calS}\big(\bg_i-\bg_{\calS}\big)\big(\bg_i-\bg_{\calS}\big)^T\right\|\left\|\bY^*(t)\right\|_* \\
&\stackrel{\text{(e)}}{\leq}\left\|\sum_{i\in\calS}\big(\bg_i-\bg_{\calS}\big)\big(\bg_i-\bg_{\calS}\big)^T\right\|
\end{align*}
The equality (a) follows because trace of a scalar is the scalar itself.
In (b) we used that $\text{tr}(\bA\bB) = \text{tr}(\bB\bA)$ for any two dimension compatible matrices $\bA,\bB$.
In (c) we used that $\text{tr}(\bA+\bB) = \text{tr}(\bA) + \text{tr}(\bB)$.
In (d) we used $\text{tr}(\bA\bB) \leq \|\bA\|\|\bB\|_*$ (shown below in \Claimref{holder}), where $\|\cdot\|_*$ denotes the nuclear norm, which is equal to the sum of singular values.
Note that the matrix norm here is induced by the $\ell_2$ norm. 
(e) follows from the fact that for a positive semi-definite matrix $\bY$, 
we have $\|\bY\|_*=\text{tr}(\bY)$, which is at most 1.
This completes the proof of \Claimref{auxilliary1-3_proof}.
\end{proof}

\begin{claim}\label{claim:holder}
For any two dimension compatible matrices $\bA,\bB$, we have $\emph{tr}(\bA\bB)\leq\|\bA\|\|\bB\|_*$, where $\|\cdot\|$ is the matrix norm induced by the $\ell_2$-norm, and $\|\cdot\|_*$ is the nuclear norm, which is equal to the sum of singular values of $\bB$.
\end{claim}
\begin{proof}
Let $r=\mathrm{rank}(\bB)$, and let $\sigma_i,i=1,2,\hdots,r$ denote the non-zero singular values of $\bB$.
By the singular value decomposition, we have $\bB=\sum_{i=1}^r\sigma_i\bu_i\bv_i^T$, where $\bu_i,\bv_i$ are the left and right singular vectors, respectively, corresponding to the singular value $\sigma_i$.
Note that $\bu_i,\bv_i$ for every $i=1,\hdots,r$ are unit norm vectors.
\begin{align*}
\text{tr}(\bA\bB) &= \text{tr}(\bA\sum_{i=1}^r\sigma_i\bu_i\bv_i^T) \\
&= \sum_{i=1}^r\sigma_i\text{tr}(\bA\bu_i\bv_i^T) \tag{Since \text{tr}(\bA+\bB) = \text{tr}(\bA) + \text{tr}(\bB)} \\
&= \sum_{i=1}^r\sigma_i\text{tr}(\bv_i^T\bA\bu_i) \tag{Since $\text{tr}(\bA\bB) = \text{tr}(\bB\bA)$} \\
&= \sum_{i=1}^r\sigma_i\bv_i^T\bA\bu_i \tag{Since $\bv_i^T\bA\bu_i$ is a scalar} \\
&\stackrel{\text{(a)}}{\leq} \sum_{i=1}^r\sigma_i\|\bv_i\|\|\bA\|\|\bu_i\| \notag \\
&= \|\bA\|\sum_{i=1}^r\sigma_i \tag{Since $\|\bu_i\|=1, \|\bv_i\|=1$, for every $i=1,\hdots,r$} \\
&= \|\bA\|\|\bB\|_*
\end{align*}
In (a), first we used the Cauchy-Schwarz inequality to write $\bv_i^T\bA\bu_i \leq \|\bv_i\|\|\bA\bu_i\|$ and then used the definition of matrix norm to write $\|\bA\bu_i\| \leq \|\bA\|\|\bu_i\|$.
Note that if $\bA,\bB$ are positive semi-definite, then equality holds in (a) above if and only if $\bB$ is a multiple of $\bu\bu^T$, where $\bu$ is the eigenvector corresponding to the largest eigenvalue of $\bA$.

This completes the proof of \Claimref{holder}.
\end{proof}
Now we are ready to prove the inductive step for \eqref{auxilliary1-1} and \eqref{auxilliary1-2}.
\begin{align}
\sum_{i\in\calS}\big(1-c_i(t+1)\big) &= \sum_{i\in\calS\cap\calA(t)}\big(1-c_i(t+1)\big) + \sum_{i\in\calS\setminus\calA(t)}\big(1-c_i(t+1)\big) \nonumber \\
&= \sum_{i\in\calS\cap\calA(t)}\big(1-c_i(t+1)\big) + \sum_{i\in\calS\setminus\calA(t)}\big(1-c_i(t)\big) \tag{Since $c_i(t+1)=c_i(t), \forall i\notin\calA(t)$} \nonumber \\
&=  \sum_{i\in\calS}\big(1-c_i(t)\big) + \sum_{i\in\calS\cap\calA(t)}\big(c_i(t)-c_i(t+1)\big) \label{auxilliary1-induction1}
\end{align}
The first term in \eqref{auxilliary1-induction1} can be bounded using the induction hypothesis \eqref{auxilliary1_interim2}: 
$\sum_{i\in\calS}\big(1-c_i(t)\big) \leq \frac{\alpha}{4}\sum_{i=1}^R\big(1-c_i(t)\big)$. The second term is equal to $\frac{1}{\tau_{\max}(t)}\sum_{i\in\calS\cap\calA(t)}\tau_i(t)c_i(t)$ (see line 6 of \Algorithmref{robust-grad-est}). 
It follows from \eqref{auxilliary1_interim1} and \Claimref{auxilliary1-3_proof} that $\sum_{i\in\calS\cap\calA(t)}c_i(t)\tau_i(t) < \frac{\alpha}{4}\sum_{i\in\calA(t)}c_i(t)\tau_i(t)$. Substituting these in \eqref{auxilliary1-induction1}, we get
\begin{align}
\sum_{i\in\calS}\big(1-c_i(t+1)\big) &< \frac{\alpha}{4}\sum_{i=1}^R\big(1-c_i(t)\big) + \frac{\alpha}{4\tau_{\max}(t)}\sum_{i\in\calA(t)}c_i(t)\tau_i(t) \nonumber \\
&= \frac{\alpha}{4}\left[\sum_{i\in[R]\setminus\calA(t)}\big(1-c_i(t)\big) + \sum_{i\in\calA(t)}\big(1-c_i(t)\big) + \sum_{i\in\calA(t)}c_i(t)\frac{\tau_i(t)}{\tau_{\max}(t)}\right] \nonumber \\
&= \frac{\alpha}{4}\left[\sum_{i\in[R]\setminus\calA(t)}\big(1-c_i(t+1)\big) + \sum_{i\in\calA(t)}\left(1-\underbrace{\left(1-\frac{\tau_i(t)}{\tau_{\max}(t)}\right)c_i(t)}_{=\ c_i(t+1)}\right) \right] \tag{Since $c_i(t+1)=c_i(t), \forall i\notin\calA(t)$} \\
&= \frac{\alpha}{4}\sum_{i=1}^R\big(1-c_i(t+1)\big). \label{auxilliary1-induction2}
\end{align}
This proves that \eqref{auxilliary1-1} holds throughout the execution of the {\bf while}-loop of \Algorithmref{robust-grad-est}. 
Now we prove \eqref{auxilliary1-2}.

From \eqref{auxilliary1-induction2}, we have $\sum_{i\in\calS}\big(1-c_i(t+1)\big) = \frac{\alpha}{4}\left(\sum_{i\in\calS}\big(1-c_i(t+1) + \sum_{i\in[R]\setminus\calS}\big(1-c_i(t+1)\big)\right)$. By rearranging terms, we get $\sum_{i\in\calS}\big(1-c_i(t+1)\big) \leq \frac{\alpha}{4-\alpha}\sum_{i\in[R]\setminus\calS}\big(1-c_i(t+1)\big)$. Since $c_i(t+1)\geq 0$ for all $i,t$, and $|\calS|\geq\alpha R$, we have 
\begin{align}
\sum_{i\in\calS}\big(1-c_i(t+1)\big) \leq \frac{\alpha(1-\alpha)}{4-\alpha}R. \label{auxilliary1-induction3}
\end{align}
It follows from \eqref{auxilliary1-induction3} that the number of $i$'s in $\calS$ for which $c_i(t+1) < \frac{1}{2}$ is at most $\frac{2\alpha(1-\alpha)}{4-\alpha}R$. So, at most $\frac{2\alpha(1-\alpha)}{4-\alpha}R$ elements of $\calS$ can be removed from $\calA(t)$ in total in line 7 of \Algorithmref{robust-grad-est}. This implies $|\calS\setminus\calA(t+1)|\leq\frac{2\alpha(1-\alpha)}{4-\alpha}R$. So, $|\calS\cap\calA(t+1)| = |\calS|-|\calS\setminus\calA(t+1)| \geq \alpha R - \frac{2\alpha(1-\alpha)}{4-\alpha}R = \frac{\alpha(2+\alpha)}{4-\alpha}R$. This proves that \eqref{auxilliary1-2} is preserved throughout the execution of the {\bf while}-loop of \Algorithmref{robust-grad-est}. 

This concludes the proof of \Lemmaref{auxilliary1}.
\end{proof}

\begin{claim}\label{claim:W-W1_rank1}
Let $\bW$ be the minimizer of \eqref{saddle-point-opt} in \Algorithmref{robust-grad-est}, when the {\bf while}-loop terminates. Let $\bW_1$ be the result of zeroing out all singular values of $\bW$ that are greater than 0.9. Then $\bW-\bW_1$ is a rank one matrix.
\end{claim}
\begin{proof}
Note that the constraints in \eqref{saddle-point-opt} imply that $\|\bW\|_F^2\leq \frac{4-\alpha}{\alpha(2+\alpha)}$. 
This can be seen as follows: 
\begin{align*}
\|\bW\|_F^2 &= \sum_{i,j\in\calA}W_{ji}^2 \\
&\leq \sum_{i.j\in\calA} W_{ji}\cdot \max_{k,l\in\calA} W_{kl} \\
&\leq \sum_{i\in\calA}\left(\sum_{j\in\calA} W_{ji}\right)\cdot\frac{4-\alpha}{\alpha(2+\alpha)R} \tag{Since $\max_{k,l\in\calA} W_{kl} \leq \frac{4-\alpha}{\alpha(2+\alpha)R}$; see \eqref{saddle-point-opt}} \\
&= \frac{4-\alpha}{\alpha(2+\alpha)}. \tag{Since $\sum_{j\in\calA} W_{ji}=1, \quad \forall i\in\calA$; see \eqref{saddle-point-opt}}
\end{align*}
Let $\sigma_1,\sigma_2,\hdots,\sigma_{|\calA|}$ denote the singular values of $\bW$, arranged in non-increasing order. Since $\|\bW\|_F^2 = \sum_{i\in\calA} \sigma_i^2$, and that, $\frac{4-\alpha}{\alpha(2+\alpha)} \leq \nicefrac{52}{33} < 2\times 0.9^2$ for $\alpha\geq\nicefrac{3}{4}$, we have that at most one singular value of $\bW$ can be greater than $0.9$. Since $\sum_{j\in\calA} W_{ji}=1, \quad \forall i\in\calA$, i.e., $\bW$ is a column stochastic matrix, we know that its largest singular value is at least 1. Together, we have that $\sigma_1\geq 1$ and $\sigma_i < 0.9, \forall i=2,3,\hdots,|\calA|$. Since $\bW_1$ is the result of zeroing out all singular values of $\bW$ that are greater than 0.9, as a consequence, $\bW-\bW_1$ is a rank one matrix.
\end{proof}

\begin{claim}\label{claim:W0-rank1}
Let $\bW_0=(\bW-\bW_1)(\bI-\bW_1)^{-1}$. Then $\bW_0$ is a column stochastic matrix of rank one.
\end{claim}
\begin{proof}
First we show that $\bW_0$ is a column stochastic matrix, i.e., ${\bf 1}^T\bW_0={\bf 1}^T$.
\begin{align*}
{\bf 1}^T\bW_0 &= {\bf 1}^T(\bW-\bW_1)(\bI-\bW_1)^{-1} \\
&= ({\bf 1}^T-{\bf 1}^T\bW_1)(\bI-\bW_1)^{-1} \tag{Since $({\bf 1}^1\bW={\bf 1}^T) \iff (\sum_{j\in\calA} W_{ji}=1,\ \forall i\in\calA$); see \eqref{saddle-point-opt}} \\
&= {\bf 1}^T(\bI-\bW_1)(\bI-\bW_1)^{-1} \\
&= {\bf 1}^T.
\end{align*}
Now we show that $\mathrm{rank}(\bW_0)=1$. Since for any two dimension-compatible matrices $\bA,\bB$, we have $\mathrm{rank}(\bA\bB) \leq \min\{\mathrm{rank}(\bA),\mathrm{rank}(\bB)\}$. As a consequence, since $\bW_0=(\bW-\bW_1)(I-\bW_1)^{-1}$, where $\mathrm{rank}(W-W_1)=1$ (from \Claimref{W-W1_rank1}), we have that $\mathrm{rank}(\bW_0)\leq 1$. 
Since $\bW_0$ is a column stochastic matrix, its largest singular value is at least one, implying that $\mathrm{rank}(\bW_0)\geq1$. As a result, we have that $\mathrm{rank}(\bW_0)=1$.
\end{proof}
Since $\bW_0$ is a column stochastic matrix of rank one, it follows that $\bW_0=\bu{\bf 1}^T$ for some vector $\bu\in\R^{|\calA|}$. Recall that $\bG_{\calA}$ is a $d\times|\calA|$ matrix whose columns are the gradients $\bg_i, i\in\calA$.
Define $\bg':=\bG_{\calA}\bu$. This implies that all the columns of $\bG_{\calA}\bW_0$ are identical and equal to $\bg'$: $\bG_{\calA}\bW_0=\bG_{\calA}\bu{\bf 1}^T=\bg'{\bf 1}^T=[\bg'\ \bg'\ \hdots\ \bg']$.
\begin{lemma}\label{lem:auxilliary2}
$\left\|\bG_{\calA}-\bg'{\bf 1}^T\right\| \leq 20\sqrt{2R}\sigma_0$.
\end{lemma}
\begin{proof}
First we show $\left\|\bG_{\calA}(\bI-\bW)\sqrt{\mathrm{diag}(c_{\calA})}\right\|\leq2\sqrt{R}\sigma_0$,
where $\sqrt{\mathrm{diag}(c_A)}$ denotes a $|\calA|\times|\calA|$ diagonal matrix, whose entries are $\sqrt{c_i}, i\in\calA$. For this, we use the following identity: 

\begin{claim}\label{claim:matrix-norm_trace-maximization}
For any matrix $\bZ$ with entries from $\R$, we have 
$$\|\bZ\|^2 = \max_{\bY\succeq\bzero, \emph{tr}(\bY)\leq1}\emph{tr}(\bZ^T\bY\bZ).$$
\end{claim}
\begin{proof}
Take any $\bY\succeq\bzero$ such that $\text{tr}(\bY)\leq1$. First we show that $\text{tr}(\bZ^T\bY\bZ) \leq \|\bZ\|^2$.
\[\text{tr}(\bZ^T\bY\bZ) \ \stackrel{\text{(a)}}{=} \ \text{tr}(\bZ\bZ^T\bY) 
\ \stackrel{\text{(b)}}{\leq}\ \|\bZ\bZ^T\|\|\bY\|_* 
\ \stackrel{\text{(c)}}{\leq}\ \|\bZ\bZ^T\| 
\ \stackrel{\text{(d)}}{=}\ \|\bZ\|^2.\]
Here (a) follows because $\text{tr}(\bA\bB) = \text{tr}(\bB\bA)$ for any $\bA,\bB$.
The inequality (b) follows from $\text{tr}(\bA\bB)\leq\|\bA\|\|\bB\|_*$ (which is proved in \Claimref{holder}), where $\|\cdot\|_*$ denotes the nuclear norm, and $\|\bB\|_*$ is equal to sum of the singular values of $\bB$. (c) follows from the fact that for a positive semi-definite matrix $\bY$, we have $\|\bY\|_*=\text{tr}(\bY)$, which is at most 1. (d) follows because $\|\bZ\| = \|\bZ^T\|$ and that the largest singular value $\sigma_{\max}(\bZ)$ of $\bZ$ is equal to $\sqrt{\lambda_{\max}(\bZ^T\bZ)}$, where $\lambda_{\max}(\bZ^T\bZ)=\|\bZ^T\bZ\|$ is the largest eigenvalue of $\bZ^T\bZ$.

Taking maximum over all $\bY\succeq\bzero$ such that $\text{tr}(\bY)\leq1$, we get 
\begin{align}\label{matrix-norm_trace-max_interim1}
\max_{\bY\succeq\bzero, \text{tr}(\bY)\leq1}\text{tr}(\bZ^T\bY\bZ) \leq \|\bZ\|^2.
\end{align}
Now we show that $\max_{\bY\succeq\bzero, \text{tr}(\bY)\leq1}\text{tr}(\bZ^T\bY\bZ) \geq \|\bZ\|^2$. 
This, together with \eqref{matrix-norm_trace-max_interim1}, proves \Claimref{matrix-norm_trace-maximization}.
\begin{align*}
\max_{\bY\succeq\bzero, \text{tr}(\bY)\leq1} \text{tr}(\bZ^T\bY\bZ) &\stackrel{\text{(e)}}{\geq} \max_{\bv:\|\bv\|=1} \text{tr}(\bZ^T\bv\bv^T\bZ) \\
&= \max_{\bv:\|\bv\|=1} \text{tr}(\bv^T\bZ\bZ^T\bv) \\
&= \max_{\bv:\|\bv\|=1} \bv^T\bZ\bZ^T\bv \\
&= \max_{\bv:\|\bv\|=1} \|\bZ^T\bv\|^2 \\
&\stackrel{\text{(g)}}{=} \|\bZ^T\|^2 \stackrel{\text{(h)}}{=} \|\bZ\|^2
\end{align*}
In (e), we restrict to those $\bY$'s which are of the form $\bY=\bv\bv^T$ for a unit norm vector $\bv$.
It is immediate that such $\bY$'s are positive semi-definite and their trace is equal to one.
Note that the vector $\bv$ that attains the maximum in the RHS of (g) is the eigenvector corresponding to the largest eigenvalue of $\bZ\bZ^T$.
In (g) we used the definition of the matrix norm $\|\bA\| = \max_{\bv:\|\bv\|=1}\|\bA\bv\|$, and in (h) we used the fact that transpose does not change the matrix norm.
\end{proof}

By letting $\bZ=[\bz_1\ \bz_2\ \hdots\ \bz_m]$, where $\{\bz_i\}_{i=1}^m$ are the columns of $\bZ$, we have $\text{tr}(\bZ^T\bY\bZ) = \sum_{i=1}^m\bz_i^T\bY\bz_i$. Applying this identity to $\bZ=\bG_{\calA}(\bI-\bW)\sqrt{\mathrm{diag}(c_{\calA})}$ (which implies $z_i=\sqrt{c_i}(\bg_i-\bG_{\calA}\bw_i)$) gives
\begin{align}
\|\bG_{\calA}(\bI-\bW)\sqrt{\mathrm{diag}(c_A)}\|_2^2  &=   \max_{\bY\succeq\bzero, \text{tr}(\bY)\leq1}\sum_{i\in\calA}\left[\sqrt{c_i}(\bg_i-\bG_{\calA}\bw_i)\right]^T\bY\left[\sqrt{c_i}(\bg_i-\bG_{\calA}\bw_i)\right] \notag \\
&= \max_{\bY\succeq\bzero, \text{tr}(\bY)\leq1}\sum_{i\in\calA}c_i(\bg_i-\bG_{\calA}\bw_i)^T\bY(\bg_i-\bG_{\calA}\bw_i) \notag \\
&\stackrel{\text{(a)}}{=} \sum_{i\in\calA} c_i\tau_i \notag \\
&\stackrel{\text{(b)}}{\leq} 4R\sigma_0^2. \label{auxilliary2_bound}
\end{align}
For (a), note that $\bW$ is the optimal weight matrix when the {\bf while}-loop of \Algorithmref{robust-grad-est} terminates, and $\tau_i$ is defined w.r.t.\ the optimal $\bW,\bY$. In (b) we used the fact that the {\em while}-loop terminates when the condition in line 5 of \Algorithmref{robust-grad-est} is violated, i.e., $\sum_{i\in\calA} c_i\tau_i \leq 4R\sigma_0^2$.

Now, we are ready to prove \Lemmaref{auxilliary2}.
\begin{align*}
\left\|\bG_{\calA}-\bg'{\bf 1}^T\right\| &= \left\|\bG_{\calA}-\bG_{\calA}\bW_0\right\| \tag{Since $\bG_{\calA}\bW_0=\bg'{\bf 1}^T$}\\
&= \left\|\bG_{\calA}(\bI-\bW_1)(\bI-\bW_1)^{-1}-\bG_{\calA}(\bW-\bW_1)(\bI-\bW_1)^{-1}\right\| \\
&= \left\|\bG_{\calA}[(\bI-\bW_1)-(\bW-\bW_1)](\bI-\bW_1)^{-1}\right\| \\
&= \left\|\bG_{\calA}(\bI-\bW)(\bI-\bW_1)^{-1}\right\| \\
&\leq \left\|\bG_{\calA}(\bI-\bW)\right\|\left\|(\bI-\bW_1)^{-1}\right\| \\
&\leq 10\left\|\bG_{\calA}(\bI-\bW)\right\| \tag{Since the largest singular value of $\bW_1$ is at most 0.9} \\
&\stackrel{\text{(a)}}{\leq} 10\sqrt{2}\left\|\bG_{\calA}(\bI-\bW)\sqrt{\mathrm{diag}(c_{\calA})}\right\|  \tag{Since $c_i\geq \frac{1}{2}$ for all $i\in\calA$} \\
&\leq 20\sqrt{2R}\sigma_0. \tag{Using \eqref{auxilliary2_bound}}
\end{align*}
In (a), we used the fact that $c_i\geq \frac{1}{2}$ for all $i\in\calA$; 
and $\sqrt{\mathrm{diag}(c_A)}$ denotes a $|\calA|\times|\calA|$ diagonal matrix, whose entries are $\sqrt{c_i}, i\in\calA$. 

This completes the proof of \Lemmaref{auxilliary2}.
\end{proof}

\begin{claim}\label{claim:resilience_honest-grad}\label{resilience_honest-grad}
The set $\calS$ in the hypothesis of \Lemmaref{poly-time_grad-est} is $(2\sigma_0\sqrt{\eps_1},\eps_1)$-resilient around its empirical mean $\bg_{\calS}$ for any $\eps_1\in[0,\frac{1}{2})$.
\end{claim}
\begin{proof}
Take an arbitrary $\eps_1\in[0,\frac{1}{2})$.
We need to show that for any subset $\calT\subseteq\calS$ such that $|\calT|\geq(1-\eps_1)|\calS|$, we have $\left\| \bg_{\calT} -\bg_{\calS} \right\| \leq 2\sigma_0\sqrt{\eps_1}$, where $\bg_{\calT} = \frac{1}{|\calT|}\sum_{i\in\calT}\bg_i$.
For any subset $\calT\subseteq[R]$, we write $\bG_{\calT}$ to denote the $d\times|\calT|$ matrix whose columns are $\bg_i, i\in\calT$.
\begin{align}
\|\bg_{\calT} - \bg_{\calS}\| &= \left\|\frac{1}{|\calT|}\sum_{i\in\calT}(\bg_i - \bg_{\calS})\right\| \nonumber \\
&= \left\|\left(\frac{1}{|\calT|} - \frac{1}{|\calS|}\right)\sum_{i\in\calT}(\bg_i-\bg_{\calS}) + \frac{1}{|\calS|}\sum_{i\in\calT}(\bg_i-\bg_{\calS})\right\| \nonumber \\
&\stackrel{\text{(a)}}{=} \left\|\left(\frac{1}{|\calT|} - \frac{1}{|\calS|}\right)\sum_{i\in\calT}(\bg_i-\bg_{\calS}) - \frac{1}{|\calS|}\sum_{i\in\calS\setminus\calT}(\bg_i-\bg_{\calS})\right\| \nonumber \\
&= \left\|\left(\frac{1}{|\calT|} - \frac{1}{|\calS|}\right)[\bG_{\calT}-\bg_{\calS}{\bf 1}^T]{\bf 1} - \frac{1}{|\calS|}[\bG_{\calS\setminus\calT}-\bg_{\calS}{\bf 1}^T]{\bf 1}\right\| \nonumber \\
&\leq \left(\frac{1}{|\calT|} - \frac{1}{|\calS|}\right)\left\|[\bG_{\calT}-\bg_{\calS}{\bf 1}^T]{\bf 1}\right\| + \frac{1}{|\calS|}\left\|[\bG_{\calS\setminus\calT}-\bg_{\calS}{\bf 1}^T]{\bf 1}\right\| \nonumber \\
&\leq \left(\frac{1}{|\calT|} - \frac{1}{|\calS|}\right)\left\|\bG_{\calT}-\bg_{\calS}{\bf 1}^T\right\|\cdot\left\|{\bf 1_{|\calT|}}\right\| + \frac{1}{|\calS|}\left\|\bG_{\calS\setminus\calT}-\bg_{\calS}{\bf 1}^T\right\|\cdot\left\|{\bf 1_{|\calS\setminus\calT|}}\right\| \nonumber \\
&= \left(\frac{|\calS|-|\calT|}{\sqrt{|\calT|}|\calS|}\right)\left\|\bG_{\calT}-\bg_{\calS}{\bf 1}^T\right\| + \frac{\sqrt{|\calS\setminus\calT|}}{|\calS|}\left\|\bG_{\calS\setminus\calT}-\bg_{\calS}{\bf 1}^T\right\| \nonumber \\
&\stackrel{\text{(b)}}{\leq} \left(\frac{|\calS|-|\calT|}{\sqrt{|\calT|}|\calS|} + \frac{\sqrt{|\calS\setminus\calT|}}{|\calS|}\right)\left\|\bG_{\calS}-\bg_{\calS}{\bf 1}^T\right\| \label{resilience_honest-grad1}
\end{align}
In the above inequalities, we denote the all 1's vector by ${\bf 1}$; its dimension is not explicitly written, and is taken what makes the operation dimension compatible.
In (a) we used the fact that $\sum_{i\in\calS}(\bg_i-\bg_{\calS})=0$ and that $\sum_{i\in\calS}(\bg_i-\bg_{\calS}) = \sum_{i\in\calT}(\bg_i-\bg_{\calS}) + \sum_{i\in\calS\setminus\calT}(\bg_i-\bg_{\calS})$.
The inequality (b) follows from the fact that for any matrix $\bA$, if we consider another matrix $\bA'$ whose column set is a subset of the column set of $\bA$, then we have $\|\bA'\|\leq\|\bA\|$. This can be seen as follows: Let $\bP$ denote the projection matrix that selects columns of $\bA$ to construct $\bA'$ such that $\bA' = \bA\bP$; note that $\bP$ is a diagonal matrix with $\bP_{ii}=1$ if and only if the $i$'th column of $\bA$ is selected. 
Since the matrix norm $\|\cdot\|$ of a matrix is equal to its largest singular value, and the largest singular value of $\bP$ is equal to 1, we have $\|\bP\|=1$.
Now (b) follows from $\|\bA'\| = \|\bA\bP\| \leq \|\bA\|\cdot\|\bP\| = \|\bA\|$, where we used $\|\bA\bB\| \leq \|\bA\|\cdot\|\bB\|$ which holds for any two dimension-compatible matrices $\bA,\bB$.

In order to bound \eqref{resilience_honest-grad1}, first we bound the coefficient $\left(\frac{|\calS|-|\calT|}{\sqrt{|\calT|}|\calS|}\right)$ as follows:
\begin{align}\label{resilience_honest-grad2}
\left(\frac{|\calS|-|\calT|}{\sqrt{|\calT|}|\calS|} + \frac{|\calS\setminus\calT|}{|\calS|}\right) \ \leq \ \left(\frac{\eps_1}{\sqrt{1-\eps_1}} + \sqrt{\eps_1}\right)\frac{1}{\sqrt{|\calS|}}
\ \leq \ \frac{2\sqrt{\eps_1}}{\sqrt{|\calS|}},
\end{align}
where in the first inequality we used $|\calT|\geq(1-\eps_1)|\calS|$ and $|\calS\setminus\calT| \leq \eps_1|\calS|$; the second inequality holds because $\eps_1<\nicefrac{1}{2}$.

Now we bound $\left\|\bG_{\calS}-\bg_{\calS}{\bf 1}^T\right\|$:
\begin{align}
\left\|\bG_{\calS}-\bg_{\calS}{\bf 1}^T\right\| &\stackrel{\text{(c)}}{=} \left\|\(\bG_{\calS}-\bg_{\calS}{\bf 1}^T\)^T\right\| \nonumber \\
&\stackrel{\text{(d)}}{=} \sqrt{\lambda_{\max}\(\(\bG_{\calS}-\bg_{\calS}{\bf 1}^T\)\(\bG_{\calS}-\bg_{\calS}{\bf 1}^T\)^T\)} \nonumber \\
&\stackrel{\text{(e)}}{=} \sqrt{\lambda_{\max}\(\sum_{i\in\calS}(\bg_{i} - \bg_{\calS})(\bg_{i} - \bg_{\calS})^T\)} \nonumber \\
&\stackrel{\text{(f)}}{\leq} \sqrt{|\calS|}\sigma_0. \label{resilience_honest-grad3}
\end{align}
Here, (c) follows from the fact that the matrix norm $\|\cdot\|$ of a rectangular matrix is equal to the matrix norm of its transpose, i.e., for any matrix $\bA$, we have $\|\bA\| = \|\bA^T\|$.
(d) follows from the fact that for any matrix $\bA$, its matrix norm $\|\bA\|$ is equal to its largest singular value, which is equal to $\|\bA\| = \sqrt{\lambda_{\max}(\bA^T\bA)}$.
(e) uses the identity $\(\bG_{\calS}-\bg_{\calS}{\bf 1}^T\)\(\bG_{\calS}-\bg_{\calS}{\bf 1}^T\)^T = \sum_{i\in\calS}(\bg_{i} - \bg_{\calS})(\bg_{i} - \bg_{\calS})^T$.
(f) follows from the hypothesis of \Lemmaref{poly-time_grad-est} that $\lambda_{\max}\(\frac{1}{|\calS|}\sum_{i\in\calS} \(\bg_i - \bg_{\calS}\)\(\bg_i - \bg_{\calS}\)^T\) \leq \sigma_0^2$.

Substituting the bounds from \eqref{resilience_honest-grad2} and \eqref{resilience_honest-grad3} back in \eqref{resilience_honest-grad1} completes the proof of \Claimref{resilience_honest-grad}.
\end{proof}

Note that using an iterative procedure, \Algorithmref{robust-grad-est} finds a subset $\calA\subset [R]$ and outputs $\btg = \frac{1}{|\calA|}\sum_{i\in\calA}\bg_i$. Now we show that $\calA$ is resilient around its sample mean $\btg$. 
\begin{claim}\label{claim:resilience_grad-est}\label{resilience_grad-est}
The set $\calA$ produced when \Algorithmref{robust-grad-est} terminates is $(80\sigma_0\sqrt{\eps_2},\eps_2)$-resilient around its sample mean $\btg$ for any $\eps_2\in[0,\frac{1}{2})$.
\end{claim}
\begin{proof}
Take an arbitrary $\eps_2\in[0,\frac{1}{2})$.
We need to show that for any subset $\calT\subseteq\calA$ such that $|\calT|\geq(1-\eps_2)|\calA|$, we have $\left\| \bg_{\calT} -\btg \right\| \leq 80\sigma_0\sqrt{\eps_2}$, where $\bg_{\calT} = \frac{1}{|\calT|}\sum_{i\in\calT}\bg_i$.
Let $\bG=[\bg_1,\hdots,\bg_R]\in\R^{d\times R}$. 
\begin{align}
\|\bg_{\calT} - \btg\| &= \left\|\frac{1}{|\calT|}\sum_{i\in\calT}\bg_i - \frac{1}{|\calA|}\sum_{i\in\calA}\bg_i\right\| \nonumber \\
&= \left\|\frac{1}{|\calT|}\sum_{i\in\calT}(\bg_i-\bg') - \frac{1}{|\calA|}\sum_{i\in\calA}(\bg_i-\bg')\right\| \tag{where $\bg'=\bG_{\calA}\bu$ and $\bu$ is such that $\bW_0=\bu{\bf 1}^T$} \\
&= \left\|\left(\frac{1}{|\calT|} - \frac{1}{|\calA|}\right)\sum_{i\in\calT}(\bg_i-\bg') - \frac{1}{|\calA|}\sum_{i\in\calA\setminus\calT}(\bg_i-\bg')\right\| \nonumber \\
&= \left\|\left(\frac{1}{|\calT|} - \frac{1}{|\calA|}\right)[\bG_{\calT}-\bg'{\bf 1}^T]{\bf 1} - \frac{1}{|\calA|}[\bG_{\calA\setminus\calT}-\bg'{\bf 1}^T]{\bf 1}\right\| \nonumber \\
&\leq \left(\frac{1}{|\calT|} - \frac{1}{|\calA|}\right)\left\|[\bG_{\calT}-\bg'{\bf 1}^T]{\bf 1}\right\| + \frac{1}{|\calA|}\left\|[\bG_{\calA\setminus\calT}-\bg'{\bf 1}^T]{\bf 1}\right\| \nonumber \\
&\leq \left(\frac{1}{|\calT|} - \frac{1}{|\calA|}\right)\left\|\bG_{\calT}-\bg'{\bf 1}^T\right\|\cdot\left\|{\bf 1_{|\calT|}}\right\| + \frac{1}{|\calA|}\left\|\bG_{\calA\setminus\calT}-\bg'{\bf 1}^T\right\|\cdot\left\|{\bf 1_{|\calA\setminus\calT|}}\right\| \nonumber \\
&= \left(\frac{|\calA|-|\calT|}{\sqrt{|\calT|}|\calA|}\right)\left\|\bG_{\calT}-\bg'{\bf 1}^T\right\| + \frac{\sqrt{|\calA\setminus\calT|}}{|\calA|}\left\|\bG_{\calA\setminus\calT}-\bg'{\bf 1}^T\right\| \nonumber \\
&\leq \left(\frac{|\calA|-|\calT|}{\sqrt{|\calT|}|\calA|} + \frac{\sqrt{|\calA\setminus\calT|}}{|\calA|}\right)\left\|\bG_{\calA}-\bg'{\bf 1}^T\right\| \label{resilience_grad-est1}
\end{align}
In the last inequality, we used that $\left\|\bG_{\calA}-\bg'{\bf 1}^T\right\|\geq\left\|\bG_{\calA'}-\bg'{\bf 1}^T\right\|$ for any subset $\calA'\subseteq\calA$. This can be shown along the lines of the reasoning given for \eqref{resilience_honest-grad1}.
In order to bound the RHS of \eqref{resilience_grad-est1}, first we bound the coefficient $\left(\frac{|\calA|-|\calT|}{\sqrt{|\calT|}|\calA|} + \frac{\sqrt{|\calA\setminus\calT|}}{|\calA|}\right)$ and then we bound $\left\|\bG_{\calA}-\btg{\bf 1}^T\right\|$.
To bound the coefficient, we can use the steps used for obtaining \eqref{resilience_honest-grad2}, which gives 
\begin{align}\label{resilience_grad-est2}
\left(\frac{|\calA|-|\calT|}{\sqrt{|\calT|}|\calA|} + \frac{\sqrt{|\calA\setminus\calT|}}{|\calA|}\right) \leq \frac{2\sqrt{\eps_2}}{\sqrt{|\calA|}}.
\end{align} 
We need to lower-bound $|\calA|$, which we do as follows:
$|\calA| \geq |\calA\cap\calS| \stackrel{\text{(a)}}{\geq} \frac{\alpha(\alpha+2)}{4-\alpha}R \stackrel{\text{(b)}}{\geq} \frac{33R}{52} > \frac{R}{2}$, where (a) follows from \Lemmaref{auxilliary1} and (b) follows by substituting $\alpha=(1-\tilde{\eps}) \geq \nicefrac{3}{4}$. Substituting this in \eqref{resilience_grad-est2} and using the resulting bound in \eqref{resilience_grad-est1} gives:
\begin{align}\label{resilience_grad-est3}
\|\bg_{\calT} - \btg\| \leq 2\sqrt{\frac{2\eps_2}{R}}\left\|\bG_{\calA}-\bg'{\bf 1}^T\right\|
\end{align}
Substituting the bound $\left\|\bG_{\calA}-\bg'{\bf 1}^T\right\|\leq 20\sqrt{2R}\sigma_0$ from \Lemmaref{auxilliary2} gives $\|\bg_{\calT} - \btg\| \leq 80\sqrt{\eps_2}\sigma_0$, which completes the proof of \Claimref{resilience_grad-est}.
\end{proof}
Now we show that \Claimref{resilience_honest-grad} and \Claimref{resilience_grad-est} imply $\|\btg-\bg_{\calS}\|\ \leq \calO(\sigma_0\sqrt{\tilde{\eps}})$, proving \Lemmaref{poly-time_grad-est}.
\begin{claim}\label{claim:final-claim}\label{final-claim}
$\|\btg-\bg_{\calS}\|\ \leq \calO(\sigma_0\sqrt{\tilde{\eps}})$, where $\tilde{\eps}\leq \nicefrac{1}{4}$.
\end{claim}
\begin{proof}
We have from \Claimref{resilience_honest-grad} that for any subset $\calT_1\subseteq\calS$ such that $|\calT_1|\geq(1-\eps_1)|\calS|$, where $\eps_1\in[0,\frac{1}{2})$, we have $\left\| \bg_{\calT_1} -\bg_{\calS} \right\| \leq 2\sigma_0\sqrt{\eps_1}$.
We have from \Claimref{resilience_grad-est} that for any subset $\calT_2\subseteq\calA$ such that $|\calT_2|\geq(1-\eps_2)|\calA|$, where $\eps_2\in[0,\frac{1}{2})$, we have $\left\| \bg_{\calT_2} -\btg \right\| \leq 80\sigma_0\sqrt{\eps_2}$.

We take $\calT_1=\calT_2=\calS\cap\calA$. For that we need to show $|\calS\cap\calA| \geq (1-\eps_1)|\calS|$ and $|\calS\cap\calA| \geq (1-\eps_2)|\calA|$ hold for some $\eps_1,\eps_2 \in[0, \frac{1}{2})$. We show these below.

We have from \Lemmaref{auxilliary1} that $|\calS\cap\calA| \geq \frac{\alpha(\alpha+2)}{4-\alpha}R$, where $\alpha=1-\tilde{\eps}\geq\frac{3}{4}$. It can be verified that $\frac{\alpha(\alpha+2)}{4-\alpha} \geq 1-\frac{5}{3}(1-\alpha)$ (which is equivalent to the fact that $(1-\alpha)^2\geq0$). Substituting this and that $R\geq\max\{|\calS|,|\calA|\}$, we get $|\calS\cap\calA| \geq \left(1-\frac{5}{3}(1-\alpha)\right)\max\{|\calS|,|\calA|\}$. Note that $\frac{5}{3}(1-\alpha) < \frac{1}{2}$. 
We have shown that if we take $\calT_1=\calT_2=\calS\cap\calA$ and $\eps_1=\eps_2=\frac{5}{3}(1-\alpha) < \frac{1}{2}$, we have $|\calT_1| \geq \left(1-\frac{5}{3}(1-\alpha)\right)|\calS|$ and $|\calT_2| \geq \left(1-\frac{5}{3}(1-\alpha)\right)|\calA|$. These give
\begin{align*}
\left\| \bg_{\calS\cap\calA} -\bg_{\calS} \right\| &\leq 2\sqrt{\frac{5}{3}}\sigma_0\sqrt{1-\alpha}. \\
\left\| \bg_{\calS\cap\calA} -\btg \right\| &\leq 80\sqrt{\frac{5}{3}}\sigma_0\sqrt{1-\alpha}
\end{align*}
Using the triangle inequality and substituting $\tilde{\eps}=1-\alpha$ gives $\|\btg-\bg_{\calS}\| \leq 82\sqrt{\frac{5}{3}}\sigma_0\sqrt{\tilde{\eps}} = \calO(\sigma_0\sqrt{\tilde{\eps}})$.
This completes the proof of \Claimref{final-claim}.
\end{proof}
This concludes the proof of \Lemmaref{poly-time_grad-est}.

\begin{claim*}[Restating \Claimref{convex-concave}]
The cost function $\Phi$ from \eqref{eq:cost-function} is convex in $\bW$ (for a fixed $\bY$) and concave in $\bY$ (for a fixed $\bW$).
\end{claim*}
\begin{proof}
Recall the definition of $\Phi$ from \eqref{eq:cost-function}:
\begin{align*}
\Phi(\bW,\bY) = \sum_{i\in\calA(t)}c_i(t)(\bg_i-\bG_{\calA(t)}\bw_i)^T\bY(\bg_i-\bG_{\calA(t)}\bw_i).
\end{align*}
In order to prove this claim, first we simplify some notation.
Since $c_i(t)\geq0$ for all $i\in\calA(t)$, we can safely ignore them.
We write $\bG_{\calA(t)}$ simply as $\bG$. We also suppress the explicit dependence on $t$. 
With these simplifications, the modified cost function, denoted by $\widetilde{\Phi}$ can be written as:
\begin{align*}
\widetilde{\Phi}(\bW,\bY) = \sum_{i\in\calA}(\bg_i-\bG\bw_i)^T\bY(\bg_i-\bG\bw_i).
\end{align*}
For $i\in\calA$, define $\bz_i = \bg_i-\bG\bw_i$. Let $\bZ$ be a matrix whose columns are $\bz_i$'s. 
Note that $\bZ=(\bG-\bG\bW)=\bG(\bI-\bW)$.
With this notation, we have $\sum_{i\in\calA}(\bg_i-\bG\bw_i)^T\bY(\bg_i-\bG\bw_i)=\sum_{i\in\calA}\bz_i^T\bY\bz_i=\text{tr}(\bZ^T\bY\bZ)$.
Now $\widetilde{\Phi}$ can be written as
\begin{align}
\widetilde{\Phi}(\bW,\bY) &= \text{tr}\Big(\big(\bG(\bI-\bW)\big)^T\bY\big(\bG(\bI-\bW)\big)\Big) \notag \\
&= \text{tr}\big(\bI-\bW^T)\bG^T\bY\bG(\bI-\bW)\big) \notag \\
&= \text{tr}\big(\bG^T\bY\bG + \bW^T\bG^T\bY\bG\bW - \bW^T\bG^T\bY\bG - \bG^T\bY\bG\bW\big) \notag \\
&\stackrel{\text{(a)}}{=} \text{tr}(\bG^T\bY\bG) + \text{tr}(\bW^T\bG^T\bY\bG\bW) - \text{tr}(\bW^T\bG^T\bY\bG) - \text{tr}(\bG^T\bY\bG\bW) \notag \\
&\stackrel{\text{(b)}}{=} \text{tr}(\bG^T\bY\bG) + \text{tr}(\bW\bW^T\bG^T\bY\bG) - \text{tr}(\bW^T\bG^T\bY\bG) - \text{tr}(\bW\bG^T\bY\bG\big) \notag \\
&= \text{tr}\big(\bG^T\bY\bG + \bW\bW^T\bG^T\bY\bG - \bW^T\bG^T\bY\bG - \bW\bG^T\bY\bG\big) \notag \\
&= \text{tr}\big((\bI+\bW\bW^T-\bW^T-\bW)\bG^T\bY\bG\big) \notag \\
&= \text{tr}\big((\bI-\bW)(\bI-\bW^T)\bG^T\bY\bG\big) \label{convex-concave_interim1}
\end{align}
In (a) we used the property that trace is linear in its arguments (i.e., $\text{tr}(\bA+\bB)=\text{tr}(\bA) + \text{tr}(\bB)$ for all $\bA,\bB$). In (b), we used the property that $\text{tr}(\bA\bB)=\text{tr}(\bB\bA)$.
In order to prove that for a fixed $\bY$, $\widetilde{\Phi}$ is convex in $\bW$, it suffices to show that, for any $\lambda\in[0,1]$ and $\bW_1\bW_2\in\W(t)$, we have $\widetilde{\Phi}(\lambda\bW_1+(1-\lambda)\bW_2,\bY) \leq \lambda\widetilde{\Phi}(\bW_1,\bY) + (1-\lambda)\widetilde{\Phi}(\bW_2,\bY)$.
\begin{align*}
\widetilde{\Phi}&(\lambda\bW_1+(1\,-\,\lambda)\bW_2,\bY) \\
&\leq \text{tr}\Big[\Big(\bI-\big(\lambda\bW_1+(1-\lambda)\bW_2\big)\Big)\Big(\bI-\big(\lambda\bW_1+(1-\lambda)\bW_2\big)^T\Big)\bG^T\bY\bG\Big] \\
&= \text{tr}\Big[\Big(\lambda\big(\bI-\bW_1\big)+(1-\lambda)\big(\bI-\bW_2\big)\Big)\Big(\lambda\big(\bI-\bW_1\big)+(1-\lambda)\big(\bI-\bW_2\big)\Big)^T\bG^T\bY\bG\Big] \\
&= \text{tr}\Big[\Big(\lambda^2\big(\bI-\bW_1\big)\big(\bI-\bW_1\big)^T + \lambda(1-\lambda)\big(\bI-\bW_1\big)\big(\bI-\bW_2\big)^T\Big)\bG^T\bY\bG\Big] \\
&\quad + \text{tr}\Big[\Big(\lambda(1-\lambda)\big(\bI-\bW_2\big)\big(\bI-\bW_1\big)^T + (1-\lambda)^2\big(\bI-\bW_2\big)\big(\bI-\bW_2\big)^T\Big)\bG^T\bY\bG\Big] \\
&\stackrel{\text{(a)}}{=} \text{tr}\Big[\Big(\big(\lambda^2-\lambda\big)\big(\bI-\bW_1\big)\big(\bI-\bW_1\big)^T + \lambda(1-\lambda)\big(\bI-\bW_1\big)\big(\bI-\bW_2\big)^T\Big)\bG^T\bY\bG\Big] \\
&\quad + \text{tr}\Big[\Big(\lambda(1-\lambda)\big(\bI-\bW_2\big)\big(\bI-\bW_1\big)^T + \big((1-\lambda)^2-(1-\lambda)\big)\big(\bI-\bW_2\big)\big(\bI-\bW_2\big)^T\Big)\bG^T\bY\bG\Big] \\
&\qquad + \text{tr}\Big[\Big(\lambda\big(\bI-\bW_1\big)\big(\bI-\bW_1\big)^T + (1-\lambda)\big(\bI-\bW_2\big)\big(\bI-\bW_2\big)^T\Big)\bG^T\bY\bG\Big] \\
&= \text{tr}\Big[\Big(\lambda(1-\lambda)\big(\bI-\bW_1\big)\big[\big(\bI-\bW_2\big)^T-\big(\bI-\bW_1\big)^T\big]\Big)\bG^T\bY\bG\Big] \\
&\quad + \text{tr}\Big[\Big(\lambda(1-\lambda)\big(\bI-\bW_2\big)\big[\big(\bI-\bW_1\big)^T-\big(\bI-\bW_2\big)^T\big] \Big)\bG^T\bY\bG\Big] \\
&\qquad + \text{tr}\Big[\Big(\lambda\big(\bI-\bW_1\big)\big(\bI-\bW_1\big)^T + (1-\lambda)\big(\bI-\bW_2\big)\big(\bI-\bW_2\big)^T\Big)\bG^T\bY\bG\Big] \\
&= \text{tr}\Big[\Big(\lambda(1-\lambda)\big[\big(\bI-\bW_1\big)-\big(\bI-\bW_2\big)\big]\big[\big(\bI-\bW_2\big)^T-\big(\bI-\bW_1\big)^T\big]\Big)\bG^T\bY\bG\Big] \\
&\quad + \text{tr}\Big[\Big(\lambda\big(\bI-\bW_1\big)\big(\bI-\bW_1\big)^T + (1-\lambda)\big(\bI-\bW_2\big)\big(\bI-\bW_2\big)^T\Big)\bG^T\bY\bG\Big] \\
&= \text{tr}\Big[-\Big(\lambda(1-\lambda)\big[\big(\bW_1-\bW_2\big)\big(\bW_1-\bW_2\big)^T\big]\Big)\bG^T\bY\bG\Big] \\
&\quad + \text{tr}\Big[\Big(\lambda\big(\bI-\bW_1\big)\big(\bI-\bW_1\big)^T + (1-\lambda)\big(\bI-\bW_2\big)\big(\bI-\bW_2\big)^T\Big)\bG^T\bY\bG\Big] \\
&\stackrel{\text{(b)}}{\leq} \text{tr}\Big[\Big(\lambda\big(\bI-\bW_1\big)\big(\bI-\bW_1\big)^T + (1-\lambda)\big(\bI-\bW_2\big)\big(\bI-\bW_2\big)^T\Big)\bG^T\bY\bG\Big] \\
&= \lambda\text{tr}\Big[\Big(\big(\bI-\bW_1\big)\big(\bI-\bW_1\big)^T\Big)\bG^T\bY\bG\Big] + (1-\lambda)\text{tr}\Big[\Big(\big(\bI-\bW_2\big)\big(\bI-\bW_2\big)^T\Big)\bG^T\bY\bG\Big] \\
&= \lambda\widetilde{\Phi}(\bW_1,\bY) + (1-\lambda)\widetilde{\Phi}(\bW_2,\bY)
\end{align*}
In (a) we added and subtracted $\text{tr}\Big[\Big(\lambda\big(\bI-\bW_1\big)\big(\bI-\bW_1\big)^T + (1-\lambda)\big(\bI-\bW_2\big)\big(\bI-\bW_2\big)^T\Big)\bG^T\bY\bG\Big]$.
In (b) we used that $\text{tr}\Big[\big(\bW_1-\bW_2\big)\big(\bW_1-\bW_2\big)^T\bG^T\bY\bG\Big] = \text{tr}\Big[\big(\bW_1-\bW_2\big)^T\bG^T\bY\bG\big(\bW_1-\bW_2\big)\Big] \geq 0$, as $\big(\bW_1-\bW_2\big)^T\bG^T\bY\bG\big(\bW_1-\bW_2\big)\succeq\bzero$ (recall that, by assumption, $\bY\succeq\bzero$), implying that all its eigenvalues are greater than or equal to zero. Since trace is equal to the sum of eigenvalues, the trace is also bigger than or equal to zero. This completes the proof of \Claimref{convex-concave}.
\end{proof}

}
\end{document}